\documentclass{article}

\usepackage{preprint,times}

\usepackage[utf8]{inputenc}
\usepackage[T1]{fontenc}

\usepackage{hyperref}
\usepackage{url}
\usepackage{booktabs}
\usepackage{amsfonts}
\usepackage{nicefrac}
\usepackage{microtype}
\usepackage{graphicx}
\usepackage{natbib}
\usepackage{doi}

\usepackage{amsmath}
\usepackage{amssymb}
\usepackage{amsthm}
\usepackage{mathtools}
\usepackage{cleveref}

\usepackage{array}
\usepackage{arydshln}
\usepackage{subcaption}
\usepackage{xcolor}
\usepackage{colortbl}

\usepackage{algorithm}
\usepackage{algpseudocode}

\theoremstyle{plain}
\newtheorem{theorem}{Theorem}
\newtheorem{lemma}{Lemma}
\newtheorem{proposition}{Proposition}

\theoremstyle{definition}
\newtheorem{assumption}{Assumption}

\theoremstyle{remark}
\newtheorem{remark}{Remark}

\newcommand{\best}[1]{\textbf{#1}}
\newcommand{\second}[1]{\underline{#1}}

\providecommand{\safeincludegraphics}[2][\linewidth]{%
    \IfFileExists{#2}{%
        \includegraphics[width=#1]{#2}%
    }{%
        \fbox{%
            \parbox[c][2.75cm][c]{#1}{%
                \centering\scriptsize
                Figure placeholder:\\[-1mm]
                \detokenize{#2}%
            }%
        }%
    }%
}

\title{
GUIDE-FBO: Guidance via Uncertainty \\
Intervention and Distributional Exchange \\
for Federated Bayesian Optimization
}

\author{
Jintao Wei
\qquad
Chenxi Li
\qquad
Songhao Wang\thanks{
Corresponding author: \texttt{wangsh2021@sustech.edu.cn}
}
\\
College of Business, Southern University of Science and Technology
\\
Shenzhen, China
}

\date{}

\begin{document}

\setcounter{footnote}{1}
\maketitle

\begin{abstract}
Federated Bayesian Optimization (FBO) enables distributed agents to collaboratively optimize expensive black-box objectives without sharing raw local observations. However, effective knowledge transfer remains challenging under communication constraints and task heterogeneity. 
We propose GUIDE-FBO, in which agents exchange compact distributions over the locations of their respective optima inferred from local Gaussian process (GP) posteriors, rather than raw observations, query points, or surrogate parameters.
The server merges and reweights these distributional components before returning a subset to each agent. Each agent then constructs a Federated Interventional GP (FI-GP), which preserves the local posterior mean and spatially rescales its covariance for local decision making.
For the upper confidence bound (UCB) instantiation, GUIDE-UCB, we prove that any bounded FI-GP uncertainty intervention preserves the leading-order cumulative regret rate of standard GP-UCB. When the transferred distributions place greater support near an optimum than in a suboptimal region, selecting the latter requires greater local posterior uncertainty.
Experiments on 12 synthetic benchmarks and three real-world optimization tasks show that GUIDE-FBO remains effective across settings ranging from homogeneous to severely heterogeneous. Ablation results highlight the importance of spatially localized uncertainty intervention, while the communication analysis shows that GUIDE-FBO exchanges only compact distributional messages.
\end{abstract}

\section{Introduction}
Optimizing expensive black-box functions is a ubiquitous challenge across domains ranging from hyperparameter tuning in automated machine learning (AutoML) and drug discovery to advanced manufacturing control. In these settings, the mapping from queries $\mathbf{x}$ to observations $y$ lacks a closed-form expression, and evaluations are costly in terms of time or budget~\citep{shahriari2015taking}.
Bayesian Optimization (BO) is a prominent methodology for such sample-efficient optimization. By employing a probabilistic surrogate model, typically a Gaussian Process (GP)~\citep{seeger2004gaussian}, to approximate the objective function and utilizing an acquisition function to guide the search, BO quantifies uncertainty to strategically balance exploration and exploitation, enabling efficient global optimization within a limited evaluation budget~\citep{wang2023recent}.

Recently, BO has been extended to distributed environments characterized by data silos and privacy concerns, termed Federated Bayesian Optimization (FBO)~\citep{dai2020federated}. We focus on federated multi-task BO. In this context, $N$ agents aim to accelerate the optimization of related but heterogeneous local objective functions $\{f_{n}(\mathbf{x})\}_{n=1}^{N}$ through knowledge transfer, without revealing their raw local observations. This setting arises in dose optimization across multiple clinical centers, where evaluations are costly and safety is critical. Differences among institutions and patient populations require personalized decisions~\citep{o1990continual,yu2015predicting,willard2025bayesian}.
Similar applications appear in federated hyperparameter tuning and robotic controller
optimization~\citep{marco2016automatic,khodak2021federated}.

However, transforming this immense potential into practice presents fundamental challenges. An FBO method must keep raw local observations private while addressing two central issues: (I) Task Heterogeneity: Objective functions can be distinct ($f_n \neq f_m$), meaning blind aggregation risks negative transfer; and (II) Communication Cost: Bandwidth is often limited, necessitating lightweight interaction protocols~\citep{dai2020federated, al2024collaborative}. Existing methods typically trade off these requirements. Federated Thompson Sampling (FTS)~\citep{dai2020federated} and its differentially private distributed exploration variant (DP-FTS-DE)~\citep{dai2021differentially} exchange random feature representations of local GP surrogates, whose communication cost grows with the approximation dimension. Federated Many-Task BO (FMTBO)~\citep{zhu2023federated} communicates lightweight GP hyperparameter information, but such sparse model-level summaries provide limited guidance for local sequential search. Constrained Gaussian Process (CGP) methods~\citep{chen2025multi} instead transfer selected high-potential queries. Under task heterogeneity, however, a point that is promising for one task may not be promising for another, making such point-level transfer fragile. These limitations motivate a fundamental question: What information should be exchanged to enable lightweight and effective knowledge transfer across heterogeneous tasks without exposing raw local observations?

To answer this question, we propose GUIDE-FBO, which uses a distributional representation of \textit{where the agents' optima are likely to lie} as the information carrier. Specifically, each agent transfers a distribution over the location of its optimum, as inferred from its GP posterior. The server receives and merges spatially redundant distributional components, reweights the merged set, and sends a subset of the resulting components to each agent. Each agent then converts the received components into a spatially varying scaling function for its local GP posterior covariance, yielding the Federated Interventional GP (FI-GP) for local BO decision making. The underlying philosophy is to direct each agent toward globally supported promising regions, thereby reducing redundant evaluations and improving sample efficiency across the federated system. Meanwhile, communication costs are lowered by transmitting only compact distribution parameters.

Our main contributions are summarized as follows:
\begin{itemize}
    \item We propose GUIDE-FBO, which exchanges compact distributions over the locations of agents' optima inferred from local GP posteriors. The server merges, reweights, and redistributes these distributions to provide global guidance with lightweight communication.

    \item We introduce FI-GP, which preserves the local GP posterior mean and spatially rescales its covariance according to the received distributions. It supports standard BO acquisition functions. For GUIDE-UCB, we prove the same leading-order cumulative regret rate as standard GP-UCB under bounded uncertainty intervention and show that informative transferred distributions raise the local uncertainty required to select a suboptimal point.

    \item GUIDE-FBO achieves the best average rank across all three heterogeneity levels on 12 synthetic benchmarks and remains effective on three real-world tasks. Ablation and communication analyses further support the proposed mechanism and its communication efficiency.

\end{itemize}
\section{Related Work and Problem Setup}
\label{sec:related_problem}

\subsection{Related Work}
\label{sec:related_work}

We categorize existing federated and collaborative BO methods according to where shared knowledge
enters the optimization pipeline.

The first category preserves a local surrogate fitted exclusively to private observations and uses external knowledge to coordinate query selection. In FTS~\citep{dai2020federated}, agents select query points using their own GPs or GP sample paths transferred from other agents using random Fourier features~\citep{rahimi2007random}. DP-FTS-DE~\citep{dai2021differentially} extends this framework with differential privacy and distributed exploration. Beyond sampled function exchange, \citet{yue2025collaborative} combines agents' local query proposals through a consensus matrix whose weights vary over time, gradually shifting from peer collaboration to personalized search. \citet{liu2025optimization} trains local classifiers on pairwise preferences generated by acquisition functions, aggregates their parameters, and uses the resulting global classifier as an implicit acquisition function.

The second category incorporates shared knowledge into the predictive modeling process used for
acquisition optimization.
FMTBO~\citep{zhu2023federated} shares local GP hyperparameters, estimates task relatedness from
predictive rankings, and combines local and aggregated surrogate predictions through a federated
ensemble acquisition function.
CGP methods~\citep{chen2025multi} receive promising designs from selected collaborators, filter them according to their consistency with local evidence, and use the accepted designs to construct a constrained GP surrogate.
Other methods communicate compact symbolic regression models while retaining local GP uncertainty estimates~\citep{wang2025efficient}, or aggregate local GP surrogates through a Wasserstein
barycenter to form a central predictive model~\citep{zhan2025collaborative}.

GUIDE-FBO differs in how transferred knowledge is integrated into local decision making.
Rather than directly coordinating query proposals or conditioning and combining predictive models, GUIDE preserves the local posterior mean and uses the received distributions only to spatially rescale its covariance for the current decision. The resulting FI-GP serves as an auxiliary decision posterior, while the underlying local GP remains fitted exclusively to private observations.

\subsection{Federated Multi-Task Bayesian Optimization}
\label{sec:problem_setup}

We consider a federated system with a central server and $N$ distributed agents, following the standard federated setting~\citep{mcmahan2017communication}. 
Each agent $n \in \{1,\dots,N\}$ owns a black-box objective function $f_n:\mathcal{X}\rightarrow\mathbb{R}$ over a shared compact domain $\mathcal{X}\subseteq\mathbb{R}^d$. The local objectives are related but may be heterogeneous. For each agent, let $\mathbf{x}_n^*$ denote an optimum of $f_n$ over $\mathcal{X}$, with
\begin{equation*}
\mathbf{x}_n^*
\in
\arg\max_{\mathbf{x}\in\mathcal{X}}
f_n(\mathbf{x}),
\qquad n=1,\dots,N.
\end{equation*}
At the start of round $t\geq1$, agent $n$ has the private observation set $\mathcal{D}_{n,t-1}=\mathcal{D}_{n,0}\cup\{(\mathbf{x}_{n,i},y_{n,i})\}_{i=1}^{t-1}$, where $y_{n,i}=f_n(\mathbf{x}_{n,i})+\epsilon_{n,i}$.
The corresponding pre-query local GP posterior is denoted by $p_{n,t-1}:=p(f_n\mid\mathcal{D}_{n,t-1})=\mathcal{GP}(\mu_{n,t-1},k_{n,t-1})$, with posterior mean $\mu_{n,t-1}$, covariance function $k_{n,t-1}$, and posterior standard deviation $\sigma_{n,t-1}$.
After selecting $\mathbf{x}_{n,t}$ and observing $y_{n,t}$, the local dataset is updated as
$\mathcal{D}_{n,t}=\mathcal{D}_{n,t-1}\cup\{(\mathbf{x}_{n,t},y_{n,t})\}$, yielding the updated local posterior $p_{n,t}$.
Importantly, raw observations remain local and are never shared with the server or other agents. The goal of federated multi-task BO is to establish a collaborative mechanism that exchanges abstract knowledge through the server and accelerates each agent's search for its optimum $\mathbf{x}_n^*$.

\section{Methodology}
\label{sec:methodology}

GUIDE-FBO proceeds in three stages. First, agents upload compact
distributions over the locations of their optima. Second, the server aggregates
and redistributes these distributional components. Third, each agent uses the
received components to construct an FI-GP for local decision making.

At round $t$, GUIDE-FBO constructs the current guidance from the pre-query
local posterior $p_{n,t-1}$. Newly constructed guidance quantities are indexed
by $t$; within a fixed round, the index is omitted when no ambiguity arises.

\subsection{Optimum Distribution Extraction}
\label{sec:local-belief-extraction}

GUIDE-FBO represents transferable knowledge by the posterior distribution over the location of agent $n$'s optimum, denoted by
$p(\mathbf{x}_n^*\mid\mathcal D_{n,t-1})$.
For agent $n$ at round $t$, since $p(\mathbf{x}_n^* \mid \mathcal{D}_{n,t-1})$ induced by $p_{n,t-1}$ is not available in closed form, we approximate it through posterior sampling. Specifically, we draw $M$ posterior sample paths $\{\widehat{f}_{n,t}^{(m)}\}_{m=1}^{M}$ from $p_{n,t-1}$ and maximize each sampled function to obtain candidate optimum locations
$\widehat{\mathbf{x}}_{n,t}^{(m)} \in \arg\max_{\mathbf{x}\in\mathcal{X}} \widehat{f}_{n,t}^{(m)}(\mathbf{x})$.

The resulting set $\mathcal{C}_{n,t} = \{\widehat{\mathbf{x}}_{n,t}^{(m)}\}_{m=1}^{M}$ provides a Monte Carlo approximation of where agent $n$'s optimum may lie. We use orthogonal random features to obtain scalable approximate posterior sample paths, with details deferred to Appendix~\ref{app:orf}. 
To compactly approximate this potentially multimodal distribution over optimum locations, we fit a Dirichlet process Gaussian mixture model (DPGMM)~\citep{rasmussen1999infinite} to the sampled optimum locations:
\[
q_{n,t}(\mathbf{x}^{*})
=
\sum_{k=1}^{K_{n,t}}
\pi_{n,k}
\mathcal{N}\!\left(
\mathbf{x}^{*}
\mid
\boldsymbol{\mu}_{n,k},
\boldsymbol{\Sigma}_{n,k}
\right).
\]
Here, $K_{n,t}$ denotes the number of active mixture components at round $t$. The mixture weight $\pi_{n,k}$ reflects the posterior support assigned to the $k$-th component, while the mean $\boldsymbol{\mu}_{n,k}$ and covariance $\boldsymbol{\Sigma}_{n,k}$ specify its center and spatial spread. This distributional representation converts discrete samples of optimum locations into a compact, transferable summary of promising regions.

To further reduce the communication cost, we use diagonal covariance matrices for the DPGMM components. Under the default uplink budget, GUIDE-FBO uploads only one representative
component, selected as the component with the largest mixture weight, $k_n^{\dagger} \in \arg\max_k \pi_{n,k}$. Figure~\ref{fig:local-module} provides an overview of the complete optimum distribution extraction process.

\begin{figure}[ht]
    \centering
    \includegraphics[width=\linewidth]{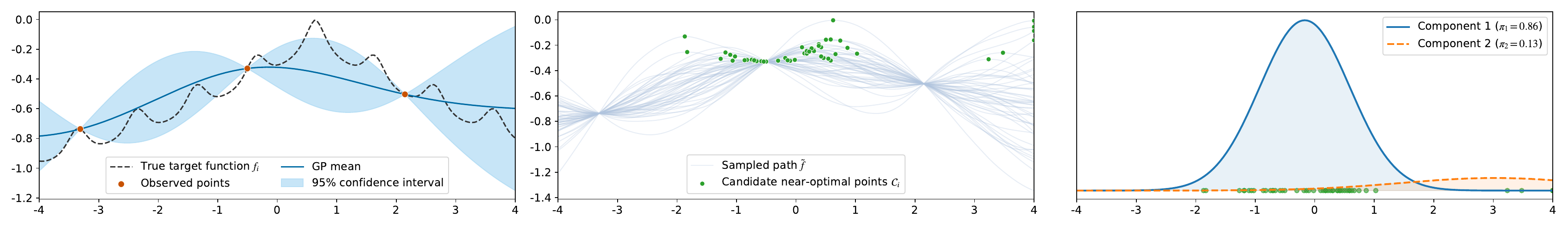}
    \caption{Optimum distribution extraction. A local GP posterior (left) generates candidate optimum locations from posterior sample paths (middle). A DPGMM models their spatial distribution, and its dominant component is selected as the compact distributional message for communication (right).}
    \label{fig:local-module}
\end{figure}

The mixture weight, however, reflects the posterior support assigned to a
component rather than how promising the corresponding region is under the
local surrogate. We therefore attach a conservative value score based on the
local lower confidence bound (LCB),
$\mathrm{LCB}_{n,t}=\mu_{n,t-1}(\boldsymbol{\mu}_{n,k_n^\dagger})-\kappa\sigma_{n,t-1}(\boldsymbol{\mu}_{n,k_n^\dagger}),$
where $\mu_{n,t-1}(\cdot)$ and $\sigma_{n,t-1}(\cdot)$ denote the posterior
mean and standard deviation of the local GP before round $t$, respectively,
and $\kappa>0$ controls the uncertainty penalty. The LCB therefore favors
components with high predicted objective values and low local posterior
uncertainty.

To account for scale differences across heterogeneous local objectives, we standardize the value estimate as $v_{n,t}=(\mathrm{LCB}_{n,t}-\bar{y}_{n,t-1})/\widehat{\sigma}^{\,y}_{n,t-1}$, where $\bar{y}_{n,t-1}$ and $\widehat{\sigma}^{\,y}_{n,t-1}$ are the empirical mean and standard deviation of the local observations available before round $t$. Agent $n$ then uploads
$\Phi_{n,t}=\left(\pi_{n,k_n^{\dagger}},\boldsymbol{\mu}_{n,k_n^{\dagger}},\boldsymbol{\Sigma}_{n,k_n^{\dagger}},v_{n,t}\right)$ to the server.

\subsection{Server-Side Component Merging and Reweighting}
\label{sec:server-merging-reweighting}

At round $t$, the server receives the uploaded components
$\{\Phi_{n,t}\}_{n=1}^{N}$. 
Broadcasting all received components would incur a downlink cost that grows
with the number of agents. GUIDE-FBO therefore returns only a limited subset
to each agent, making it important to avoid spending this budget on spatially
redundant components. 

The server consequently first merges components that represent nearby regions.
Write the received components as $\{(\pi_n,\boldsymbol{\mu}_n,\boldsymbol{\Sigma}_n,v_n)\}_{n=1}^{N}$.
The server measures the distance between two component centers using the
root-mean-square (RMS) Euclidean distance
$d_{\mathrm{RMS}}(\boldsymbol{\mu}_n,\boldsymbol{\mu}_m)
=\sqrt{\|\boldsymbol{\mu}_n-\boldsymbol{\mu}_m\|_2^2/d}$,
where $d=\dim(\mathcal{X})$, and performs complete-linkage clustering using
these centers. Specifically, the distance between two clusters is
$D(\mathcal{I}_a,\mathcal{I}_b)
=\max_{n\in\mathcal{I}_a,\,m\in\mathcal{I}_b}
d_{\mathrm{RMS}}(\boldsymbol{\mu}_n,\boldsymbol{\mu}_m)$.
Starting from singleton clusters, the closest pair is merged whenever
$D(\mathcal{I}_a,\mathcal{I}_b)\leq\delta_{\mathrm{merge}}$,
where $\delta_{\mathrm{merge}}\geq0$ is the spatial merging threshold, until
no eligible pair remains.

Each resulting cluster $\mathcal{I}_\ell$ is represented by a moment-matched Gaussian component:
\begin{equation*}
\pi'_\ell=\sum_{n\in\mathcal{I}_\ell}\pi_n,\qquad\boldsymbol{\mu}'_\ell=\frac{\sum_{n\in\mathcal{I}_\ell}\pi_n\boldsymbol{\mu}_n}{\pi'_\ell},
\qquad
\boldsymbol{\Sigma}'_\ell=\frac{\sum_{n\in\mathcal{I}_\ell}\pi_n\left[\boldsymbol{\Sigma}_n+(\boldsymbol{\mu}_n-\boldsymbol{\mu}'_\ell)(\boldsymbol{\mu}_n-\boldsymbol{\mu}'_\ell)^\top\right]}{\pi'_\ell}.
\end{equation*}
Its value score is aggregated using the same mixture weighted average,
$v'_\ell=(\pi'_\ell)^{-1}\sum_{n\in\mathcal{I}_\ell}\pi_n v_n$.
Under the default diagonal covariance implementation, only the diagonal
entries of $\boldsymbol{\Sigma}'_\ell$ are retained.

Because $\pi'_\ell$ reflects accumulated posterior support but not the
estimated quality of the corresponding region, the server further reweights
the merged components using their value scores as
$\omega_\ell=\pi'_\ell\exp(v'_\ell/\tau_t)/\sum_{j=1}^{L_t}\pi'_j\exp(v'_j/\tau_t)$,
where $\tau_t$ is the adaptive temperature parameter in the current round.
Thus, $\omega_\ell$ balances accumulated posterior support with the
conservative value score. Figure~\ref{fig:server-module} summarizes this
server-side transformation from the received components to the merged and
reweighted global component set.

\begin{figure}[ht]
    \centering
    \includegraphics[width=\linewidth]{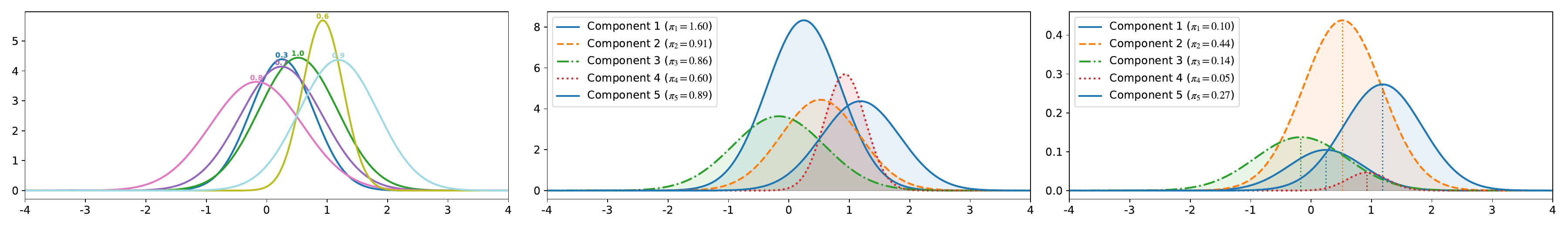}
    \caption{Server-side component processing: received local components (left), moment-matched components after spatial merging (middle), and value-aware reweighted global components (right).}
    \label{fig:server-module}
\end{figure}

The resulting compact global component set is
$\mathcal{B}_t=\{(\omega_\ell,\boldsymbol{\mu}'_\ell,\boldsymbol{\Sigma}'_\ell)\}_{\ell=1}^{L_t}$,
where $L_t\leq N$ and $\omega_\ell$ reflects the global importance of each
merged component as a promising region. Sending the same subset to every agent could concentrate
their parallel searches on the same few regions. To maintain diversity across
agents, the server independently samples up to $P$ components without
replacement for each agent according to $\{\omega_\ell\}_{\ell=1}^{L_t}$,
retaining their global weights. The resulting guidance packet sent to agent
$n$ is
$\Psi_{n,t}=\{(\omega_{n,p},\boldsymbol{\mu}_{n,p},\boldsymbol{\Sigma}_{n,p})\}_{p=1}^{P_t}$,
where $P_t=\min\{P,L_t\}$.

\subsection{Federated Interventional Gaussian Process}
\label{sec:figp}

The final stage integrates the received guidance $\Psi_{n,t}$ into agent
$n$'s local decision making. Because the exchanged components indicate where
promising regions may lie but do not provide objective-value information,
FI-GP preserves the local posterior mean and lets federated guidance act only
through uncertainty. We refer to this as a mean-preserving uncertainty
intervention.

Given
$\Psi_{n,t}=\{(\omega_{n,p},\boldsymbol{\mu}_{n,p},\boldsymbol{\Sigma}_{n,p})\}_{p=1}^{P_t}$,
agent $n$ first constructs a spatial guidance field
\[
G_{n,t}(\mathbf{x})
=
\sum_{p=1}^{P_t}
\omega_{n,p}
\exp\!\left(
-\frac{1}{2}
(\mathbf{x}-\boldsymbol{\mu}_{n,p})^\top
\boldsymbol{\Sigma}_{n,p}^{-1}
(\mathbf{x}-\boldsymbol{\mu}_{n,p})
\right).
\]
Each component contributes most strongly around its center, with
$\omega_{n,p}$ controlling its influence and
$\boldsymbol{\Sigma}_{n,p}$ determining its spatial extent.
$G_{n,t}$ is a nonnegative guidance field rather than a probability density.
Because each Gaussian-shaped term lies in $(0,1]$ and the retained global
weights satisfy $\sum_{p=1}^{P_t}\omega_{n,p}\leq1$, we have
$0\leq G_{n,t}(\mathbf{x})\leq1$.

We then convert this spatial guidance field into a spatially varying uncertainty scaling
function,
$S_{n,t}(\mathbf{x})=1+\lambda_tG_{n,t}(\mathbf{x})$,
where $\lambda_t=\lambda_{\max}/\sqrt{t}$ controls the intervention strength.
It follows that
$1\leq S_{n,t}(\mathbf{x})\leq1+\lambda_t$, so the intervention can only
expand local posterior uncertainty.
The $1/\sqrt{t}$ decay gives federated guidance greater
influence early in optimization and gradually reduces it as local evidence
accumulates.

Let $\mu_{n,t-1}(\mathbf{x})$ and
$k_{n,t-1}(\mathbf{x},\mathbf{x}')$
denote the mean and covariance functions of the local GP posterior before
round $t$. FI-GP keeps the local posterior mean unchanged and spatially
rescales its covariance to form the round-$t$ decision posterior:
\[
\widetilde{\mu}_{n,t}(\mathbf{x})
=
\mu_{n,t-1}(\mathbf{x}),
\qquad
\widetilde{k}_{n,t}(\mathbf{x},\mathbf{x}')
=
S_{n,t}(\mathbf{x})
k_{n,t-1}(\mathbf{x},\mathbf{x}')
S_{n,t}(\mathbf{x}').
\]
This rescaling preserves positive semidefiniteness and hence defines a valid
GP covariance, as shown in Appendix~\ref{app:figp-scaling}. We denote the
resulting FI-GP decision posterior by $\widetilde{p}_{n,t}$. 
Accordingly, its
standard deviation satisfies
$\widetilde{\sigma}_{n,t}(\mathbf{x})
=S_{n,t}(\mathbf{x})\sigma_{n,t-1}(\mathbf{x})$.
This decision posterior is used only for query selection; the fitted local GP
remains unchanged and continues to be updated solely from local observations.
Figure~\ref{fig:figp-module} illustrates this mean-preserving uncertainty
intervention.

\begin{figure}[h]
    \centering
    \includegraphics[width=0.46\linewidth]{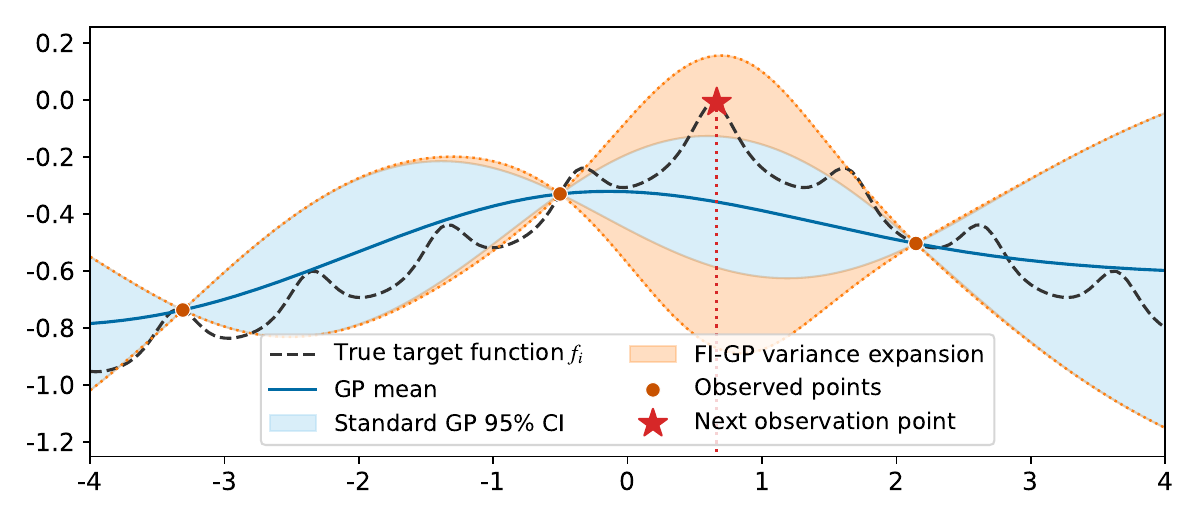}
    \caption{Federated interventional decision making. FI-GP preserves the local posterior mean while expanding uncertainty around regions supported by federated guidance to guide the next query.}
    \label{fig:figp-module}
\end{figure}

Importantly, the absolute effect of federated guidance is further modulated by local
posterior uncertainty:
$
\widetilde{\sigma}_{n,t}(\mathbf{x})
-\sigma_{n,t-1}(\mathbf{x})
=
\lambda_t G_{n,t}(\mathbf{x})
\sigma_{n,t-1}(\mathbf{x}).
$
Thus, for the same guidance intensity, the intervention has a smaller
absolute effect where the local GP is already more certain.

For acquisition-based rules such as UCB and noisy expected improvement (NEI),
agent $n$ selects
$\mathbf{x}_{n,t}\in\arg\max_{\mathbf{x}\in\mathcal{X}}
\alpha(\mathbf{x};\widetilde{p}_{n,t})$.
For Thompson sampling (TS), the agent applies the same uncertainty scaling to
a local posterior sample path and optimizes the resulting path. Thus, the same
uncertainty intervention can be combined with different BO decision rules.

The complete GUIDE-FBO procedure is summarized in
Algorithm~\ref{alg:guide-fbo} in Appendix~\ref{app:algorithm}.

\section{Theoretical Analysis}
\label{sec:theory}

We analyze GUIDE-UCB under two guidance regimes. First, under arbitrary
bounded guidance, including inaccurate or misleading guidance, we establish
whether GUIDE-UCB preserves the regret rate of standard GP-UCB. Second, under
informative guidance, we characterize the local posterior uncertainty that a
suboptimal point must retain to be selected.

For agent $n$, let $r_{n,t}=f_n(\mathbf{x}_n^*)-f_n(\mathbf{x}_{n,t})$ denote its instantaneous regret and $R_{n,T}=\sum_{t=1}^{T}r_{n,t}$ its cumulative regret. Using the FI-GP uncertainty scaling defined in Section~\ref{sec:figp}, GUIDE-UCB selects
\begin{equation}
\mathbf{x}_{n,t}
\in
\arg\max_{\mathbf{x}\in\mathcal X}
\mu_{n,t-1}(\mathbf{x})
+
\sqrt{\beta_t}
\bigl(
1+\lambda_tG_{n,t}(\mathbf{x})
\bigr)
\sigma_{n,t-1}(\mathbf{x}).
\label{eq:guide-ucb-rule}
\end{equation}
Our analysis relies on the standard simultaneous GP confidence event used in GP-UCB and kernelized bandit analyses~\citep{srinivas2010gaussian,chowdhury2017kernelized}. 
The guidance field is uniformly bounded as
$0\leq G_{n,t}(\mathbf{x})\leq G_{\max}$, with $G_{\max}=1$ being a valid
bound under the normalized weighting used in GUIDE-FBO; see
Remark~\ref{rem:bounded-guidance}. Formal confidence conditions, auxiliary results, and complete proofs are provided in Appendix~\ref{app:theory}.

\begin{theorem}[Cumulative regret under arbitrary bounded guidance]
\label{thm:hetero-regret}
Under Assumption~\ref{assump:local-confidence}, suppose that the standard cumulative posterior variance bound in Lemma~\ref{lem:info-gain} holds,
$\sum_{t=1}^{T}
\sigma_{n,t-1}^{2}(\mathbf{x}_{n,t})
\leq
C_{\gamma}\gamma_{n,T}$,
where $\gamma_{n,T}$ is the standard maximum information-gain (log-determinant) complexity term associated with the local kernel of agent $n$. If $\lambda_t=\lambda_{\max}/\sqrt{t}$ for some fixed $\lambda_{\max}\geq0$, then, with probability at least $1-\delta$, every agent $n$ satisfies
\begin{equation}
R_{n,T}
\leq
\sqrt{\beta_T C_{\gamma}\gamma_{n,T}}
\sqrt{
4T
+
4\lambda_{\max}G_{\max}(2\sqrt{T}-1)
+
\lambda_{\max}^{2}G_{\max}^{2}(1+\log T)
}.
\label{eq:hetero-regret-bound}
\end{equation}
Consequently,
$R_{n,T}
=
\mathcal{O}\!\left(
\sqrt{T\beta_T\gamma_{n,T}}
\right)$.
\end{theorem}

Theorem~\ref{thm:hetero-regret} shows that GUIDE-UCB preserves the leading-order cumulative regret rate of standard GP-UCB under arbitrary bounded guidance. In Eq.~\ref{eq:hetero-regret-bound}, the guidance-dependent terms grow only as $\sqrt{T}$ and $\log T$, compared with the leading $T$ term inside the square root. Thus, bounded guidance affects
the regret bound only through lower-order terms. A fixed intervention strength $\lambda_t\equiv\lambda$ also preserves the same leading-order regret rate up to a constant factor; see Appendix~\ref{app:constant-intervention}.

\begin{proposition}[Selection threshold under informative guidance]
\label{prop:effective-threshold}
Under Assumption~\ref{assump:local-confidence}, fix an agent $n$ and round $t$, and suppose that $\lambda_t\ge0$. If GUIDE-UCB selects a suboptimal point $\mathbf{x}_{n,t}$ with $\Delta_n(\mathbf{x}_{n,t}):=f_n(\mathbf{x}_n^*)-f_n(\mathbf{x}_{n,t})>0$, then
\begin{equation}
    \sigma_{n,t-1}(\mathbf{x}_{n,t})
    \ge
    \frac{
        \Delta_n(\mathbf{x}_{n,t})
        +
        \lambda_t\sqrt{\beta_t}
        G_{n,t}(\mathbf{x}_n^*)\sigma_{n,t-1}(\mathbf{x}_n^*)
    }{
        \sqrt{\beta_t}
        \bigl(2+\lambda_tG_{n,t}(\mathbf{x}_{n,t})\bigr)
    }.
    \label{eq:exact-threshold}
\end{equation}
If, in addition, $G_{n,t}(\mathbf{x}_n^*)\ge c$ and $G_{n,t}(\mathbf{x}_{n,t})\le\varepsilon_G$ for some $c>0$ and $\varepsilon_G\ge0$, then under $\lambda_t=\lambda_{\max}/\sqrt t$ with $\lambda_{\max}\ge0$,
\begin{equation*}
    \sigma_{n,t-1}(\mathbf{x}_{n,t})
    \ge
    \frac{
        \Delta_n(\mathbf{x}_{n,t})
        +
        \frac{\lambda_{\max}c}{\sqrt t}
        \sqrt{\beta_t}
        \sigma_{n,t-1}(\mathbf{x}_n^*)
    }{
        \sqrt{\beta_t}
        \left(
            2+\frac{\lambda_{\max}\varepsilon_G}{\sqrt t}
        \right)
    }.
\end{equation*}
\end{proposition}

Proposition~\ref{prop:effective-threshold} gives a necessary local
uncertainty condition for selecting a suboptimal point under informative
guidance. When the guidance is large near the agent's optimum and small at a suboptimal candidate, that candidate can be selected only if it retains sufficiently large local posterior uncertainty.
In particular, when $\varepsilon_G$ is negligible, the lower bound on the required uncertainty exceeds the standard GP-UCB counterpart $\Delta_n(\mathbf{x}_{n,t})/(2\sqrt{\beta_t})$ by approximately $\lambda_{\max}c\,\sigma_{n,t-1}(\mathbf{x}_n^*)/(2\sqrt t)$.
Thus, when guidance favors the optimum over a suboptimal candidate, the latter can be selected only if it remains sufficiently uncertain under the local GP.

\section{Experiments}
\label{sec:experiments}

\subsection{Experimental Setup}
\label{sec:exp_setup}

\paragraph{Baselines.}
We compare GUIDE-FBO against nine competitive methods from three categories.
The first category contains non-collaborative BO methods, including TS~\citep{thompson1933likelihood}, UCB~\citep{srinivas2010gaussian}, and NEI~\citep{letham2019constrained,balandat2020botorch}.
The second category contains classical FBO methods, including FTS~\citep{dai2020federated} and its distributed exploration variant FTS-DE~\citep{dai2021differentially}.
The third category contains recent federated and collaborative BO methods, including FMTBO~\citep{zhu2023federated} and the CGP family (CGP-TS, CGP-UCB, and CGP-NEI)~\citep{chen2025multi}.
We instantiate GUIDE-FBO with the same three local acquisition rules, yielding GUIDE-TS, GUIDE-UCB, and GUIDE-NEI. These variants use the same distributional exchange and uncertainty intervention mechanism and differ only in the local acquisition function.

\paragraph{Benchmarks and heterogeneity.}
We evaluate all methods on 12 synthetic benchmark functions and three
real-world tasks. Landmine Detection~\citep{xue2007multi} uses 29 landmine
fields as agents that tune support vector machines with radial basis function
kernels; Activity Recognition~\citep{anguita2013public} uses 30 subjects as
agents that tune logistic-regression classifiers; and FedHPO-Bench
(FedHPO)~\citep{wang2023fedhpo} uses Cora, CiteSeer, and PubMed as three
graph convolutional network hyperparameter-optimization tasks evaluated with
the official surrogates.
For each synthetic benchmark and experimental run, the local objectives are
generated from a common base function using independently sampled random
shifts and rotations for different agents. The same realized objectives are
used by all compared methods within each run. Level~1 uses no transformation;
Level~2 uses $(\delta_{\mathrm{shift}},\delta_{\mathrm{rot}})=(0.05,0.1)$;
and Level~3 uses $(0.3,1.0)$, corresponding to homogeneous, mild, and severe
heterogeneity, respectively. Complete task construction and heterogeneity
details are provided in Appendix~\ref{app:heterogeneity}.

\paragraph{Protocol and metrics.}
All experiments use 30 initial evaluations per agent and 50 BO rounds and
are repeated over 10 independent runs. Synthetic experiments use $N=16$
agents with Gaussian observation noise of standard deviation $0.1$, while
the real-world tasks use their natural numbers of agents.
For synthetic benchmarks, we report simple regret averaged across agents.
For Landmine Detection, we report the area under the receiver operating
characteristic curve (AUC), while Activity Recognition and FedHPO use
validation accuracy. Higher values are better for all real-world metrics.
Complete benchmark definitions and implementation details are provided in
Appendix~\ref{app:benchmarks}.

\subsection{Main Results}
\label{sec:main_results}

\paragraph{Synthetic benchmarks.}
Table~\ref{tab:synthetic_average_rank} shows that GUIDE achieves strong
aggregate performance across all three heterogeneity levels. GUIDE-UCB obtains
the best average rank over the 12 synthetic benchmarks at Levels~1--3, with
ranks of $1.33$, $1.58$, and $2.58$, respectively, while GUIDE-NEI ranks
second with $2.42$, $2.17$, and $3.42$.
Matched acquisition comparisons further isolate the contribution of GUIDE
itself: GUIDE-UCB improves upon UCB on $12/12$, $11/12$, and $8/12$
benchmarks from Levels~1--3; GUIDE-NEI improves upon NEI on $12/12$,
$11/12$, and $9/12$; and GUIDE-TS improves upon TS on $12/12$, $12/12$,
and $8/12$. These results show that the benefit of GUIDE extends across different acquisition rules
rather than being specific to one of them.

At Levels~1 and~2, where the empirical mean pairwise Spearman correlations
of objective rankings are $1.000$ and $0.557$, respectively, GUIDE-UCB and
GUIDE-NEI frequently reach lower-regret regions earlier and maintain favorable
convergence trajectories as more local observations are collected.
This pattern reflects the mechanism of GUIDE: FI-GP increases local posterior
uncertainty around regions supported by the transferred distributions,
allowing them to be explored earlier when they remain plausible under the
local posterior.

At Level~3, the mean pairwise Spearman correlation of objective rankings
decreases to $0.081$, indicating that useful information is substantially
less transferable across agents. GUIDE-UCB and GUIDE-NEI nevertheless retain
the first and second best average ranks, although their advantages become
smaller and more task dependent. Two design features limit the influence of
mismatched guidance. First, the intervention strength decays over time, so
federated guidance gradually gives way to local evidence. Second, the
uncertainty increment is proportional to local posterior uncertainty, so
guidance has little effect where the local GP is already confident. Negative
transfer can nevertheless occur on individual tasks under severe
heterogeneity. Detailed results and convergence analyses are provided in
Appendix~\ref{app:full-results}.

\begin{table*}[ht]
    \caption{Average rank across 12 synthetic benchmarks under three heterogeneity levels, computed from the final average simple regret over 10 independent runs. Lower is better; best and second-best results are shown in bold and underlined.}
    \label{tab:synthetic_average_rank}
    \centering
    \setlength{\tabcolsep}{3.1pt}
    \renewcommand{\arraystretch}{1.08}
    \scriptsize
    \resizebox{\textwidth}{!}{%
    \begin{tabular}{lcccccccccccc}
    \toprule
    \textbf{Heterogeneity}
    & \textbf{TS}
    & \textbf{UCB}
    & \textbf{NEI}
    & \textbf{FTS}
    & \textbf{FTS-DE}
    & \textbf{FMTBO}
    & \textbf{CGP-TS}
    & \textbf{CGP-UCB}
    & \textbf{CGP-NEI}
    & \textbf{GUIDE-TS}
    & \textbf{GUIDE-UCB}
    & \textbf{GUIDE-NEI} \\
    \midrule
    Level 1: $(0,0)$       & 11.33 & 7.00 & 5.50 & 10.42 & 9.83 & 5.75 & 10.17 & 4.25 & 2.75 & 7.25 & \best{1.33} & \second{2.42} \\
    Level 2: $(0.05,0.1)$  & 10.67 & 6.17 & 5.33 & 10.25 & 10.75 & 5.58 & 10.25 & 4.67 & 3.08 & 7.50 & \best{1.58} & \second{2.17} \\
    Level 3: $(0.3,1.0)$   & 9.17 & 4.83 & 5.42 & 11.08 & 10.33 & 3.67 & 8.33 & 4.25 & 6.75 & 8.17 & \best{2.58} & \second{3.42} \\
    \bottomrule
    \end{tabular}%
    }
\end{table*}

\paragraph{Real-world tasks.}
Table~\ref{tab:realworld_results} shows that a GUIDE variant achieves the
highest final mean performance on each real-world task, although the
differences among the leading methods are small on Activity Recognition and
FedHPO. GUIDE-UCB obtains the highest AUC on Landmine Detection, while
GUIDE-NEI obtains the highest validation accuracy on Activity Recognition
and FedHPO and the best average rank of $1.67$. Moreover, GUIDE-UCB has higher
final mean performance than UCB, and GUIDE-NEI than NEI, on all three tasks.
By contrast, GUIDE-TS improves over TS only on Landmine Detection and has an
average rank of $9.67$ across the three tasks. Complete convergence
trajectories are provided in Appendix~\ref{app:realworld_curves}.

\begin{table*}[ht]
    \caption{
    Results on three real-world federated optimization benchmarks, averaged over 10 independent runs. Higher task performance and lower average rank are better; best and second-best results are shown in bold and underlined.
    }
    \label{tab:realworld_results}
    \centering
    \setlength{\tabcolsep}{3.0pt}
    \renewcommand{\arraystretch}{1.08}
    \scriptsize
    \resizebox{\textwidth}{!}{%
    \begin{tabular}{lcccccccccccc}
    \toprule
    \textbf{Task}
    & \textbf{TS}
    & \textbf{UCB}
    & \textbf{NEI}
    & \textbf{FTS}
    & \textbf{FTS-DE}
    & \textbf{FMTBO}
    & \textbf{CGP-TS}
    & \textbf{CGP-UCB}
    & \textbf{CGP-NEI}
    & \textbf{GUIDE-TS}
    & \textbf{GUIDE-UCB}
    & \textbf{GUIDE-NEI} \\
    \midrule

    Landmine Detection (AUC)    & 0.803987 & 0.805593 & 0.805508 & 0.804052 & 0.801349 & 0.805970 
                                & 0.805247 & \second{0.806501} & 0.805218 & 0.805001 & \best{0.806604} 
                                & 0.806452 \\

    Activity Recognition (Acc.)   & 0.987787 & 0.987821 & 0.987804 & 0.987800 & 0.987739
                                & 0.987891 & 0.987701 & 0.987959 & \second{0.987990}
                                & 0.987754 & 0.987938 & \best{0.987994} \\

    FedHPO (Acc.)   & 0.850968 & 0.850973 & 0.851145 & 0.850439 & 0.851098 & 0.851049
                    & 0.850766 & \second{0.851262} & 0.851068 & 0.850853 & 0.851178 & \best{0.851271} \\

    \midrule
    \textbf{Average rank}   & 9.67 & 6.33 & 5.67 & 10.00 & 9.33 & 5.33 & 10.00 & \second{2.33}
                            & 5.33 & 9.67 & 2.67 & \best{1.67} \\

    \bottomrule
    \end{tabular}%
    }
\end{table*}

\subsection{Ablation and Communication Efficiency}
\label{sec:ablation_communication}

\paragraph{Ablation study.}
Table~\ref{tab:main_ablation} examines the contribution of the main components of GUIDE-UCB. Under Level~2 heterogeneity, Full GUIDE-UCB achieves the best average rank of $1.50$ and remains in the top two on all four representative benchmarks. The largest degradation occurs when the spatially varying uncertainty intervention is replaced by uniform uncertainty scaling: its average rank increases to $6.25$, close to Vanilla UCB at $6.75$. The absolute results show the same pattern. On
Ackley, Full GUIDE-UCB, Uniform uncertainty scaling, and Vanilla UCB obtain $8.4785$, $19.1366$, and $19.4350$, respectively; on Rastrigin, the corresponding values are $58.3082$, $80.3952$, and $80.4771$. These results indicate that, among the tested Level~2 ablations, the spatial localization of the uncertainty intervention is the strongest factor distinguishing GUIDE-UCB from a uniform increase in posterior uncertainty.

The effects of the remaining modules are smaller and more task dependent. Multimodal belief modeling, value-aware reweighting, component merging, and agent-specific sampling contribute to the
aggregate performance of the complete method, although individual ablations can outperform Full GUIDE-UCB on particular functions. Appendix~\ref{app:ablation} further shows that the relative importance of these components changes under stronger heterogeneity. At Level~3, Full GUIDE-UCB retains the best aggregate rank of $2.50$, whereas removing value-aware reweighting gives the worst rank of $5.25$. This suggests that value-aware server reweighting becomes more important when cross-agent agreement is weak.

\begin{table*}[ht]
    \caption{Ablation study of GUIDE-UCB on four synthetic benchmarks under Level~2 heterogeneity. Final average simple regret over 10 runs is reported; lower is better. The best and second-best results in each row are shown in bold and underlined, respectively.}
    \label{tab:main_ablation}
    \centering
    \setlength{\tabcolsep}{3.0pt}
    \renewcommand{\arraystretch}{1.08}
    \scriptsize
    \resizebox{\textwidth}{!}{%
    \begin{tabular}{lccccccc}
    \toprule
    \textbf{Benchmark}
    & \shortstack{\textbf{Vanilla}\\\textbf{UCB}}
    & \shortstack{\textbf{Single-Gaussian}\\\textbf{belief}}
    & \shortstack{\textbf{w/o value-aware}\\\textbf{reweighting}}
    & \shortstack{\textbf{w/o component}\\\textbf{merging}}
    & \shortstack{\textbf{w/o agent-specific}\\\textbf{sampling}}
    & \shortstack{\textbf{Uniform uncertainty}\\\textbf{scaling}}
    & \shortstack{\textbf{Full}\\\textbf{GUIDE-UCB}} \\
    \midrule
    Ackley      & 19.4350 & 8.6510 & \best{8.4753} & 8.5349 & 8.4855 & 19.1366 & \second{8.4785}\\
    Rastrigin   & 80.4771 & 60.2355 & \second{59.1892} & 59.9796 & 60.0268 & 80.3952 & \best{58.3082} \\
    Zakharov    & 49.8750 & 25.5947 & 22.9307 & 25.3701 & \best{21.5431} & 52.0749 & \second{22.1717} \\
    Michalewicz & 6.1248 & 5.9494 & 5.9814 & 5.9650 & \second{5.8771} & 6.1078 & \best{5.8629} \\
    \midrule
    \textbf{Average rank}
                & 6.75 & 4.50 & 2.75 & 3.75 & \second{2.50} & 6.25 & \best{1.50} \\
    \bottomrule
    \end{tabular}%
    }
\end{table*}

\paragraph{Communication efficiency.}
Table~\ref{tab:main_communication} shows that GUIDE-FBO obtains its optimization gains with a compact communication message. Under the default synthetic setting with $d=10$ and $P=5$, GUIDE-FBO communicates at most $127$ scalars per agent per BO round, compared with $212$ for CGP, $1000$ for FTS, and $4000$ for FTS-DE. FMTBO has the lowest communication cost among the collaborative baselines, requiring only $14$ scalars per round. This lower communication volume, however, is accompanied by weaker aggregate optimization performance than GUIDE-UCB on the synthetic benchmarks, while both GUIDE-UCB and GUIDE-NEI achieve higher final mean performance than FMTBO on all three real-world tasks.
Overall, GUIDE provides a favorable trade-off between preserving spatial information about promising regions and limiting communication cost.

\begin{table}[ht]
    \caption{Communication cost per agent per round for $d=10$. GUIDE-FBO uses diagonal covariance and downloads at most $P=5$ components per round.}
    \label{tab:main_communication}
    \centering
    \small
    \setlength{\tabcolsep}{3pt}
    \renewcommand{\arraystretch}{1.08}
    \begin{tabular}{lccccc}
    \toprule
    & \textbf{FTS}
    & \textbf{FTS-DE}
    & \textbf{FMTBO}
    & \textbf{CGP}
    & \textbf{GUIDE} \\
    \midrule
    \textbf{Uplink}   & 500  & 2000 & 12 & 12  & 22 \\
    \textbf{Downlink} & 500  & 2000 & 2  & 200 & $\leq 105$ \\
    \textbf{Total}    & 1000 & 4000 & 14 & 212 & $\leq 127$ \\
    \bottomrule
    \end{tabular}
\end{table}

\subsection{Additional Analyses}
\label{sec:additional_analyses}

The additional experiments further examine the sensitivity and computational efficiency of GUIDE. 
The Level~3 sensitivity study in Appendix~\ref{app:sensitivity} shows that,
on the tested benchmarks, GUIDE remains stable across moderate choices of
$M$, $P$, and $\delta_{\mathrm{merge}}$, while overly strong uncertainty
intervention degrades performance.
These results suggest that the default
configuration is not narrowly tuned to a single parameter setting.
The runtime study in Appendix~\ref{app:efficiency} shows that GUIDE introduces
additional computation relative to independent BO, but remains more efficient
than the corresponding CGP variants on most Level~2 synthetic benchmarks.
In applications with expensive objective evaluations, this optimizer-side
overhead may constitute a smaller fraction of the total optimization cost.

\section{Conclusion}
\label{sec:conclusion}

GUIDE-FBO approaches federated Bayesian optimization from the perspective of what should be communicated across heterogeneous agents. By representing local search information as distributions over optimum locations, it separates global knowledge exchange from local surrogate fitting and uses transferred information only to spatially rescale posterior uncertainty during local decision making.
This intervention can be combined with UCB, NEI, and TS without redesigning their local decision rules. Our analysis of GUIDE-UCB establishes the same leading-order regret rate as GP-UCB under bounded federated guidance and shows that, when guidance is sufficiently stronger near the optimum than at a suboptimal point, the latter can be selected only if it remains sufficiently uncertain under the local GP. 
Experiments further demonstrate the value of this mechanism across heterogeneous tasks. Although GUIDE-FBO avoids sharing raw local observations, the current framework does not provide a formal
privacy guarantee for the exchanged distributional information. Future work could incorporate differential privacy into local distribution release or server aggregation while studying how privacy noise interacts with federated guidance.

\section*{AI Use Statement}

In this work, we used generative AI tools to provide critical ingredients for proving mathematical claims, design or provide feedback on research methodology and experiments, implement methods, and assist with translation. 

We did not use generative AI tools to generate synthetic datasets, develop theoretical models or conceptual frameworks, formulate mathematical claims, assist in writing proofs, propose or refine hypotheses, clean or reformat datasets, or interpret experimental results. Qualitative and thematic data analysis is not applicable to this work.

We additionally used generative AI tools to create or edit software code, search for information, improve the readability of the manuscript, and propose or refine the paper title and keywords. 

All AI-assisted work was reviewed by the authors. Mathematical claims and proofs were independently checked, experimental interpretations were verified against the reported results, and AI-assisted code was reviewed and tested. We take full responsibility for the final content of this work, including any text, claims, or artifacts produced with the aid of generative AI.

\section*{Reproducibility Statement}

Detailed experimental settings, benchmark construction, data processing, and GUIDE-FBO hyperparameters are provided in Appendix~\ref{app:benchmarks}. 
Full assumptions and proofs for the theoretical results are given in Appendix~\ref{app:theory}.
The complete source code and experiment scripts are publicly available at \url{https://github.com/JintaoWEI/GUIDE-FBO-Federated-Bayesian-Optimization}.

\clearpage
\bibliographystyle{unsrtnat}
\bibliography{references}

\clearpage

\appendix

\section*{Appendix Contents}
\label{app:contents}
\begin{enumerate}
    \item[\textbf{A.}] \hyperref[app:algorithm]{Complete GUIDE-FBO Pseudocode}
    \item[\textbf{B.}] \hyperref[app:theory]{Theoretical Assumptions, Lemmas, and Proofs}
    \item[\textbf{C.}] \hyperref[app:benchmarks]{Benchmarks, Heterogeneity, Metrics, and Experimental Settings}
    \item[\textbf{D.}] \hyperref[app:full-results]{Complete Experimental Results and Convergence Trajectories}
    \item[\textbf{E.}] \hyperref[app:ablation]{Ablation Study Details}
    \item[\textbf{F.}] \hyperref[app:sensitivity]{Sensitivity Analysis}
    \item[\textbf{G.}] \hyperref[app:communication]{Communication Cost Analysis}
    \item[\textbf{H.}] \hyperref[app:efficiency]{Computational Efficiency Analysis}
    \item[\textbf{I.}] \hyperref[app:orf]{Orthogonal Random Features for Posterior-Optimum Sampling}
\end{enumerate}

\section{Complete GUIDE-FBO Pseudocode}
\label{app:algorithm}

The complete optimization procedure is given in Algorithm~\ref{alg:guide-fbo}. The three stages correspond directly to optimum distribution extraction, server-side component merging and reweighting, and FI-GP-based local decision making.

\begin{algorithm}[h]
\caption{GUIDE-FBO}
\label{alg:guide-fbo}
\begin{algorithmic}[1]
\Require Agents $\{1,\ldots,N\}$, initial local datasets
$\{\mathcal{D}_{n,0}\}_{n=1}^{N}$, total rounds $T$,
posterior-optimum samples $M$, downlink packet size $P$,
merge threshold $\delta_{\mathrm{merge}}$, guidance strength $\lambda_{\max}$
\Ensure Best observed local incumbents $\{\mathbf{x}_{n}^{\mathrm{best}}\}_{n=1}^{N}$

\For{each agent $n$ in parallel}
    \State Fit the initial local GP posterior
    $p_{n,0}:=p(f_n\mid\mathcal{D}_{n,0})$
\EndFor

\For{$t=1,\ldots,T$}

    \Statex \textbf{Optimum distribution extraction}
    \For{each agent $n$ in parallel}
        \State Draw $M$ posterior sample paths from $p_{n,t-1}$ and collect candidate optima $\mathcal{C}_{n,t}$
        \State Fit a DPGMM $q_{n,t}(\mathbf{x}^*)$ to $\mathcal{C}_{n,t}$
        \State Select $k_n^\dagger\in\arg\max_k\pi_{n,k}$ and compute the standardized conservative value score $v_{n,t}$
        \State Upload
        $\Phi_{n,t}
        =
        (\pi_{n,k_n^\dagger},
        \boldsymbol{\mu}_{n,k_n^\dagger},
        \boldsymbol{\Sigma}_{n,k_n^\dagger},
        v_{n,t})$
    \EndFor

    \Statex \textbf{Server-side component merging and reweighting}
    \State Cluster uploaded components using complete linkage with threshold $\delta_{\mathrm{merge}}$
    \State Moment-match each cluster $\mathcal{I}_\ell$ to obtain
    $(\pi'_\ell,\boldsymbol{\mu}'_\ell,\boldsymbol{\Sigma}'_\ell,v'_\ell)_{\ell=1}^{L_t}$
    \State Compute
    $\omega_\ell
    \propto
    \pi'_\ell\exp(v'_\ell/\tau_t)$
    and form
    $\mathcal{B}_t
    =
    \{(\omega_\ell,\boldsymbol{\mu}'_\ell,\boldsymbol{\Sigma}'_\ell)\}_{\ell=1}^{L_t}$

    \For{each agent $n$}
        \State Set $P_t=\min\{P,L_t\}$
        \State Sample $P_t$ components without replacement from $\mathcal{B}_t$ according to $\{\omega_\ell\}_{\ell=1}^{L_t}$
        \State Send
        $\Psi_{n,t}
        =
        \{(\omega_{n,p},\boldsymbol{\mu}_{n,p},\boldsymbol{\Sigma}_{n,p})\}_{p=1}^{P_t}$
    \EndFor

    \Statex \textbf{Federated interventional decision making}
    \For{each agent $n$ in parallel}
        \State Construct $G_{n,t}(\mathbf{x})$ and
        $S_{n,t}(\mathbf{x})
        =
        1+\lambda_tG_{n,t}(\mathbf{x})$,
        where $\lambda_t=\lambda_{\max}/\sqrt{t}$
        \State Construct $\widetilde{p}_{n,t}$ by preserving the mean of $p_{n,t-1}$ and spatially scaling its covariance
        \State Select $\mathbf{x}_{n,t}$ using the chosen local BO decision rule under $\widetilde{p}_{n,t}$
        \State Evaluate
        $y_{n,t}
        =
        f_n(\mathbf{x}_{n,t})+\epsilon_{n,t}$
        \State Update
        $\mathcal{D}_{n,t}
        =
        \mathcal{D}_{n,t-1}
        \cup
        \{(\mathbf{x}_{n,t},y_{n,t})\}$
        and refit $p_{n,t}$
    \EndFor

\EndFor

\State \Return
$\mathbf{x}_{n}^{\mathrm{best}}
\in
\arg\max_{(\mathbf{x},y)\in\mathcal{D}_{n,T}}y$
for each agent $n$
\end{algorithmic}
\end{algorithm}

\section{Theoretical Assumptions, Lemmas, and Proofs}
\label{app:theory}

\setcounter{assumption}{0}
\renewcommand{\theassumption}{B.\arabic{assumption}}
\setcounter{lemma}{0}
\renewcommand{\thelemma}{B.\arabic{lemma}}
\setcounter{remark}{0}
\renewcommand{\theremark}{B.\arabic{remark}}

\providecommand{\theHassumption}{}
\renewcommand{\theHassumption}{B.\arabic{assumption}}
\providecommand{\theHlemma}{}
\renewcommand{\theHlemma}{B.\arabic{lemma}}
\providecommand{\theHremark}{}
\renewcommand{\theHremark}{B.\arabic{remark}}

This appendix states the confidence conditions and structural properties used in Section~\ref{sec:theory}, establishes the auxiliary results required by the analysis, and provides the complete proofs. Throughout, round $t$ starts from the local posterior conditioned on $\mathcal D_{n,t-1}$, while all guidance and decision quantities constructed during that round carry index $t$. For theoretical clarity, we additionally restore the round index on all received guidance-component parameters in this appendix.

\begin{assumption}[Local GP confidence event]
\label{assump:local-confidence}
Let $\{\beta_t\}_{t\ge1}$ be a positive and nondecreasing sequence. With probability at least $1-\delta$, the following event holds simultaneously for all agents $n\in[N]$, rounds $t\ge1$, and points $\mathbf{x}\in\mathcal X$:
\[
    |f_n(\mathbf{x})-\mu_{n,t-1}(\mathbf{x})|
    \le
    \sqrt{\beta_t}\sigma_{n,t-1}(\mathbf{x}).
\]
\end{assumption}

Assumption~\ref{assump:local-confidence} is the standard simultaneous confidence event used in GP-UCB and kernelized bandit analyses~\citep{srinivas2010gaussian,chowdhury2017kernelized}. Sufficient RKHS and noise conditions are provided in Section~\ref{app:confidence-conditions}.

\begin{remark}[Nonnegativity and boundedness of the GUIDE field]
\label{rem:bounded-guidance}
For agent $n$ at round $t$, the GUIDE field is
\[
G_{n,t}(\mathbf{x})
=
\sum_{p=1}^{P_t}
\omega_{n,t,p}
\exp\left(
-\frac{1}{2}
(\mathbf{x}-\boldsymbol{\mu}_{n,t,p})^\top
\boldsymbol{\Sigma}_{n,t,p}^{-1}
(\mathbf{x}-\boldsymbol{\mu}_{n,t,p})
\right),
\]
where $\omega_{n,t,p}\ge0$ and
$\boldsymbol{\Sigma}_{n,t,p}\succ0$.
Each Gaussian guidance factor lies in $(0,1]$, and therefore
\[
0
\le
G_{n,t}(\mathbf{x})
\le
\sum_{p=1}^{P_t}\omega_{n,t,p}.
\]
Under the normalized global reweighting and retained-weight downlink protocol,
\[
\sum_{p=1}^{P_t}\omega_{n,t,p}\le1
\]
for every agent $n$ and round $t$. Hence,
\[
G_{\max}
:=
\sup_{n\in[N],\,t\ge1,\,\mathbf{x}\in\mathcal X}
G_{n,t}(\mathbf{x})
\le1.
\]
\end{remark}

\begin{lemma}[Instantaneous regret decomposition]
\label{lem:instant-regret}
Under Assumption~\ref{assump:local-confidence}, for any nonnegative intervention strength $\lambda_t$, the GUIDE-UCB rule in Eq.~\ref{eq:guide-ucb-rule} satisfies, for every agent $n$ and round $t$,
\begin{equation*}
    r_{n,t}
    \le
    \sqrt{\beta_t}
    \Big[
        2\sigma_{n,t-1}(\mathbf{x}_{n,t})
        +
        \lambda_tG_{n,t}(\mathbf{x}_{n,t})
        \sigma_{n,t-1}(\mathbf{x}_{n,t})
        -
        \lambda_tG_{n,t}(\mathbf{x}_n^*)
        \sigma_{n,t-1}(\mathbf{x}_n^*)
    \Big].
\end{equation*}
\end{lemma}

Lemma~\ref{lem:instant-regret} is the one-step inequality underlying both theoretical results. For the cumulative regret bound, the final nonpositive term is discarded. For the informative-guidance result, it is retained to characterize how guidance around $\mathbf{x}_n^*$ changes the selection threshold of suboptimal points.

\subsection{Sufficient Conditions for the Local Confidence Event}
\label{app:confidence-conditions}

The main analysis is stated directly in terms of Assumption~\ref{assump:local-confidence}. A standard set of sufficient conditions is given below.

\begin{assumption}[Compact domain and bounded kernels]
\label{assump:app-domain-kernel}
The decision domain $\mathcal X$ is compact. For every agent $n$, the kernel $k_n$ is bounded on the diagonal:
\[
    k_n(\mathbf{x},\mathbf{x})\le1,
    \qquad
    \forall \mathbf{x}\in\mathcal X.
\]
\end{assumption}

\begin{assumption}[RKHS-bounded local objectives]
\label{assump:app-rkhs}
For every agent $n$, the objective $f_n$ belongs to the RKHS $\mathcal H_{k_n}$ induced by $k_n$ and satisfies
\[
    \|f_n\|_{\mathcal H_{k_n}}\le B_n.
\]
Let $B=\max_{n\in[N]}B_n$.
\end{assumption}

\begin{assumption}[Sub-Gaussian observation noise]
\label{assump:app-noise}
At round $t$, agent $n$ observes
\[
    y_{n,t}
    =
    f_n(\mathbf{x}_{n,t})+\epsilon_{n,t},
\]
where $\epsilon_{n,t}$ is conditionally $R_\epsilon$-sub-Gaussian with respect to the filtration generated by the observation history.
\end{assumption}

Under Assumptions~\ref{assump:app-domain-kernel}--\ref{assump:app-noise}, standard kernelized-bandit concentration results imply Assumption~\ref{assump:local-confidence} for an appropriate nondecreasing confidence sequence $\{\beta_t\}_{t\ge1}$~\citep{srinivas2010gaussian,chowdhury2017kernelized}. A union bound over the $N$ agents yields simultaneous validity across the federation.

As is standard in kernelized-bandit analyses, the theoretical local posterior is formed using a fixed kernel $k_n$ and a fixed positive regularization parameter $\sigma_{\mathrm{gp}}^2>0$. This parameter determines the posterior variance and the log-determinant complexity term used below; it does not impose a Gaussian distribution on the actual observation noise, which remains conditionally $R_\epsilon$-sub-Gaussian as specified in Assumption~\ref{assump:app-noise}. In the empirical implementation, kernel hyperparameters are estimated from the local observations.

\subsection{Maximum Information Gain}
\label{app:information-gain}

The optimization procedure starts from a nonempty initial local dataset
$\mathcal D_{n,0}$. To keep the theoretical notation consistent with the
round convention in Section~\ref{sec:problem_setup}, we condition the
sequential analysis on this initial dataset.

Let
$\mathcal X_{n,0}=\{\mathbf x:(\mathbf x,y)\in\mathcal D_{n,0}\}$
denote the initial design locations of agent $n$. For the fixed theoretical
kernel $k_n$ and GP regularization parameter
$\sigma_{\mathrm{gp}}^2>0$, define the covariance function after conditioning
on the initial dataset as
\begin{equation*}
\begin{aligned}
k_n^{(0)}(\mathbf x,\mathbf x')
&=
k_n(\mathbf x,\mathbf x')
\\
&\quad -
k_n(\mathbf x,\mathcal X_{n,0})
\left(
\mathbf K_{n,0}
+
\sigma_{\mathrm{gp}}^2\mathbf I
\right)^{-1}
k_n(\mathcal X_{n,0},\mathbf x'),
\end{aligned}
\end{equation*}
where
\[
\mathbf K_{n,0}
=
\left[
k_n(\mathbf x,\mathbf x')
\right]_{\mathbf x,\mathbf x'\in\mathcal X_{n,0}}.
\]
Thus, $k_n^{(0)}$ is precisely the covariance of the local GP posterior
available before the first BO round. Subsequent posterior variances
$\sigma_{n,t-1}^2(\mathbf x)$ are obtained by further conditioning this
initial posterior on the BO observations collected in rounds
$1,\ldots,t-1$.

For a finite query set
$A=\{\mathbf x_1,\ldots,\mathbf x_T\}\subset\mathcal X$, define
\[
\mathbf K_{n,A}^{(0)}
=
\left[
k_n^{(0)}(\mathbf x,\mathbf x')
\right]_{\mathbf x,\mathbf x'\in A}.
\]
The maximum information gain conditional on the initial dataset is
\begin{equation*}
\gamma_{n,T}^{(0)}
=
\frac{1}{2}
\max_{A\subset\mathcal X:\,|A|=T}
\log\det\left(
\mathbf I
+
\sigma_{\mathrm{gp}}^{-2}
\mathbf K_{n,A}^{(0)}
\right).
\end{equation*}

For comparison, let
\begin{equation*}
\gamma_{n,T}
=
\frac{1}{2}
\max_{A\subset\mathcal X:\,|A|=T}
\log\det\left(
\mathbf I
+
\sigma_{\mathrm{gp}}^{-2}
\mathbf K_{n,A}
\right),
\end{equation*}
where
\[
\mathbf K_{n,A}
=
\left[
k_n(\mathbf x,\mathbf x')
\right]_{\mathbf x,\mathbf x'\in A}.
\]
This is the standard maximum information-gain, or equivalently
log-determinant complexity, quantity used in GP-UCB and kernelized-bandit
analyses~\citep{srinivas2010gaussian,chowdhury2017kernelized}.

Conditioning on the initial observations cannot increase covariance.
Therefore, for every finite $A$,
\[
\mathbf K_{n,A}^{(0)}
\preceq
\mathbf K_{n,A},
\]
which implies
\begin{equation}
\gamma_{n,T}^{(0)}
\le
\gamma_{n,T}.
\label{eq:conditional-information-gain-upper}
\end{equation}
Consequently, the standard information-gain complexity
$\gamma_{n,T}$ remains a valid upper bound after accounting for the
initial local dataset.

Under a Gaussian observation model, the log-determinant quantities above
coincide with the corresponding mutual information between noisy
observations and latent function values. The analysis here only requires
their log-determinant forms and therefore does not require Gaussian
observation noise.

\begin{lemma}[Cumulative posterior variance]
\label{lem:info-gain}
Under Assumption~\ref{assump:app-domain-kernel}, for any adaptively
selected sequence
$\{\mathbf x_{n,t}\}_{t=1}^{T}$ generated after conditioning on
$\mathcal D_{n,0}$,
\[
\sum_{t=1}^{T}
\sigma_{n,t-1}^{2}(\mathbf x_{n,t})
\le
C_{\gamma}\gamma_{n,T}^{(0)}
\le
C_{\gamma}\gamma_{n,T},
\]
where
\[
C_{\gamma}
=
\frac{2}{
\log\left(
1+\sigma_{\mathrm{gp}}^{-2}
\right)
}.
\]
\end{lemma}

\begin{proof}
Let
\[
u_t
=
\sigma_{n,t-1}^{2}(\mathbf x_{n,t}).
\]
Since conditioning cannot increase posterior variance and
$k_n(\mathbf x,\mathbf x)\le1$ under
Assumption~\ref{assump:app-domain-kernel}, we have
$u_t\in[0,1]$. By concavity of
$u\mapsto\log(1+\sigma_{\mathrm{gp}}^{-2}u)$ on $[0,1]$,
\[
\log\left(
1+\sigma_{\mathrm{gp}}^{-2}u_t
\right)
\ge
u_t
\log\left(
1+\sigma_{\mathrm{gp}}^{-2}
\right),
\]
and hence
\[
u_t
\le
\frac{
\log\left(
1+\sigma_{\mathrm{gp}}^{-2}u_t
\right)
}{
\log\left(
1+\sigma_{\mathrm{gp}}^{-2}
\right)
}.
\]

For the realized BO query sequence, define the kernel matrix under the
initial-data-conditioned covariance as
\[
\mathbf K_{n,1:T}^{(0)}
=
\left[
k_n^{(0)}(\mathbf x_{n,s},\mathbf x_{n,t})
\right]_{s,t=1}^{T}.
\]
The standard sequential determinant identity, now applied after
conditioning on $\mathcal D_{n,0}$, gives
\[
\frac{1}{2}
\log\det\left(
\mathbf I
+
\sigma_{\mathrm{gp}}^{-2}
\mathbf K_{n,1:T}^{(0)}
\right)
=
\frac{1}{2}
\sum_{t=1}^{T}
\log\left(
1+
\sigma_{\mathrm{gp}}^{-2}
\sigma_{n,t-1}^{2}(\mathbf x_{n,t})
\right).
\]
Therefore,
\[
\begin{aligned}
\sum_{t=1}^{T}
\sigma_{n,t-1}^{2}(\mathbf x_{n,t})
&\le
\frac{2}{
\log\left(
1+\sigma_{\mathrm{gp}}^{-2}
\right)
}
\frac{1}{2}
\log\det\left(
\mathbf I
+
\sigma_{\mathrm{gp}}^{-2}
\mathbf K_{n,1:T}^{(0)}
\right)
\\
&\le
C_{\gamma}\gamma_{n,T}^{(0)}
\\
&\le
C_{\gamma}\gamma_{n,T},
\end{aligned}
\]
where the final inequality follows from
Eq.~\ref{eq:conditional-information-gain-upper}.
\end{proof}

\subsection{Validity and Variance Scaling of FI-GP}
\label{app:figp-scaling}

\begin{lemma}[FI-GP covariance validity and variance scaling]
\label{lem:figp-scaling}
Let the local GP posterior of agent $n$ before round $t$ be
\[
    f_n(\mathbf{x})\mid\mathcal D_{n,t-1}
    \sim
    \mathcal{GP}\left(
        \mu_{n,t-1}(\mathbf{x}),
        k_{n,t-1}(\mathbf{x},\mathbf{x}')
    \right).
\]
For $\lambda_t\ge0$, define
\[
S_{n,t}(\mathbf{x})
=
1+\lambda_tG_{n,t}(\mathbf{x})
\]
and
\[
\widetilde{k}_{n,t}(\mathbf{x},\mathbf{x}')
=
S_{n,t}(\mathbf{x})
k_{n,t-1}(\mathbf{x},\mathbf{x}')
S_{n,t}(\mathbf{x}').
\]
Then $\widetilde{k}_{n,t}$ is a valid positive-semidefinite covariance kernel. Moreover, the FI-GP decision posterior preserves the local posterior mean and satisfies
\[
\widetilde{\mu}_{n,t}(\mathbf{x})
=
\mu_{n,t-1}(\mathbf{x}),
\qquad
\widetilde{\sigma}_{n,t}(\mathbf{x})
=
S_{n,t}(\mathbf{x})
\sigma_{n,t-1}(\mathbf{x}).
\]
\end{lemma}

\begin{proof}
Consider an arbitrary finite set
$\mathbf X=\{\mathbf{x}_1,\ldots,\mathbf{x}_q\}$. Let
\[
    \mathbf K
    =
    [
        k_{n,t-1}(\mathbf{x}_i,\mathbf{x}_j)
    ]_{i,j=1}^q
\]
and define
\[
    \mathbf D_S
    =
    \operatorname{diag}\left(
        S_{n,t}(\mathbf{x}_1),
        \ldots,
        S_{n,t}(\mathbf{x}_q)
    \right).
\]
The covariance matrix induced by
$\widetilde{k}_{n,t}$ is
\[
\widetilde{\mathbf K}
=
\mathbf D_S
\mathbf K
\mathbf D_S.
\]
Since $\mathbf K\succeq0$, for every vector $\mathbf a\in\mathbb R^q$,
\[
    \mathbf a^\top
    \widetilde{\mathbf K}
    \mathbf a
    =
    (
        \mathbf D_S\mathbf a
    )^\top
    \mathbf K
    (
        \mathbf D_S\mathbf a
    )
    \ge0.
\]
Thus, $\widetilde{\mathbf K}\succeq0$ for every finite set $\mathbf X$, and $\widetilde{k}_{n,t}$ is a valid covariance kernel.

The FI-GP mean is defined to equal the local posterior mean. For the
marginal variance,
\[
\begin{aligned}
    \widetilde{\sigma}_{n,t}^2(\mathbf{x})
    &=
    \widetilde{k}_{n,t}(\mathbf{x},\mathbf{x})\\
    &=
    S_{n,t}^2(\mathbf{x})k_{n,t-1}(\mathbf{x},\mathbf{x})\\
    &=
    S_{n,t}^2(\mathbf{x})\sigma_{n,t-1}^2(\mathbf{x}).
\end{aligned}
\]
By Remark~\ref{rem:bounded-guidance} and $\lambda_t\ge0$, we have $S_{n,t}(\mathbf{x})\ge1$. Taking square roots gives
\[
    \widetilde{\sigma}_{n,t}(\mathbf{x})
    =
    S_{n,t}(\mathbf{x})\sigma_{n,t-1}(\mathbf{x}).
\]
\end{proof}

\subsection{Proof of Lemma~\ref{lem:instant-regret}}
\label{app:proof-instant-regret}

\begin{proof}
Fix an arbitrary agent $n$ and round $t$, and define
\[
    S_{n,t}(\mathbf{x})
    =
    1+\lambda_tG_{n,t}(\mathbf{x}).
\]
By the GUIDE-UCB selection rule in Eq.~\ref{eq:guide-ucb-rule},
\[
    \mu_{n,t-1}(\mathbf{x}_{n,t})
    +
    \sqrt{\beta_t}
    S_{n,t}(\mathbf{x}_{n,t})
    \sigma_{n,t-1}(\mathbf{x}_{n,t})
    \ge
    \mu_{n,t-1}(\mathbf{x}_n^*)
    +
    \sqrt{\beta_t}
    S_{n,t}(\mathbf{x}_n^*)
    \sigma_{n,t-1}(\mathbf{x}_n^*).
\]
Rearranging gives
\[
\begin{aligned}
    \mu_{n,t-1}(\mathbf{x}_n^*)
    -
    \mu_{n,t-1}(\mathbf{x}_{n,t})
    \le
    \sqrt{\beta_t}
    \Big[
        &
        S_{n,t}(\mathbf{x}_{n,t})
        \sigma_{n,t-1}(\mathbf{x}_{n,t})\\
        &-
        S_{n,t}(\mathbf{x}_n^*)
        \sigma_{n,t-1}(\mathbf{x}_n^*)
    \Big].
\end{aligned}
\]

Under Assumption~\ref{assump:local-confidence},
\[
    f_n(\mathbf{x}_n^*)
    \le
    \mu_{n,t-1}(\mathbf{x}_n^*)
    +
    \sqrt{\beta_t}
    \sigma_{n,t-1}(\mathbf{x}_n^*)
\]
and
\[
    f_n(\mathbf{x}_{n,t})
    \ge
    \mu_{n,t-1}(\mathbf{x}_{n,t})
    -
    \sqrt{\beta_t}
    \sigma_{n,t-1}(\mathbf{x}_{n,t}).
\]
Consequently,
\[
\begin{aligned}
    r_{n,t}
    &=
    f_n(\mathbf{x}_n^*)
    -
    f_n(\mathbf{x}_{n,t})\\
    &\le
    \mu_{n,t-1}(\mathbf{x}_n^*)
    -
    \mu_{n,t-1}(\mathbf{x}_{n,t})
    +
    \sqrt{\beta_t}
    \sigma_{n,t-1}(\mathbf{x}_n^*)
    +
    \sqrt{\beta_t}
    \sigma_{n,t-1}(\mathbf{x}_{n,t})\\
    &\le
    \sqrt{\beta_t}
    \Big[
        S_{n,t}(\mathbf{x}_{n,t})
        \sigma_{n,t-1}(\mathbf{x}_{n,t})
        -
        S_{n,t}(\mathbf{x}_n^*)
        \sigma_{n,t-1}(\mathbf{x}_n^*)\\
    &\hspace{35mm}
        +
        \sigma_{n,t-1}(\mathbf{x}_n^*)
        +
        \sigma_{n,t-1}(\mathbf{x}_{n,t})
    \Big].
\end{aligned}
\]
Substituting
$S_{n,t}(\mathbf{x})=1+\lambda_tG_{n,t}(\mathbf{x})$
yields
\[
\begin{aligned}
    r_{n,t}
    \le
    \sqrt{\beta_t}
    \Big[
        2\sigma_{n,t-1}(\mathbf{x}_{n,t})
        &+
        \lambda_tG_{n,t}(\mathbf{x}_{n,t})
        \sigma_{n,t-1}(\mathbf{x}_{n,t})\\
        &-
        \lambda_tG_{n,t}(\mathbf{x}_n^*)
        \sigma_{n,t-1}(\mathbf{x}_n^*)
    \Big],
\end{aligned}
\]
which proves Lemma~\ref{lem:instant-regret}.
\end{proof}

\subsection{Proof of Theorem~\ref{thm:hetero-regret}}
\label{app:proof-hetero-regret}

\begin{proof}
Fix an arbitrary agent $n$. By Lemma~\ref{lem:instant-regret}, and because $\lambda_t\ge0$ and $G_{n,t}(\mathbf{x}_n^*)\ge0$, discarding the final nonpositive term gives
\[
    r_{n,t}
    \le
    \sqrt{\beta_t}
    \left(
        2+
        \lambda_tG_{n,t}(\mathbf{x}_{n,t})
    \right)
    \sigma_{n,t-1}(\mathbf{x}_{n,t}).
\]
By Remark~\ref{rem:bounded-guidance},
$G_{n,t}(\mathbf{x}_{n,t})\le G_{\max}$.
Using
$\lambda_t=\lambda_{\max}/\sqrt t$
with $\lambda_{\max}\ge0$,
\[
    r_{n,t}
    \le
    \sqrt{\beta_t}
    \left(
        2+
        \frac{
            \lambda_{\max}G_{\max}
        }{
            \sqrt t
        }
    \right)
    \sigma_{n,t-1}(\mathbf{x}_{n,t}).
\]
Since $\{\beta_t\}_{t\ge1}$ is nondecreasing under Assumption~\ref{assump:local-confidence},
\[
    R_{n,T}
    \le
    \sqrt{\beta_T}
    \sum_{t=1}^T
    \left(
        2+
        \frac{
            \lambda_{\max}G_{\max}
        }{
            \sqrt t
        }
    \right)
    \sigma_{n,t-1}(\mathbf{x}_{n,t}).
\]
Applying Cauchy--Schwarz to the complete product gives
\[
\begin{aligned}
    R_{n,T}
    \le
    \sqrt{\beta_T}
    &\sqrt{
        \sum_{t=1}^T
        \left(
            2+
            \frac{
                \lambda_{\max}G_{\max}
            }{
                \sqrt t
            }
        \right)^2
    }\\
    &\times
    \sqrt{
        \sum_{t=1}^T
        \sigma_{n,t-1}^2(\mathbf{x}_{n,t})
    }.
\end{aligned}
\]
Applying Lemma~\ref{lem:info-gain} yields
\[
    R_{n,T}
    \le
    \sqrt{
        \beta_T
        C_\gamma
        \gamma_{n,T}
    }
    \sqrt{
        \sum_{t=1}^T
        \left(
            2+
            \frac{
                \lambda_{\max}G_{\max}
            }{
                \sqrt t
            }
        \right)^2
    }.
\]
The remaining coefficient satisfies
\[
\begin{aligned}
    \sum_{t=1}^T
    \left(
        2+
        \frac{
            \lambda_{\max}G_{\max}
        }{
            \sqrt t
        }
    \right)^2
    &=
    4T
    +
    4\lambda_{\max}G_{\max}
    \sum_{t=1}^T
    t^{-1/2}\\
    &\quad
    +
    \lambda_{\max}^2G_{\max}^2
    \sum_{t=1}^T
    t^{-1}\\
    &\le
    4T
    +
    4\lambda_{\max}G_{\max}
    (2\sqrt T-1)\\
    &\quad
    +
    \lambda_{\max}^2G_{\max}^2
    (1+\log T),
\end{aligned}
\]
where
\[
    \sum_{t=1}^T
    t^{-1/2}
    \le
    2\sqrt T-1,
    \qquad
    \sum_{t=1}^T
    t^{-1}
    \le
    1+\log T.
\]
Combining these inequalities gives
\[
    R_{n,T}
    \le
    \sqrt{
        \beta_T
        C_\gamma
        \gamma_{n,T}
    }
    \sqrt{
        4T
        +
        4\lambda_{\max}G_{\max}(2\sqrt T-1)
        +
        \lambda_{\max}^2G_{\max}^2(1+\log T)
    }.
\]
For fixed $\lambda_{\max}$ and $G_{\max}$, the guidance-dependent
$\sqrt{T}$ and $\log T$ terms are lower order than the leading
$T$ term inside the second square root. Therefore,
\[
    R_{n,T}
    =
    \mathcal O\left(
        \sqrt{
            T\beta_T\gamma_{n,T}
        }
    \right).
\]
\end{proof}

\subsection{Constant Intervention Strength}
\label{app:constant-intervention}

For a fixed intervention strength
$\lambda_t\equiv\lambda$
with $\lambda\ge0$, Lemma~\ref{lem:instant-regret} and Remark~\ref{rem:bounded-guidance} imply
\[
    r_{n,t}
    \le
    \sqrt{\beta_t}
    (2+\lambda G_{\max})
    \sigma_{n,t-1}(\mathbf{x}_{n,t}).
\]
Using the monotonicity of $\beta_t$, Cauchy--Schwarz, and Lemma~\ref{lem:info-gain},
\[
\begin{aligned}
    R_{n,T}
    &\le
    (2+\lambda G_{\max})
    \sqrt{\beta_T}
    \sum_{t=1}^T
    \sigma_{n,t-1}(\mathbf{x}_{n,t})\\
    &\le
    (2+\lambda G_{\max})
    \sqrt{
        T\beta_T
        \sum_{t=1}^T
        \sigma_{n,t-1}^2(\mathbf{x}_{n,t})
    }\\
    &\le
    (2+\lambda G_{\max})
    \sqrt{
        T\beta_T
        C_\gamma
        \gamma_{n,T}
    }.
\end{aligned}
\]

Hence, for any fixed $\lambda\geq0$, the intervention changes only the constant factor in the regret bound and remains sublinear whenever $\beta_T\gamma_{n,T}=o(T)$.
\subsection{Proof of Proposition~\ref{prop:effective-threshold}}
\label{app:proof-effective-threshold}

\begin{proof}
Fix an arbitrary agent $n$ and round $t$. Let $\mathbf{x}_{n,t}$ be a selected suboptimal point, so that
\[
    \Delta_n(\mathbf{x}_{n,t})
    =
    f_n(\mathbf{x}_n^*)
    -
    f_n(\mathbf{x}_{n,t})
    =
    r_{n,t}
    >
    0.
\]
Starting from Lemma~\ref{lem:instant-regret},
\[
\begin{aligned}
    \Delta_n(\mathbf{x}_{n,t})
    \le
    \sqrt{\beta_t}
    \Big[
        2\sigma_{n,t-1}(\mathbf{x}_{n,t})
        &+
        \lambda_tG_{n,t}(\mathbf{x}_{n,t})
        \sigma_{n,t-1}(\mathbf{x}_{n,t})\\
        &-
        \lambda_tG_{n,t}(\mathbf{x}_n^*)
        \sigma_{n,t-1}(\mathbf{x}_n^*)
    \Big].
\end{aligned}
\]
Rearranging gives
\[
\begin{aligned}
    \sqrt{\beta_t}
    \left(
        2+
        \lambda_tG_{n,t}(\mathbf{x}_{n,t})
    \right)
    \sigma_{n,t-1}(\mathbf{x}_{n,t})
    \ge
    \Delta_n(\mathbf{x}_{n,t})
    +
    \lambda_t\sqrt{\beta_t}
    G_{n,t}(\mathbf{x}_n^*)
    \sigma_{n,t-1}(\mathbf{x}_n^*).
\end{aligned}
\]
Because $\beta_t>0$ and $\lambda_t\ge0$, the denominator below is strictly positive. Therefore,
\[
    \sigma_{n,t-1}(\mathbf{x}_{n,t})
    \ge
    \frac{
        \Delta_n(\mathbf{x}_{n,t})
        +
        \lambda_t\sqrt{\beta_t}
        G_{n,t}(\mathbf{x}_n^*)
        \sigma_{n,t-1}(\mathbf{x}_n^*)
    }{
        \sqrt{\beta_t}
        \left(
            2+
            \lambda_tG_{n,t}(\mathbf{x}_{n,t})
        \right)
    }.
\]

If
$G_{n,t}(\mathbf{x}_n^*)\ge c$
and
$G_{n,t}(\mathbf{x}_{n,t})\le\varepsilon_G$
for some $c>0$ and $\varepsilon_G\ge0$, then under
$\lambda_t=\lambda_{\max}/\sqrt t$
with $\lambda_{\max}\ge0$,
\[
    \sigma_{n,t-1}(\mathbf{x}_{n,t})
    \ge
    \frac{
        \Delta_n(\mathbf{x}_{n,t})
        +
        \frac{
            \lambda_{\max}c
        }{
            \sqrt t
        }
        \sqrt{\beta_t}
        \sigma_{n,t-1}(\mathbf{x}_n^*)
    }{
        \sqrt{\beta_t}
        \left(
            2+
            \frac{
                \lambda_{\max}\varepsilon_G
            }{
                \sqrt t
            }
        \right)
    }.
\]
This proves Proposition~\ref{prop:effective-threshold}.
\end{proof}

\paragraph{Comparison with standard GP-UCB.}
Setting $\lambda_t=0$ in Eq.~\ref{eq:exact-threshold} recovers the standard GP-UCB selection threshold $\Delta_n(\mathbf{x}_{n,t})/(2\sqrt{\beta_t})$.
For $\lambda_t>0$, direct comparison with the bound under $G_{n,t}(\mathbf{x}_n^*)\geq c$ and $G_{n,t}(\mathbf{x}_{n,t})\leq\varepsilon_G$ shows that the GUIDE-UCB threshold is strictly larger whenever
\[
2\sqrt{\beta_t}\,
c\,
\sigma_{n,t-1}(\mathbf{x}_n^*)
>
\varepsilon_G
\Delta_n(\mathbf{x}_{n,t}).
\]
Thus, the increase in the selection threshold occurs when guidance near the agent's optimum is sufficiently strong relative to the guidance assigned to the competing suboptimal point.

\subsection{Geometric Interpretation of Informative Guidance}
\label{app:effective-guidance-natural}

The informative-guidance condition admits a direct geometric interpretation. If a received component $p^\dagger$ lies within Mahalanobis radius $\rho$ of the agent's optimum, then
$G_{n,t}(\mathbf{x}_n^*)\geq
\omega_{n,t,p^\dagger}\exp(-\rho^2/2)$.
Conversely, if a suboptimal point $\mathbf{x}$ has Mahalanobis distance at least $R_{\mathrm{tail}}$ from every received component, then, because the retained weights sum to at most one,
$G_{n,t}(\mathbf{x})\leq\exp(-R_{\mathrm{tail}}^2/2)$.
Hence, the conditions $G_{n,t}(\mathbf{x}_n^*)\geq c$ and
$G_{n,t}(\mathbf{x})\leq\varepsilon_G$ arise naturally when the received guidance places substantial support near the agent's optimum while the suboptimal candidate lies in the tails of the received components.

\section{Benchmarks, Heterogeneity, Metrics, and Experimental Settings}
\label{app:benchmarks}

\subsection{Synthetic Benchmarks}

Table~\ref{tab:appendix_benchmarks} summarizes the 12 synthetic functions. We convert minimization functions to maximization by negating the function value where necessary. The implementation follows the corresponding BoTorch test-function implementations where available~\citep{balandat2020botorch}, together with the Sphere, Weierstrass, Ellipsoid, and Zakharov implementations used in our code. The dimensions and search domains used in the experiments are reported explicitly in Table~\ref{tab:appendix_benchmarks}.

\begin{table}[h]
    \caption{Synthetic benchmark functions. All objectives are evaluated as maximization problems.}
    \label{tab:appendix_benchmarks}
    \centering
    \setlength{\tabcolsep}{4.0pt}
    \renewcommand{\arraystretch}{1.08}
    \small
    \resizebox{0.8\columnwidth}{!}{%
    \begin{tabular}{lccll}
    \toprule
    \textbf{Function} & \textbf{Dim.} & \textbf{Domain} & \textbf{Maximum} & \textbf{Landscape} \\
    \midrule
    Ackley          & 10 & $[-32.768,32.768]^{10}$  & $0$      & Multimodal \\
    Levy            & 10 & $[-10,10]^{10}$          & $0$      & Multimodal \\
    Griewank        & 10 & $[-600,600]^{10}$        & $0$      & Multimodal \\
    Rastrigin       & 10 & $[-5.12,5.12]^{10}$      & $0$      & Highly multimodal \\
    Weierstrass     & 10 & $[-0.5,0.5]^{10}$        & $0$      & Irregular multimodal \\
    Ellipsoid       & 10 & $[-5.12,5.12]^{10}$      & $0$      & Ill-conditioned \\
    Sphere          & 10 & $[-5,5]^{10}$            & $0$      & Unimodal \\
    Zakharov        & 10 & $[-5,5]^{10}$            & $0$      & Plate-shaped \\
    Rosenbrock      & 10 & $[-2.048,2.048]^{10}$    & $0$      & Narrow curved valley \\
    Michalewicz     & 10 & $[0,\pi]^{10}$           & $9.66$   & Multimodal with steep valleys \\
    Powell          & 10 & $[-5,5]^{10}$            & $0$      & Non-separable \\
    Styblinski--Tang& 10 & $[-5,5]^{10}$            & $391.66$ & Multimodal, separable \\
    \bottomrule
    \end{tabular}%
    }
\end{table}

For Powell, we optimize a 10-dimensional search vector as implemented in our experiments; because the objective is evaluated in complete four-variable blocks, the first eight coordinates are active when $d=10$, while the remaining two do not affect the function value.

All synthetic BO variables are represented in normalized search coordinates before local modeling and server-side component merging. The original domains in Table~\ref{tab:appendix_benchmarks} define the underlying black-box functions, while the merging threshold $\delta_{\mathrm{merge}}$ is applied in the normalized search space.

\subsection{Real-world federated optimization tasks}
\label{app:realworld_tasks}

\paragraph{Landmine Detection.}
The Landmine dataset~\citep{xue2007multi} contains 29 binary classification tasks corresponding to 29 landmine fields, with each field treated as one FBO agent.
For each field, the BO objective is evaluated using fixed three-fold stratified cross-validation on the local training data. Each queried configuration trains a support vector machine (SVM) with a radial basis function (RBF) kernel, with feature standardization
fitted separately within each training fold. The optimized hyperparameters are
\[
C_{\mathrm{SVM}}\in[10^{-4},10],
\qquad
\gamma_{\mathrm{RBF}}\in[10^{-2},10],
\]
both decoded on logarithmic scales. The BO objective is the mean area under the receiver operating characteristic curve (AUC) across the cross-validation folds.

\paragraph{Activity Recognition Using Mobile Phone Sensors.}
Following FTS~\citep{dai2020federated}, we use the Human Activity Recognition Using Smartphones dataset from the University of California, Irvine (UCI) Machine Learning Repository. The dataset contains measurements from 30 subjects performing six activities, represented by
561 features. Each subject is treated as one FBO agent. The official partitions are merged, and the samples of each subject are divided into a stratified $50/50$ local training and validation split for each experimental replication. The features are further standardized using statistics from the local training split only.

Each BO query specifies three logistic regression hyperparameters:
\[
B\in[20,60],\qquad
\lambda_{\mathrm{L2}}\in[10^{-6},1],\qquad
\eta\in[10^{-2},10^{-1}],
\]
where $B$ is the mini-batch size, $\lambda_{\mathrm{L2}}$ is the L2 regularization parameter, and $\eta$ is the learning rate. Each query trains a fresh linear classifier with 561 inputs and 6 outputs for 100 epochs using Stochastic Gradient Descent and cross-entropy loss. The BO objective is validation accuracy.

\paragraph{FedHPO.}
We use FedHPO-Bench~\citep{wang2023fedhpo} to construct a federated hyperparameter optimization benchmark. Cora, CiteSeer, and PubMed are treated as three outer FBO agents, with each agent corresponding to a separate Graph Convolutional Network (GCN) hyperparameter optimization
task under Federated Averaging (FedAvg). BO operates in a normalized search space, and each query is decoded according to the official FedHPO-Bench configuration space. Objective values are obtained from the official surrogate models using validation accuracy. We use full client
participation and the highest available training-round fidelity.

\subsection{Synthetic Task Heterogeneity and Empirical Similarity}
\label{app:heterogeneity}

\paragraph{Task construction.}
Generating related benchmark instances through transformations of a common base function is standard practice in black-box optimization. In particular, BBOB and COCO use search-space translations and rotations to generate different instances of benchmark functions \citep{bbob2019,hansen2021coco}. Similar constructions have also been adopted in collaborative BO. For example,~\citet{chen2025multi} introduces agent heterogeneity by translating a common base function,
\begin{equation*}
    f_n(\mathbf x)
    =
    f_{\mathrm{base}}(\mathbf x+\boldsymbol\psi_n),
\end{equation*}
where $\boldsymbol\psi_n$ is sampled from a ball whose radius is proportional to the search-domain width.

Following this general principle, we combine agent-specific translations and rotations to control the degree of task heterogeneity. All variables are first represented in the normalized BO domain. For agent $n$, we define
\begin{equation*}
    f_n(\mathbf x)
    =
    f_{\mathrm{base}}
    \left(
        \mathbf R_n(\mathbf x-\mathbf z_n)
    \right),
\end{equation*}
where $\mathbf z_n$ controls the displacement of the landscape and $\mathbf R_n$ applies a rotation that, for non-rotationally-invariant functions, also changes the coordinate interactions.

Let $\mathbf l$ and $\mathbf u$ denote the lower and upper bounds of the normalized search domain. The shift vector is sampled as
\begin{equation*}
    \mathbf z_n
    \sim
    \mathcal N
    \left(
        \mathbf 0,
        \left(
            \frac{
                \delta_{\mathrm{shift}}
                \|\mathbf u-\mathbf l\|_2
            }{\sqrt d}
        \right)^2
        \mathbf I
    \right).
\end{equation*}
Hence,
\begin{equation*}
    \sqrt{\mathbb E\|\mathbf z_n\|_2^2}
    =
    \delta_{\mathrm{shift}}
    \|\mathbf u-\mathbf l\|_2,
\end{equation*}
so $\delta_{\mathrm{shift}}$ directly controls the root-mean-square shift magnitude relative to the domain scale.

For the rotation, we draw $\mathbf A_n\in\mathbb R^{d\times d}$ with independent standard Gaussian entries and construct the skew-symmetric matrix
\begin{equation*}
    \mathbf H_n
    =
    \mathbf A_n-\mathbf A_n^\top.
\end{equation*}
We then define
\begin{equation*}
    \mathbf R_n
    =
    \exp
    \left(
        \delta_{\mathrm{rot}}\mathbf H_n
    \right).
\end{equation*}
Since the matrix exponential of a skew-symmetric matrix is orthogonal,
$\mathbf R_n$ is a valid rotation matrix. The parameter
$\delta_{\mathrm{rot}}$ controls the strength of the rotation.

We consider three regimes:
\begin{align*}
    \text{Level 1:}\quad
    &(\delta_{\mathrm{shift}},\delta_{\mathrm{rot}})
    =(0,0),\\
    \text{Level 2:}\quad
    &(\delta_{\mathrm{shift}},\delta_{\mathrm{rot}})
    =(0.05,0.1),\\
    \text{Level 3:}\quad
    &(\delta_{\mathrm{shift}},\delta_{\mathrm{rot}})
    =(0.3,1.0).
\end{align*}
Level~1 contains identical local objectives, Level~2 introduces mild heterogeneity, and Level~3 applies substantially stronger transformations. For every experimental replication, the same agent-specific transformations are used by all compared methods.

\paragraph{Empirical task similarity.}
The transformation parameters specify how tasks are generated but do not
directly quantify the similarity of the resulting objective landscapes.
We therefore complement them with a rank-based empirical similarity
measure.

For each benchmark $b$, heterogeneity level $\ell$, and experimental
replication $r$, we generate a common scrambled Sobol probe set
\begin{equation*}
    \mathcal X_{\mathrm{probe}}
    =
    \{\mathbf x_q\}_{q=1}^{Q},
    \qquad Q=4096,
\end{equation*}
over the normalized search domain using a fixed Sobol seed of zero.
For agent $n$, we evaluate its noise-free latent objective on all probe
points and form
\begin{equation*}
    \mathbf v_{n}^{(b,\ell,r)}
    =
    \left[
        f_n(\mathbf x_1),
        \ldots,
        f_n(\mathbf x_Q)
    \right].
\end{equation*}

For every pair of agents $n<m$, their similarity is measured by the
Spearman rank correlation
\begin{equation*}
    \rho_{n,m}^{(b,\ell,r)}
    =
    \rho_{\mathrm S}
    \left(
        \mathbf v_n^{(b,\ell,r)},
        \mathbf v_m^{(b,\ell,r)}
    \right).
\end{equation*}
The mean similarity of one benchmark realization is
\begin{equation*}
    \mathrm{Sim}_{b,\ell,r}
    =
    \frac{2}{N(N-1)}
    \sum_{n<m}
    \rho_{n,m}^{(b,\ell,r)}.
\end{equation*}
Finally, we average over the $B=12$ synthetic benchmarks and
$R=10$ experimental realizations:
\begin{equation*}
    \overline{\mathrm{Sim}}_{\ell}
    =
    \frac{1}{BR}
    \sum_{b=1}^{B}
    \sum_{r=1}^{R}
    \mathrm{Sim}_{b,\ell,r}.
\end{equation*}

Spearman correlation measures agreement in the relative ranking of common
candidate locations and is insensitive to differences in objective scale.
We use it only to characterize the generated task heterogeneity.

\begin{table}[ht]
    \caption{
    Mean pairwise Spearman correlation between the 16 agent objectives. Each benchmark entry is averaged over 10 task realizations, and Overall additionally averages across the 12 benchmarks. Higher values indicate greater agreement in candidate rankings.
    }
    \label{tab:task_similarity}
    \centering
    \small
    \setlength{\tabcolsep}{5.0pt}
    \renewcommand{\arraystretch}{1.05}
    \begin{tabular}{lccc}
    \toprule
    \textbf{Benchmark}
    & \textbf{Level 1}
    & \textbf{Level 2}
    & \textbf{Level 3} \\
    \midrule
    Ackley             & 1.000 & 0.485 & 0.027 \\
    Levy               & 1.000 & 0.450 & 0.076 \\
    Griewank           & 1.000 & 0.861 & 0.155 \\
    Rastrigin          & 1.000 & 0.489 & 0.136 \\
    Weierstrass        & 1.000 & 0.464 & 0.005 \\
    Ellipsoid          & 1.000 & 0.740 & 0.120 \\
    Sphere             & 1.000 & 0.861 & 0.155 \\
    Zakharov           & 1.000 & 0.527 & 0.013 \\
    Rosenbrock         & 1.000 & 0.676 & 0.115 \\
    Michalewicz        & 1.000 & 0.034 & -0.001 \\
    Powell             & 1.000 & 0.635 & 0.046 \\
    Styblinski--Tang   & 1.000 & 0.466 & 0.124 \\
    \midrule
    \textbf{Overall}
                       & \textbf{1.000}
                       & \textbf{0.557}
                       & \textbf{0.081} \\
    \bottomrule
    \end{tabular}
\end{table}

The empirical similarities clearly separate the three heterogeneity regimes. Level~1 is exactly homogeneous, Level~2 retains moderate agreement in candidate rankings, and Level~3 exhibits only weak average agreement across agents.

\subsection{Evaluation Metrics}
\label{app:evaluation_metrics}

For a synthetic objective, let $f_n^*=f_n(\mathbf x_n^*)$ denote the optimal value of agent $n$, and let
$f_{n,t}^{\mathrm{best}}$ denote its best noise-free value observed through round $t$.
The simple regret is
\begin{equation*}
    s_{n,t}
    =
    f_n^*-f_{n,t}^{\mathrm{best}},
\end{equation*}
and the reported average simple regret is
\begin{equation*}
    \overline{s}_t
    =
    \frac{1}{N}
    \sum_{n=1}^{N}s_{n,t}.
\end{equation*}
Lower values are better.

For Landmine Detection, the reported performance is AUC averaged across the 29 agents.
For Activity Recognition and FedHPO, the reported performance is validation accuracy averaged across their agents.
Higher values are better for these real-world metrics.
For convergence plots, we use the lower-is-better metrics $1-\mathrm{AUC}$ and $1-\mathrm{Accuracy}$. The main text reports the
original AUC and accuracy values.

Average ranks are computed independently within each benchmark and then averaged over the corresponding benchmark set.

\subsection{Shared Experimental Settings}
\label{app:shared-settings}

All compared methods use identical search domains, per-agent evaluation budgets, and repetition seeds within each benchmark. For a fair comparison, all baseline methods use the hyperparameter settings recommended in their original papers or official implementations. For synthetic experiments, all methods except FTS-DE use the same Collaborative Latin Hypercube Sampling (CLHS) initialization following~\citet{chen2025multi}. FTS-DE retains the initialization prescribed by its Distributed Exploration mechanism~\citep{dai2021differentially}, which assigns agents to different search subregions during initialization. Synthetic experiments use $N=16$ agents, 30 initial observations per agent, 50 BO rounds, and Gaussian observation noise with standard deviation $0.1$.
The real-world experiments use their natural task counts: 29 agents for Landmine Detection, 30 agents for Activity Recognition, and 3 agents for FedHPO. Each experiment uses 30 initial evaluations per agent and 50 BO rounds, and the complete federated experiment is repeated independently 10 times.

\paragraph{GUIDE-FBO configuration.}
GUIDE-FBO introduces four principal tunable hyperparameters: the number of posterior-optimum samples $M$, the downlink packet size $P$, the merging threshold $\delta_{\mathrm{merge}}$, and the maximum guidance strength $\lambda_{\max}$. Unless otherwise stated, the same configuration is used across all benchmarks without benchmark-specific tuning. The remaining quantities in Table~\ref{tab:guide_hyperparameters} are fixed modeling or implementation choices.

\begin{table}[ht]
    \caption{
    Default configuration of GUIDE-FBO. The same settings are used across
    benchmarks unless otherwise stated.
    }
    \label{tab:guide_hyperparameters}
    \centering
    \small
    \setlength{\tabcolsep}{6pt}
    \renewcommand{\arraystretch}{1.08}
    \begin{tabular}{lc}
        \toprule
        \textbf{Parameter} & \textbf{Setting} \\
        \midrule
        Posterior-optimum samples $M$          & $500$ \\
        ORF features $D_{\mathrm{ORF}}$        & $500$ \\
        DPGMM covariance                       & Diagonal \\
        LCB coefficient $\kappa$               & $1.0$ \\
        Merge threshold $\delta_{\mathrm{merge}}$ & $0.05$ \\
        Downlink packet size $P$               & $5$ \\
        Maximum guidance strength $\lambda_{\max}$ & $1.0$ \\
        Guidance schedule $\lambda_t$          & $\lambda_{\max}/\sqrt{t}$ \\
        \bottomrule
    \end{tabular}
\end{table}

The server-side softmax temperature is determined automatically from the spread of the standardized component scores rather than treated as an independently tuned hyperparameter. Specifically, letting
\[
\Delta v_t
=
\max_{\ell} v'_\ell-\min_{\ell} v'_\ell,
\]
we use
\[
\tau_t
=
\begin{cases}
\max\{0.1\Delta v_t,\,10^{-2}\},
& \Delta v_t>10^{-6},\\
10^{-2},
& \text{otherwise}.
\end{cases}
\]
The local surrogate is implemented as a fixed-noise GP with an ARD Mat\'ern-$5/2$ kernel, whose hyperparameters are fitted by maximizing the exact marginal likelihood. The ORF posterior-path approximation follows the corresponding Mat\'ern-$5/2$ spectral construction described in Appendix~\ref{app:orf}.

\section{Complete Experimental Results and Convergence Trajectories}
\label{app:full-results}

This section reports the complete numerical results and convergence trajectories underlying Section~\ref{sec:main_results}. For each synthetic heterogeneity level, we report the final average simple regret
on all 12 benchmark functions and the corresponding trajectories over the 50 BO rounds. We then report the trajectories for the three real-world tasks. All curves show the mean over 10 independent runs, with shaded regions denoting $\pm 1$ standard error of the mean.

\subsection{Level 1: $\delta_{\mathrm{shift}}=0$ and $\delta_{\mathrm{rot}}=0$}

\begin{table*}[h]
    \caption{Complete synthetic results under Level 1 (Homogeneous). Entries report the final mean average simple regret over 10 independent runs. Lower is better. The best and second-best results in each row are shown in bold and underlined, respectively.}
    \label{tab:appendix_level1_results}
    \centering
    \setlength{\tabcolsep}{3.0pt}
    \renewcommand{\arraystretch}{1.07}
    \scriptsize
    \resizebox{\textwidth}{!}{%
    \begin{tabular}{lcccccccccccc}
    \toprule
    \textbf{Benchmark}
    & \textbf{TS} & \textbf{UCB} & \textbf{NEI}
    & \textbf{FTS} & \textbf{FTS-DE} & \textbf{FMTBO}
    & \textbf{CGP-TS} & \textbf{CGP-UCB} & \textbf{CGP-NEI}
    & \textbf{GUIDE-TS} & \textbf{GUIDE-UCB} & \textbf{GUIDE-NEI} \\
    \midrule
    Ackley             & 19.4789 & 19.3082 & 18.4994 & 19.4870 & 19.6941 & 19.0814 
                       & 19.5242 & 8.7013 & 6.5675 & 18.6182 & \best{4.4896} & \second{5.2015} \\
    Levy               & 24.9042 & 11.9439 & 10.5159 & 23.0893 & 23.5361 & 9.3873 
                       & 22.9011 & 3.7128 & 2.6747 & 11.6750 & \best{1.5865} & \second{1.9122} \\
    Griewank           & 47.1321 & 14.1117 & 13.3126 & 43.9996 & 15.4063 & 11.3637 
                       & 43.2063 & 6.3414 & \second{4.9622} & 23.2003 & \best{4.9517} & 6.3052 \\
    Rastrigin          & 105.6477 & 80.4791 & 79.9759 & 102.8319 & 103.4912 & 81.3840 
                       & 102.9572 & 62.0303 & 59.7812 & 90.3331 & \best{53.4890} & \second{54.6122} \\
    Weierstrass        & 13.5724 & 12.8962 & 11.8821 & 13.5131 & 13.5052 & 12.2728 
                       & 13.5183 & 8.0722 & 8.0706 & 12.4086 & \best{3.2670} & \second{3.7779} \\
    Ellipsoid          & 146.4074 & 117.5092 & 87.9959 & 138.8153 & 137.5396 & 81.6766 
                       & 140.1417 & 26.8736 & 24.0600 & 98.2097 & \best{1.7311} & \second{3.0529} \\
    Sphere             & 23.9937 & 14.8929 & 11.2488 & 23.4224 & 21.1340 & 9.8874 
                       & 22.6427 & 3.9151 & 2.9221 & 11.6519 & \best{0.6487} & \second{1.0040} \\
    Zakharov           & 60.3082 & 47.1536 & 46.0734 & 58.7893 & 60.6199 & 46.2480 
                       & 59.6063 & 48.3546 & 45.0008 & 44.7664 & \best{15.6653} & \second{22.4356} \\
    Rosenbrock         & 376.8339 & 110.0299 & 105.3660 & 363.1293 & 270.5447 & 111.8808 
                       & 340.1700 & 90.6186 & 73.0144 & 162.5873 & \best{45.4639} & \second{51.9817} \\
    Michalewicz        & 6.3605 & 5.8237 & 5.8773 & 6.4475 & 6.2028 & 6.2124 
                       & 6.4715 & \second{5.2263} & \best{5.0059} & 6.3292 & 5.2700 & 5.2798 \\
    Powell             & 677.0909 & 139.5000 & 123.8940 & 716.2145 & 707.2955 & 116.0112 
                       & 651.4241 & 195.2524 & 116.9121 & 329.9346 & \best{38.6966} & \second{38.7592} \\
    Styblinski--Tang    & 156.5159 & 107.2154 & 103.6553 & 156.2315 & 155.8989 & 103.3376 
                       & 153.1554 & \second{62.7248} & \best{56.2330} & 150.9906 & 99.6545 & 102.8457 \\
    \bottomrule
    \end{tabular}%
    }
\end{table*}

\begin{figure*}[h]
    \centering
    \begin{subfigure}[t]{0.25\textwidth}
        \centering
        \includegraphics[width=\linewidth]{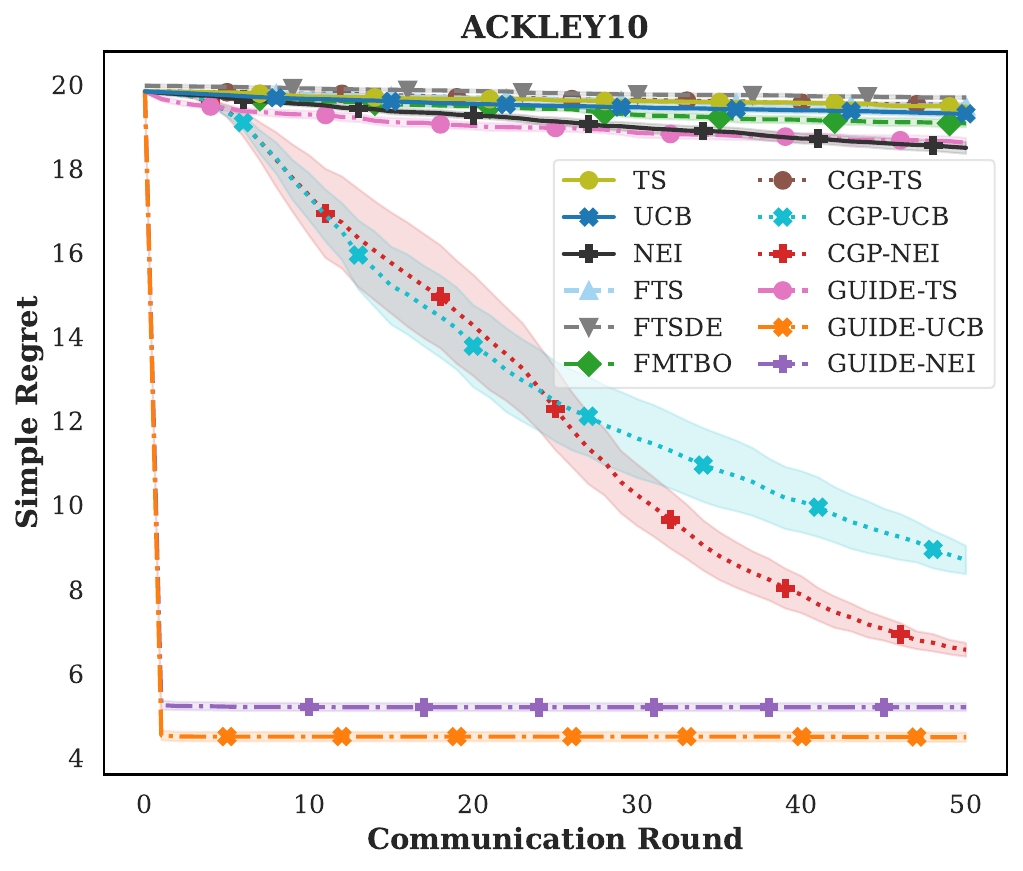}
        \caption{Ackley}
    \end{subfigure}\hfill
    \begin{subfigure}[t]{0.25\textwidth}
        \centering
        \includegraphics[width=\linewidth]{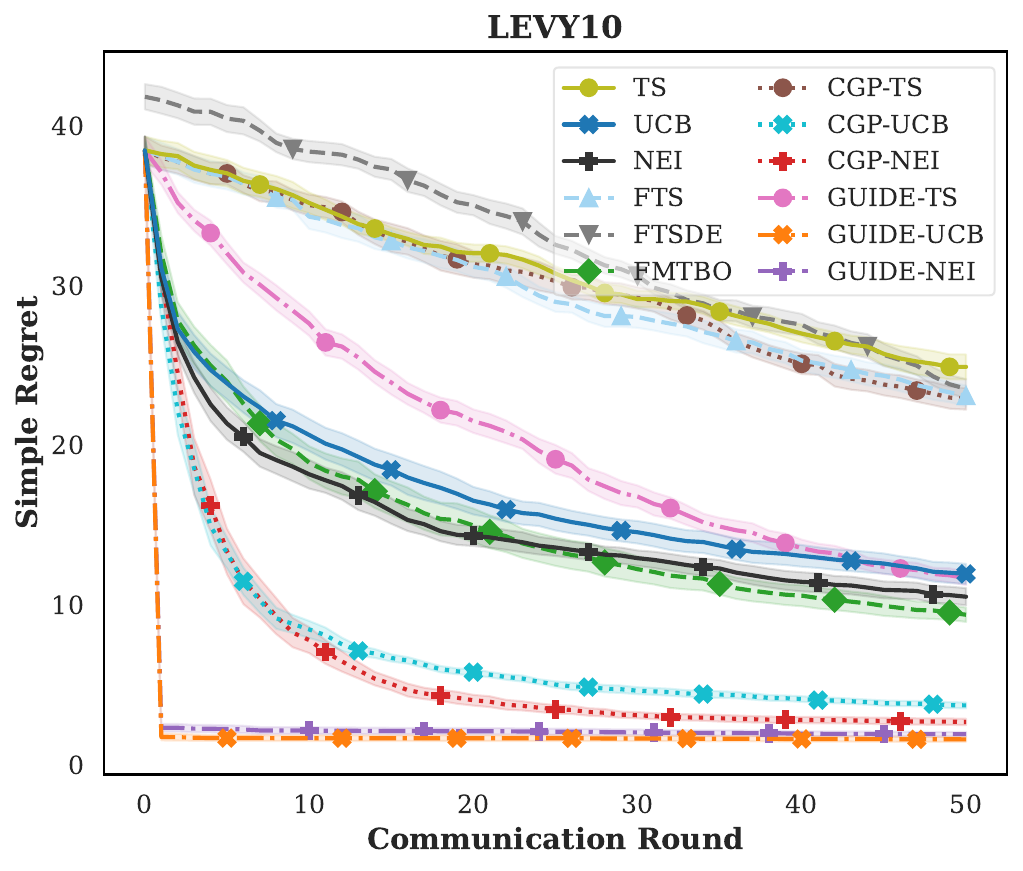}
        \caption{Levy}
    \end{subfigure}\hfill
    \begin{subfigure}[t]{0.25\textwidth}
        \centering
        \includegraphics[width=\linewidth]{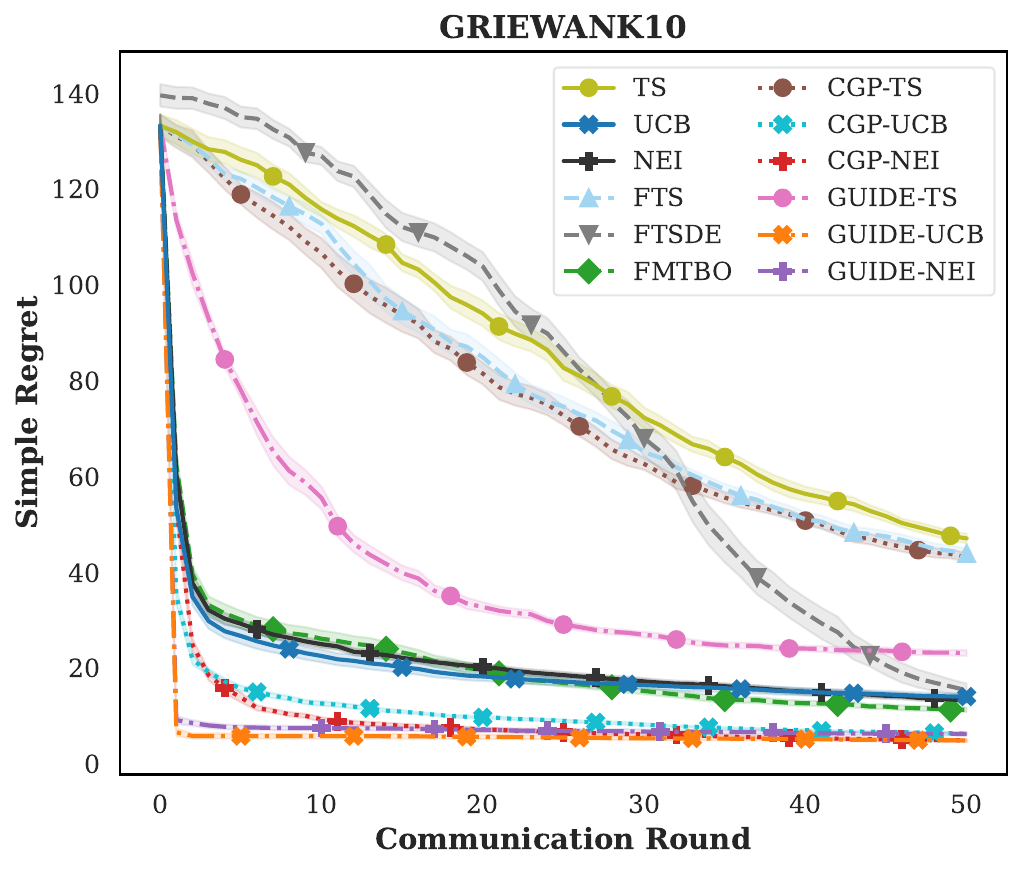}
        \caption{Griewank}
    \end{subfigure}\hfill
    \begin{subfigure}[t]{0.25\textwidth}
        \centering
        \includegraphics[width=\linewidth]{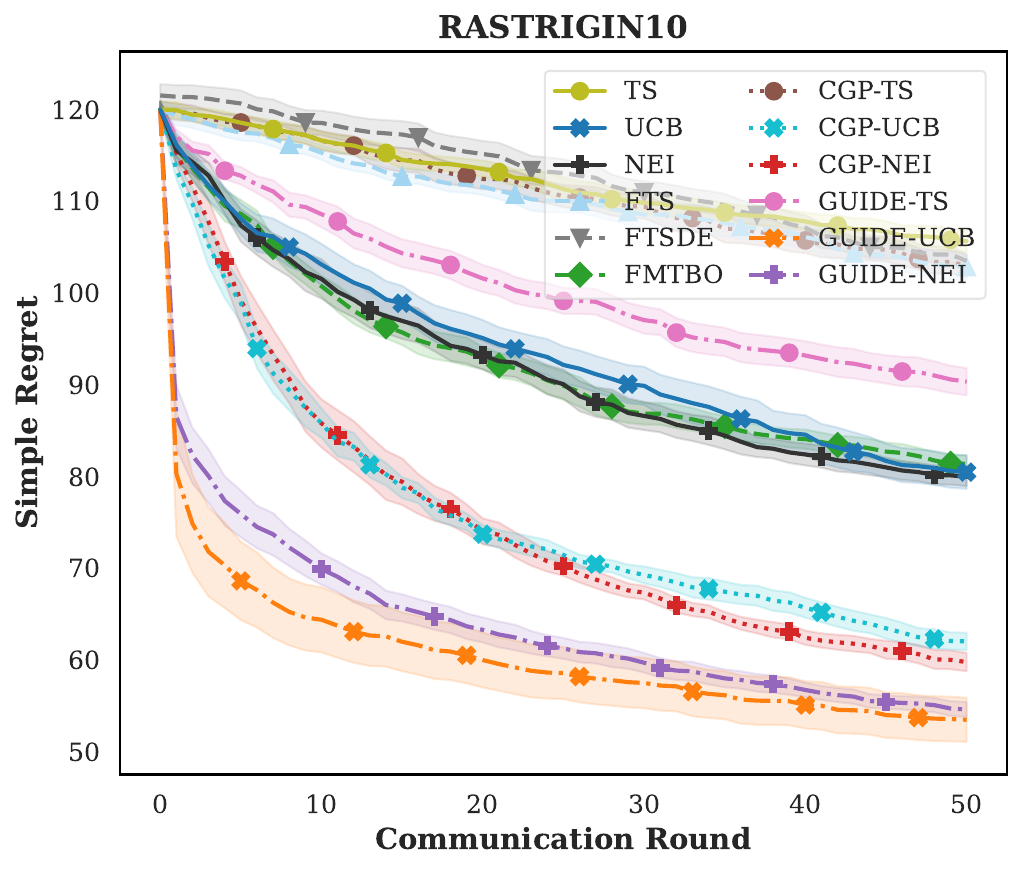}
        \caption{Rastrigin}
    \end{subfigure}\hfill

    \begin{subfigure}[t]{0.25\textwidth}
        \centering
        \includegraphics[width=\linewidth]{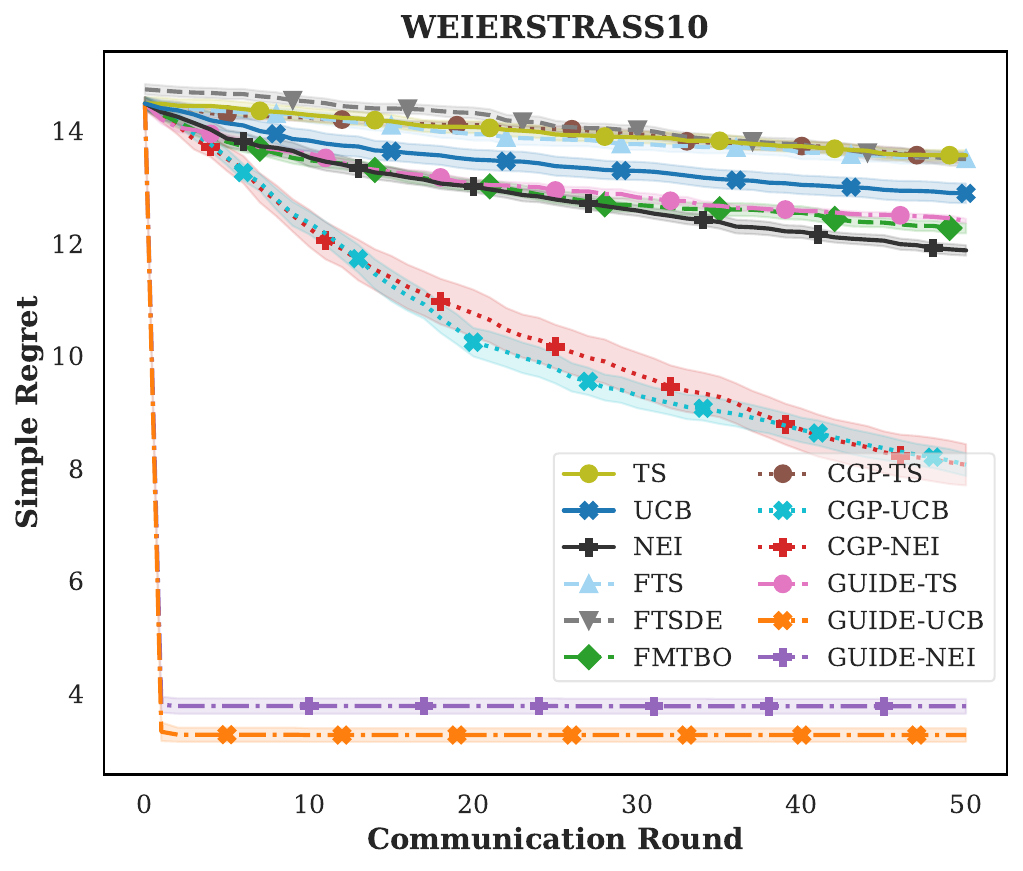}
        \caption{Weierstrass}
    \end{subfigure}\hfill
    \begin{subfigure}[t]{0.25\textwidth}
        \centering
        \includegraphics[width=\linewidth]{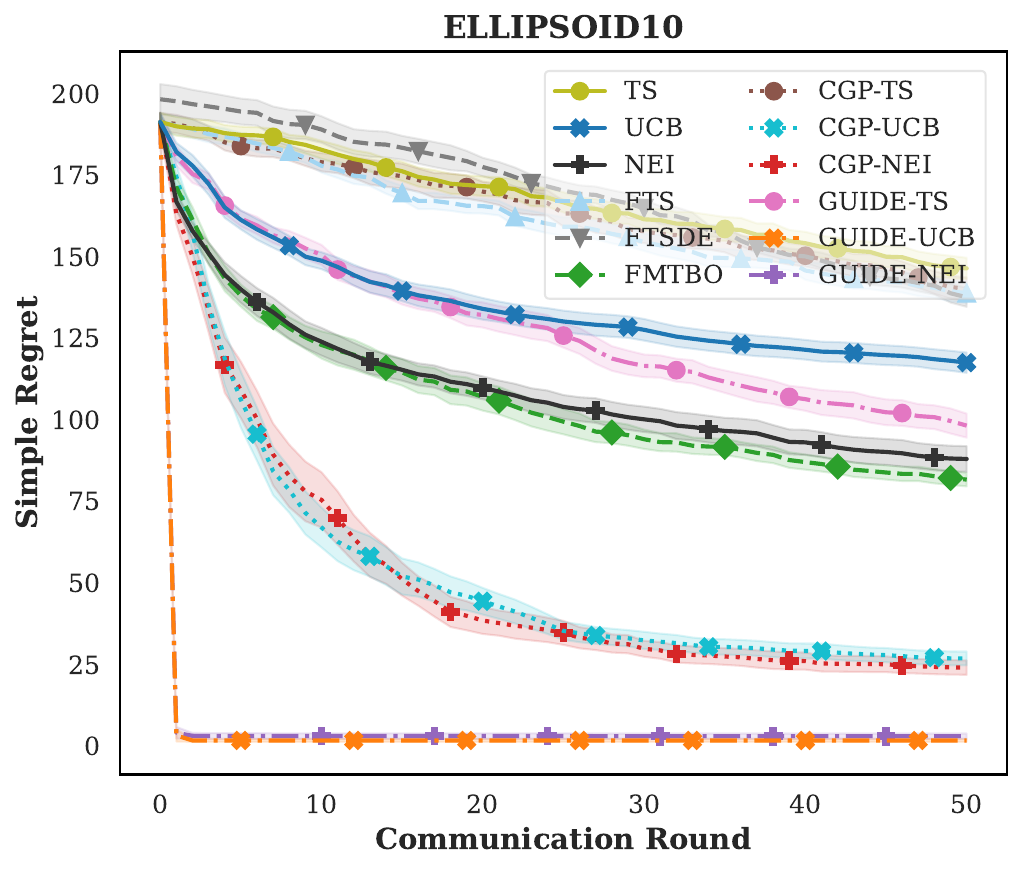}
        \caption{Ellipsoid}
    \end{subfigure}\hfill
    \begin{subfigure}[t]{0.25\textwidth}
        \centering
        \includegraphics[width=\linewidth]{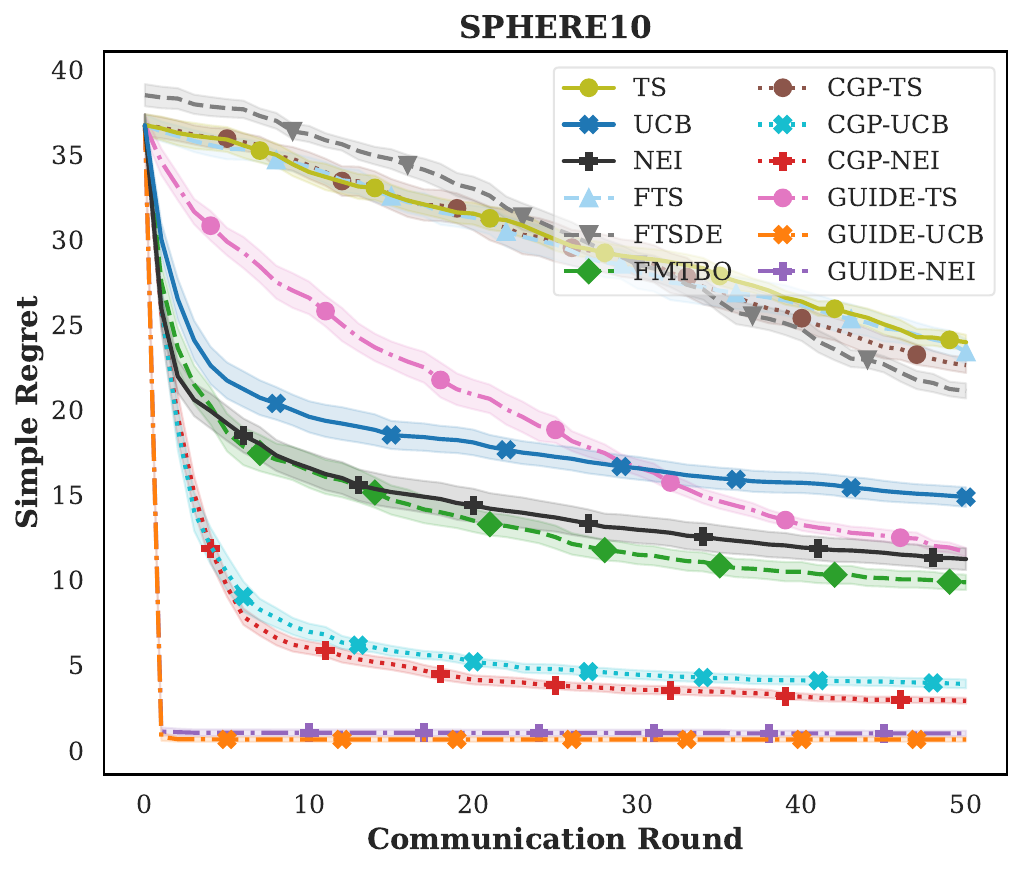}
        \caption{Sphere}
    \end{subfigure}\hfill
    \begin{subfigure}[t]{0.25\textwidth}
        \centering
        \includegraphics[width=\linewidth]{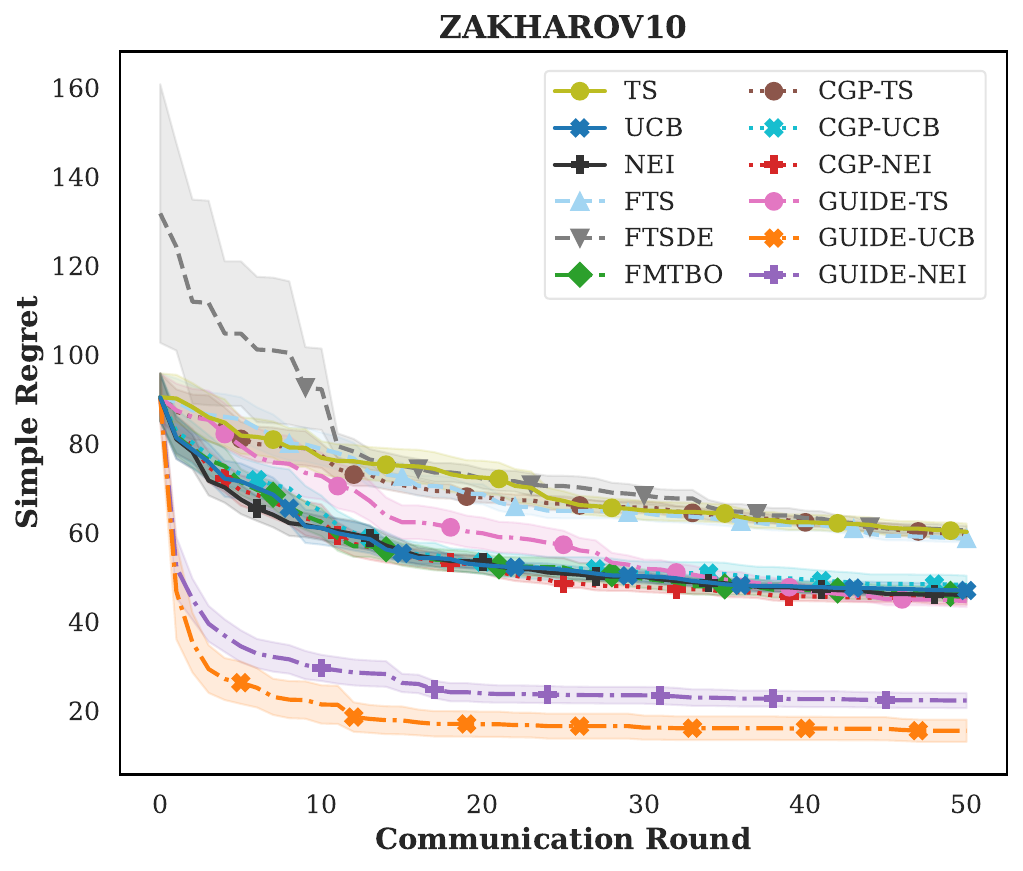}
        \caption{Zakharov}
    \end{subfigure}

    \begin{subfigure}[t]{0.25\textwidth}
        \centering
        \includegraphics[width=\linewidth]{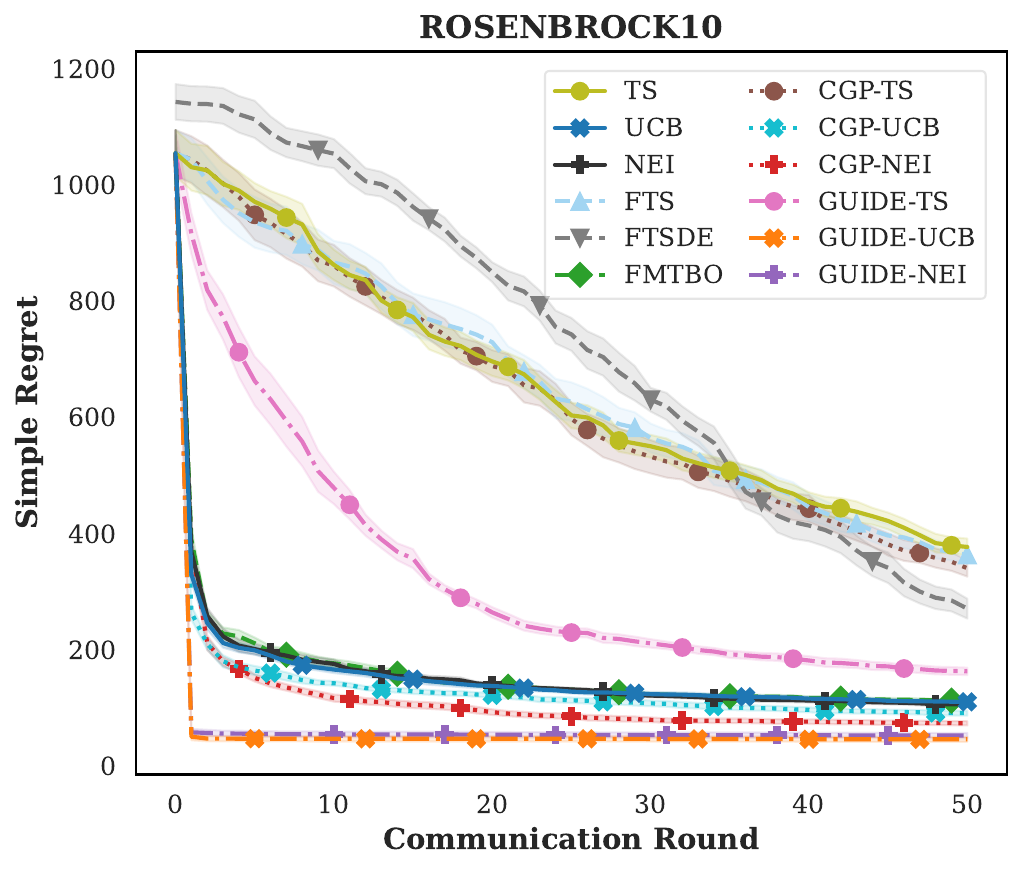}
        \caption{Rosenbrock}
    \end{subfigure}\hfill
    \begin{subfigure}[t]{0.25\textwidth}
        \centering
        \includegraphics[width=\linewidth]{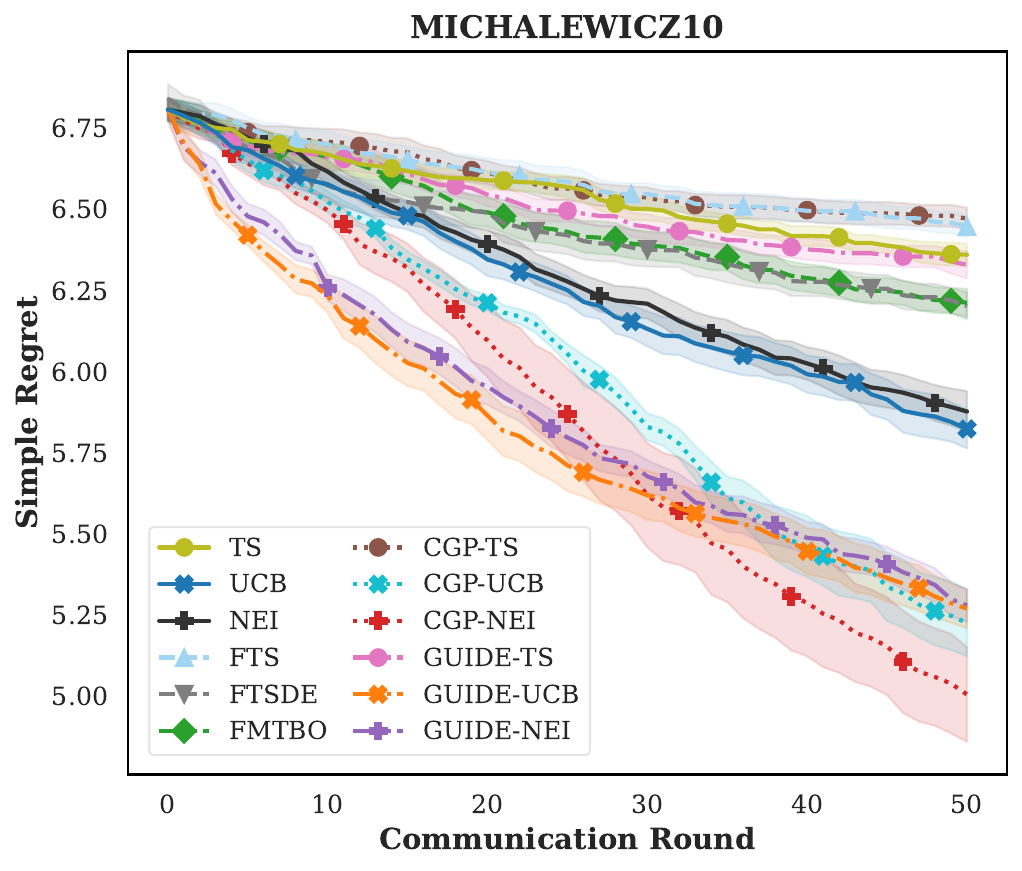}
        \caption{Michalewicz}
    \end{subfigure}\hfill
    \begin{subfigure}[t]{0.25\textwidth}
        \centering
        \includegraphics[width=\linewidth]{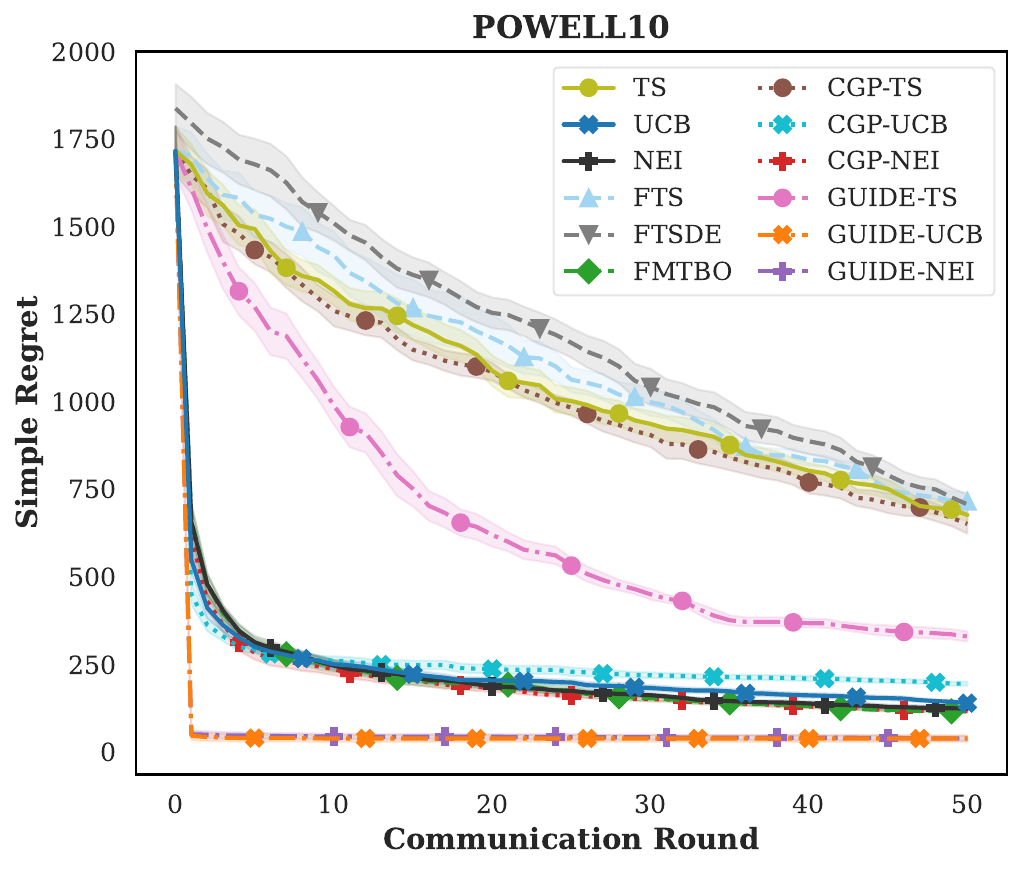}
        \caption{Powell}
    \end{subfigure}\hfill
    \begin{subfigure}[t]{0.25\textwidth}
        \centering
        \includegraphics[width=\linewidth]{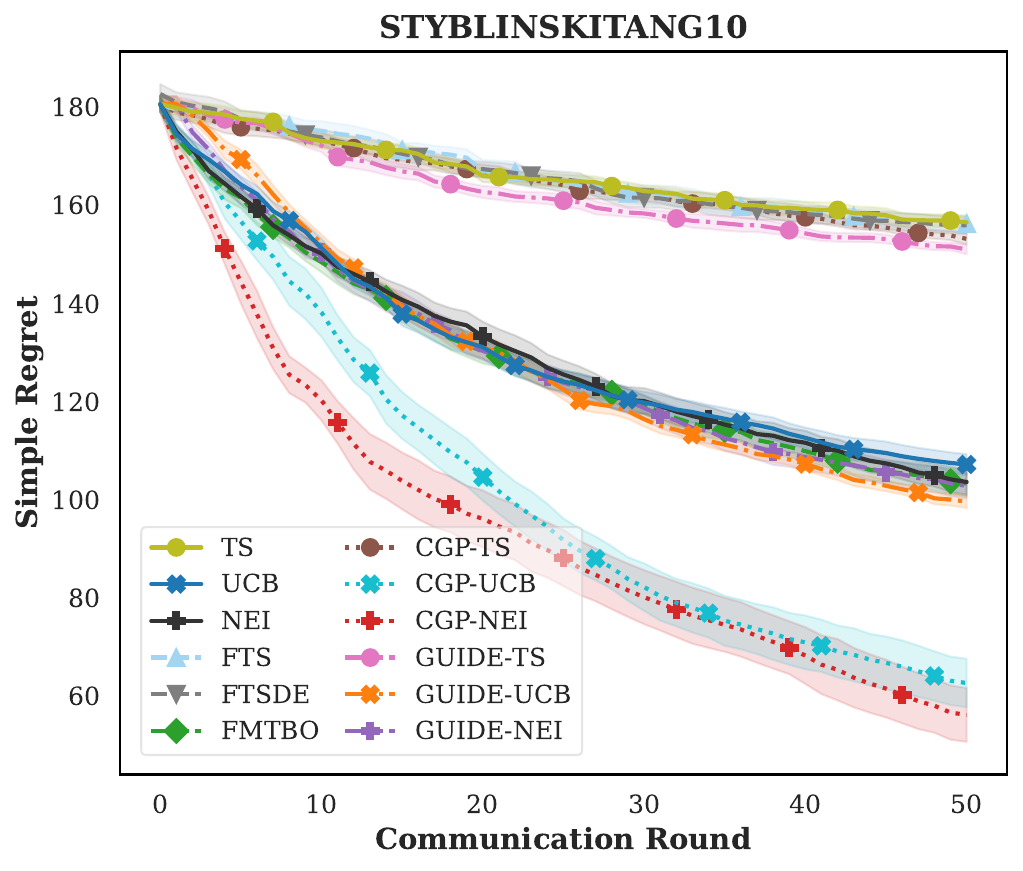}
        \caption{Styblinski--Tang}
    \end{subfigure}

    \caption{Convergence trajectories for all 12 synthetic functions under Level 1 (Homogeneous). Curves show the mean average simple regret over 10 independent runs, and shaded regions denote $\pm 1$ standard error of the mean across runs. Lower is better.}
    \label{fig:appendix_level1_curves}
\end{figure*}

Level~1 represents the fully aligned setting. All agents optimize the same
latent objective, and the empirical mean pairwise Spearman correlation is
$1.000$. GUIDE-UCB achieves the best average rank of $1.33$ and obtains the
lowest final average simple regret on $10$ of the $12$ benchmarks.
GUIDE-NEI ranks second overall with an average rank of $2.42$ and is among
the two best methods on nine benchmarks. Moreover, GUIDE-UCB, GUIDE-NEI,
and GUIDE-TS each outperform their matched independent acquisition rule on
all $12$ functions.

The numerical gaps are also large on several benchmarks.
UCB and GUIDE-UCB obtain final regrets of $117.5092$ and $1.7311$ on
Ellipsoid, $14.8929$ and $0.6487$ on Sphere, $47.1536$ and $15.6653$ on
Zakharov, $110.0299$ and $45.4639$ on Rosenbrock, and $139.5000$ and
$38.6966$ on Powell. GUIDE-UCB also obtains $4.4896$ on Ackley and
$1.5865$ on Levy, compared with $8.7013$ and $3.7128$ for CGP-UCB.
These results show that, when the local objectives are fully aligned,
exchanging distributions over the locations of the agents' optima can substantially reduce simple regret within a fixed evaluation budget.

The convergence trajectories in
Figure~\ref{fig:appendix_level1_curves} make this effect more apparent.
Except on Michalewicz and Styblinski--Tang, GUIDE-UCB and GUIDE-NEI
separate from most competing methods relatively early and remain in the
lowest-regret group over much of the subsequent optimization horizon. This
behavior is consistent with the mechanism of GUIDE. Because the agents optimize the same objective, their uploaded components are more likely to concentrate around the same promising regions. Server-side merging can therefore combine consistent information from different agents. FI-GP then increases
local posterior uncertainty selectively around these regions without changing
the local posterior mean. The resulting FI-GP makes these globally supported regions more attractive
to UCB and NEI when they remain plausible under the local posterior, which
can lead them to be evaluated earlier.

GUIDE-TS shows a smaller visual advantage. Unlike UCB and NEI, TS selects
each query by maximizing a random sample path drawn from FI-GP. Federated
guidance changes the distribution from which this path is sampled, but the
realized path can still attain its maximum elsewhere. This additional
sampling variability makes the effect of spatial guidance less consistent
across rounds. GUIDE-TS still improves upon independent TS on all $12$
functions, but does not show the same clear early separation as GUIDE-UCB
and GUIDE-NEI.

Michalewicz and Styblinski--Tang are the two main exceptions. Michalewicz
contains multiple steep valleys, so its high-quality regions are spatially
narrow. Region-level guidance is less useful when its spatial support does
not closely match the relevant valley. Since the tasks are identical at Level~1, the
high-potential points transferred by CGP are directly relevant to every
agent, which provides a favorable setting for point-level transfer.
CGP-NEI and CGP-UCB accordingly obtain $5.0059$ and $5.2263$, compared
with $5.2700$ and $5.2798$ for GUIDE-UCB and GUIDE-NEI.

Styblinski--Tang has many competing local basins with similar structure.
Because several local basins can remain plausible, a distribution over
optimum locations may provide a less precise signal than a directly
transferred high-quality point.
Under Level~1, where all agents optimize exactly the same objective, directly
transferring a good point can be more efficient because the same point is
useful to every agent. This favors the point-level transfer used by CGP.
CGP-NEI and CGP-UCB obtain final regrets of $56.2330$ and $62.7248$,
while GUIDE-UCB and GUIDE-NEI obtain $99.6545$ and $102.8457$. 
These two exceptions show that, even when the tasks are identical, the most
effective form of transferred information still depends on the geometry of
the objective.

\subsection{Level 2: $\delta_{\mathrm{shift}}=0.05$ and $\delta_{\mathrm{rot}}=0.1$}

\begin{table*}[h]
    \caption{Complete synthetic results under Level 2 (Mild heterogeneity). Entries report the final mean average simple regret over 10 independent runs. Lower is better. The best and second-best results in each row are shown in bold and underlined, respectively.}
    \label{tab:appendix_level2_results}
    \centering
    \setlength{\tabcolsep}{3.0pt}
    \renewcommand{\arraystretch}{1.07}
    \scriptsize
    \resizebox{\textwidth}{!}{%
    \begin{tabular}{lcccccccccccc}
    \toprule
    \textbf{Benchmark}
    & \textbf{TS} & \textbf{UCB} & \textbf{NEI}
    & \textbf{FTS} & \textbf{FTS-DE} & \textbf{FMTBO}
    & \textbf{CGP-TS} & \textbf{CGP-UCB} & \textbf{CGP-NEI}
    & \textbf{GUIDE-TS} & \textbf{GUIDE-UCB} & \textbf{GUIDE-NEI} \\
    \midrule
    Ackley             & 19.5136 & 19.4350 & 18.7628 & 19.5511 & 19.7171 & 18.9928 
                       & 19.5347 & 10.1460 & 10.9368 & 18.8906 & \best{8.4785} & \second{9.5379} \\
    Levy               & 24.3585 & 11.0529 & 10.1363 & 23.4060 & 25.8641 & 9.7791 
                       & 24.2127 & 4.3481 & 3.5741 & 12.7145 & \best{3.2816} & \second{3.3482} \\
    Griewank           & 44.4283 & 13.5953 & 11.9247 & 44.1421 & 26.2381 & 11.3470 
                       & 43.6216 & 7.1888 & \second{7.1002} & 24.8356 & \best{6.3000} & 7.8444 \\
    Rastrigin          & 101.5417 & 80.4771 & 78.9030 & 101.4711 & 101.7745 & 80.9836 
                       & 103.0048 & 64.6460 & 62.5459 & 88.0868 & \best{58.3082} & \second{60.1472} \\
    Weierstrass        & 13.4680 & 12.9344 & 12.1369 & 13.4517 & 13.6576 & 12.5151 
                       & 13.5531 & 8.9999 & 9.3717 & 12.7340 & \best{6.6966} & \second{6.8316} \\
    Ellipsoid          & 139.8685 & 113.6023 & 86.7074 & 139.5091 & 130.8015 & 78.4957 
                       & 136.6720 & 37.2115 & 29.1808 & 91.6564 & \best{13.1331} & \second{13.3723} \\
    Sphere             & 21.6883 & 14.2443 & 10.2896 & 23.9013 & 20.9801 & 9.0656 
                       & 21.6460 & 4.5267 & 3.3859 & 11.3485 & \best{2.2689} & \second{2.5063} \\
    Zakharov           & 62.9710 & 49.8750 & 48.2225 & 62.7988 & 62.9958 & 47.6923 
                       & 62.7797 & 50.1082 & 47.1746 & 48.4994 & \second{22.1717} & \best{21.8975} \\
    Rosenbrock         & 385.1965 & 97.8257 & 104.8257 & 365.1723 & 356.7614 & 107.8170 
                       & 392.6816 & 98.7769 & 83.6536 & 183.4833 & \best{55.2114} & \second{59.3597} \\
    Michalewicz        & 6.6497 & 6.1248 & 6.1995 & 6.5934 & 6.5978 & 6.4469 
                       & 6.6054 & 6.1800 & 6.1802 & 6.5636 & \best{5.8629} & \second{5.8783} \\
    Powell             & 701.0029 & 157.3605 & 126.0791 & 746.8994 & 853.3029 & 117.3149 
                       & 710.5802 & 205.5435 & 116.5086 & 348.1294 & \second{94.6703} & \best{74.3549} \\
    Styblinski--Tang    & 159.3047 & \second{114.4204} & 114.9073 & 160.0953 & 162.8046 & 119.9616 
                       & 157.4158 & 127.6079 & \best{109.7452} & 158.4985 & 123.2241 & 116.9887 \\
    \bottomrule
    \end{tabular}%
    }
\end{table*}

\begin{figure*}[h]
    \centering
    \begin{subfigure}[t]{0.25\textwidth}
        \centering
        \includegraphics[width=\linewidth]{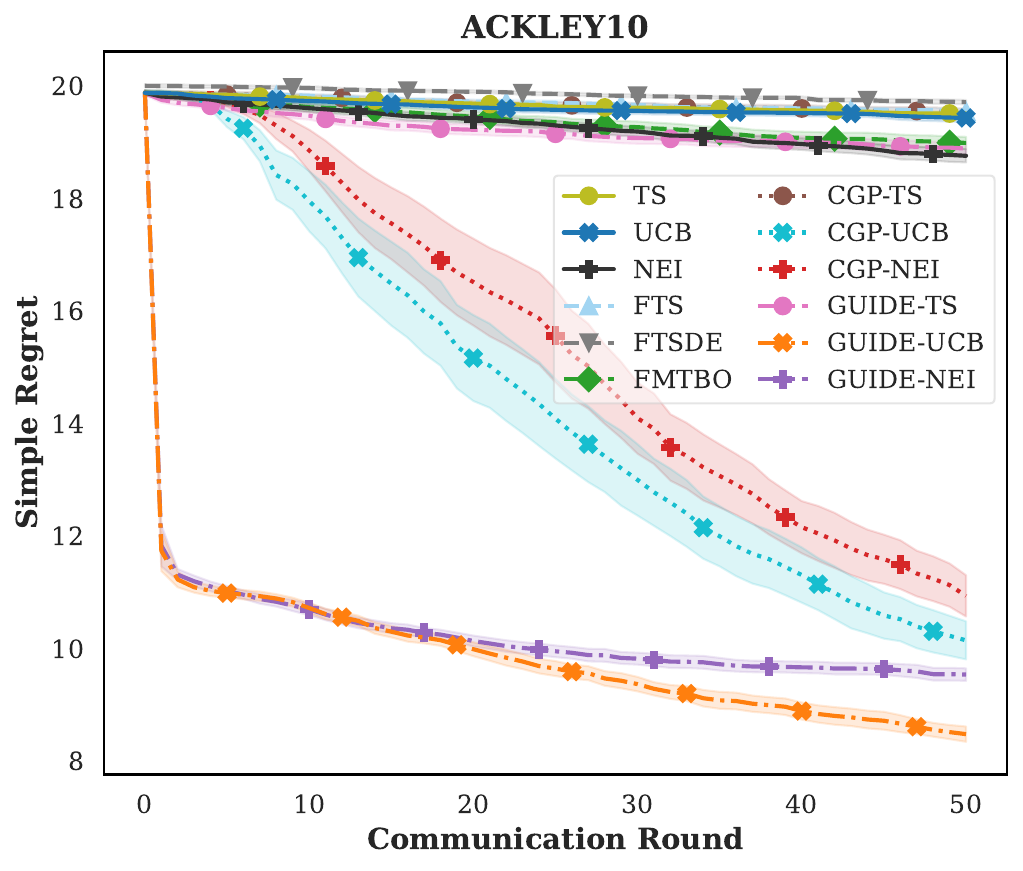}
        \caption{Ackley}
    \end{subfigure}\hfill
    \begin{subfigure}[t]{0.25\textwidth}
        \centering
        \includegraphics[width=\linewidth]{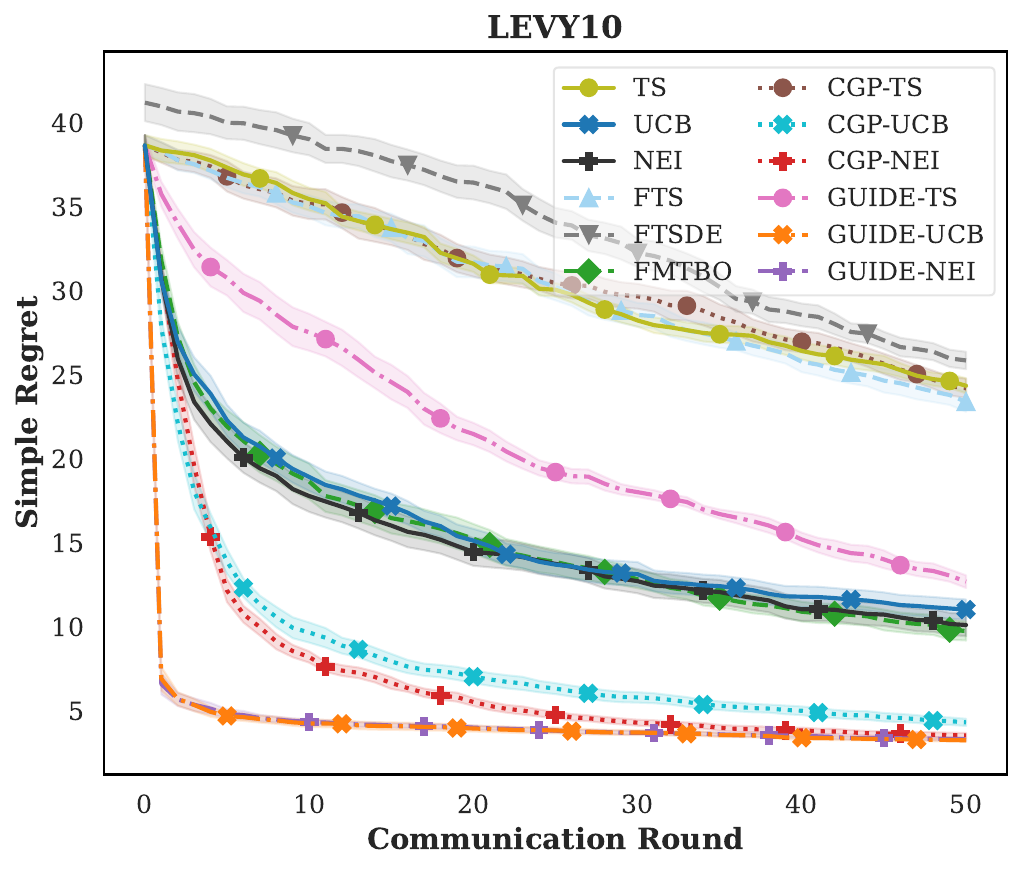}
        \caption{Levy}
    \end{subfigure}\hfill
    \begin{subfigure}[t]{0.25\textwidth}
        \centering
        \includegraphics[width=\linewidth]{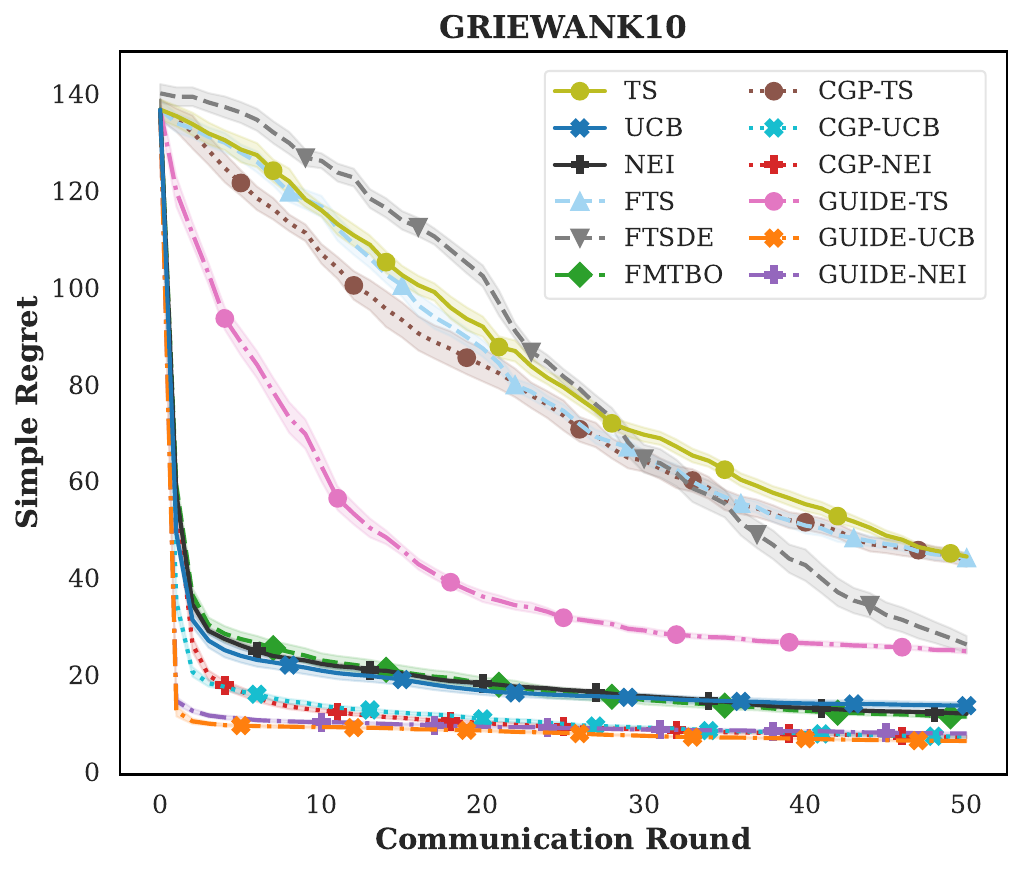}
        \caption{Griewank}
    \end{subfigure}\hfill
    \begin{subfigure}[t]{0.25\textwidth}
        \centering
        \includegraphics[width=\linewidth]{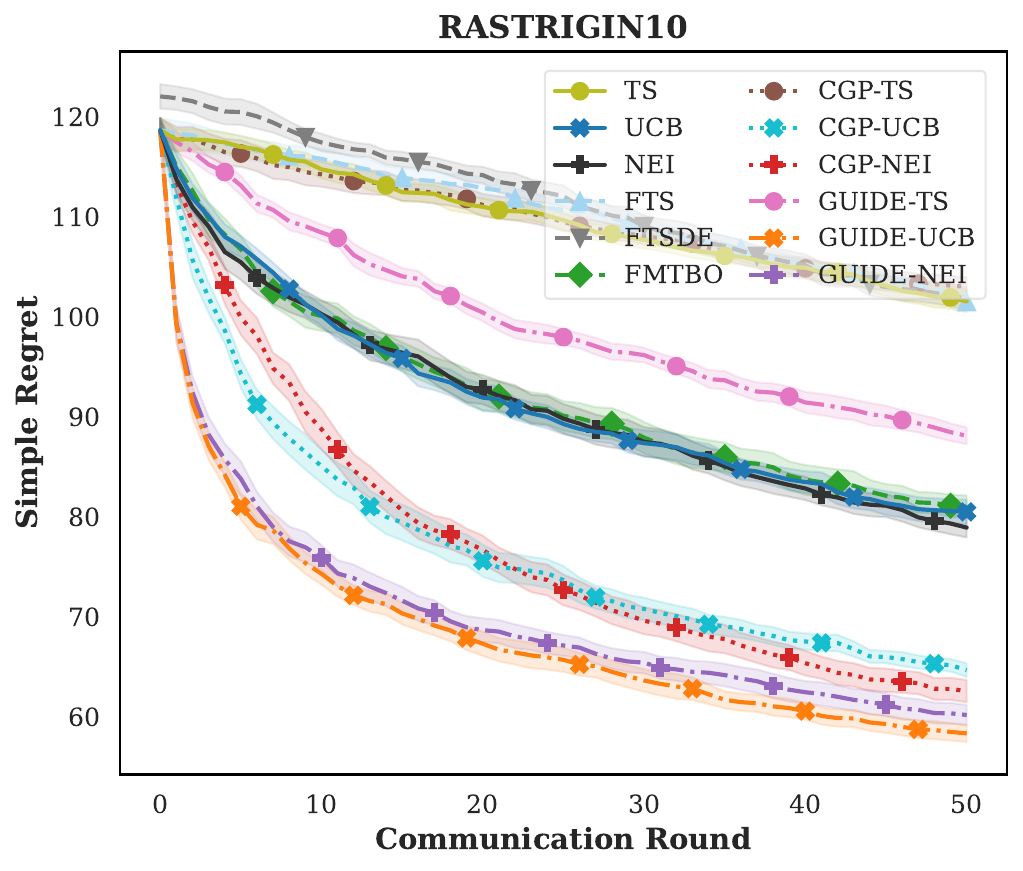}
        \caption{Rastrigin}
    \end{subfigure}

    \begin{subfigure}[t]{0.25\textwidth}
        \centering
        \includegraphics[width=\linewidth]{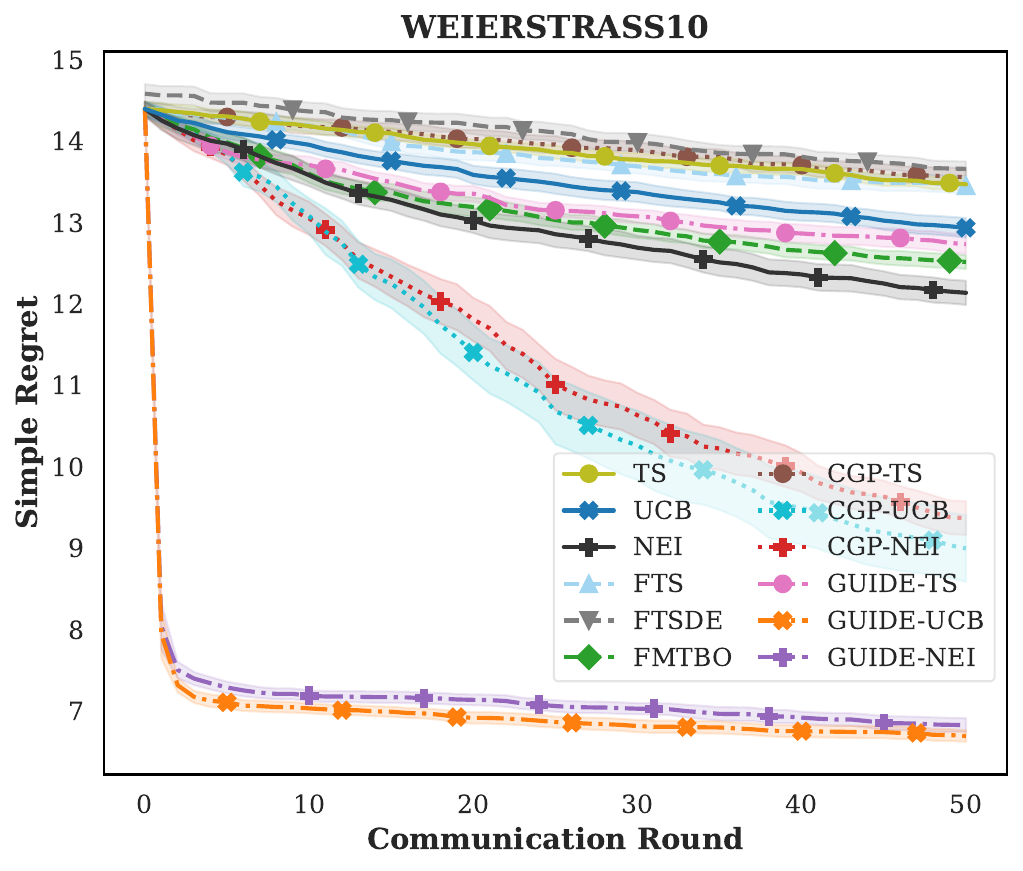}
        \caption{Weierstrass}
    \end{subfigure}\hfill
    \begin{subfigure}[t]{0.25\textwidth}
        \centering
        \includegraphics[width=\linewidth]{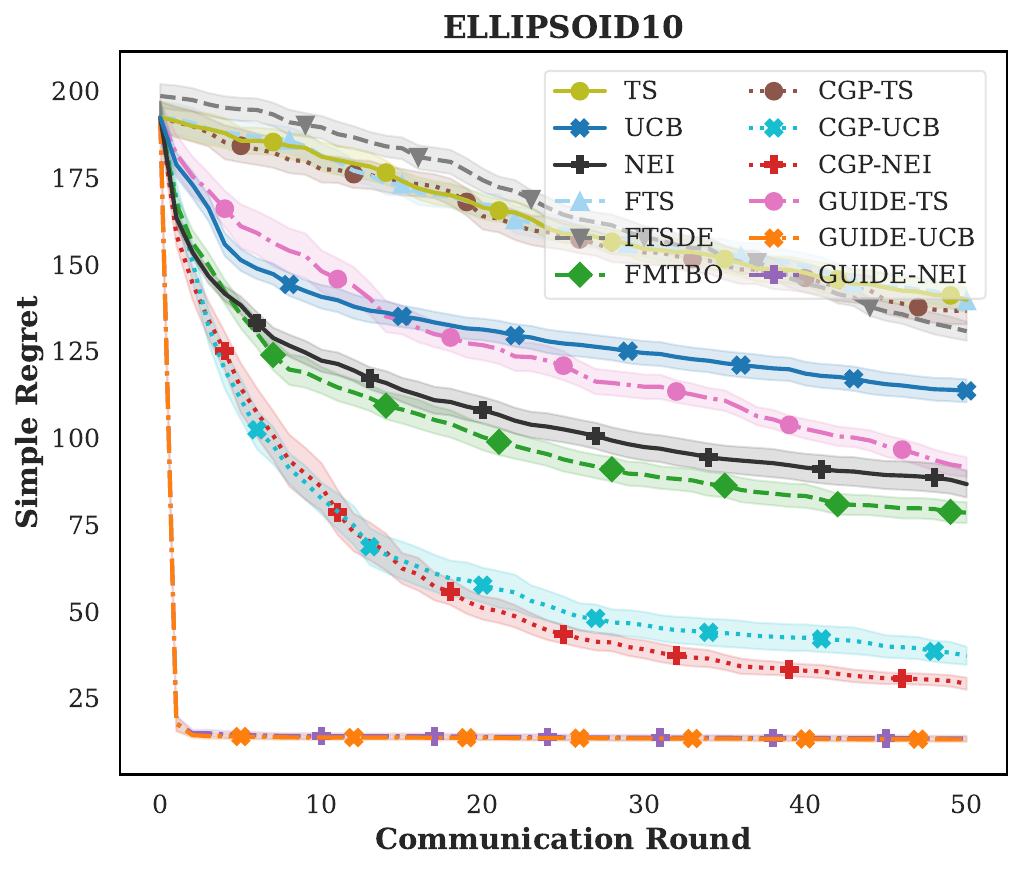}
        \caption{Ellipsoid}
    \end{subfigure}\hfill
    \begin{subfigure}[t]{0.25\textwidth}
        \centering
        \includegraphics[width=\linewidth]{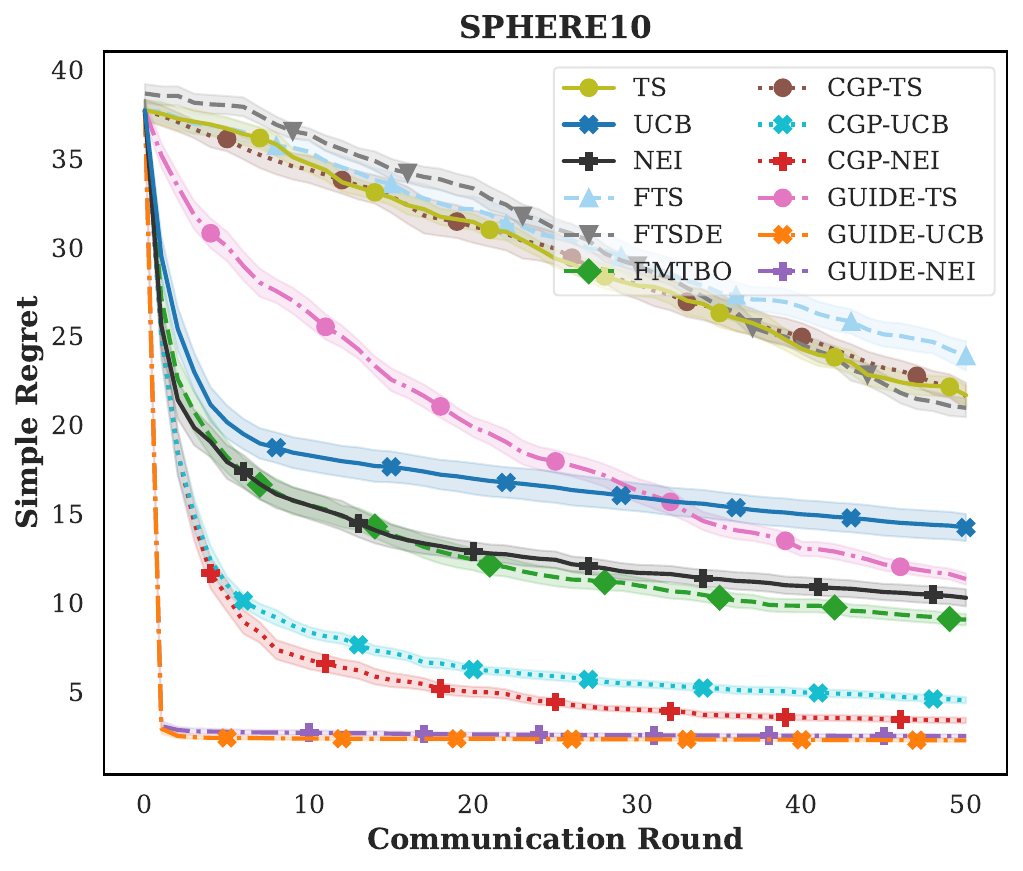}
        \caption{Sphere}
    \end{subfigure}\hfill
    \begin{subfigure}[t]{0.25\textwidth}
        \centering
        \includegraphics[width=\linewidth]{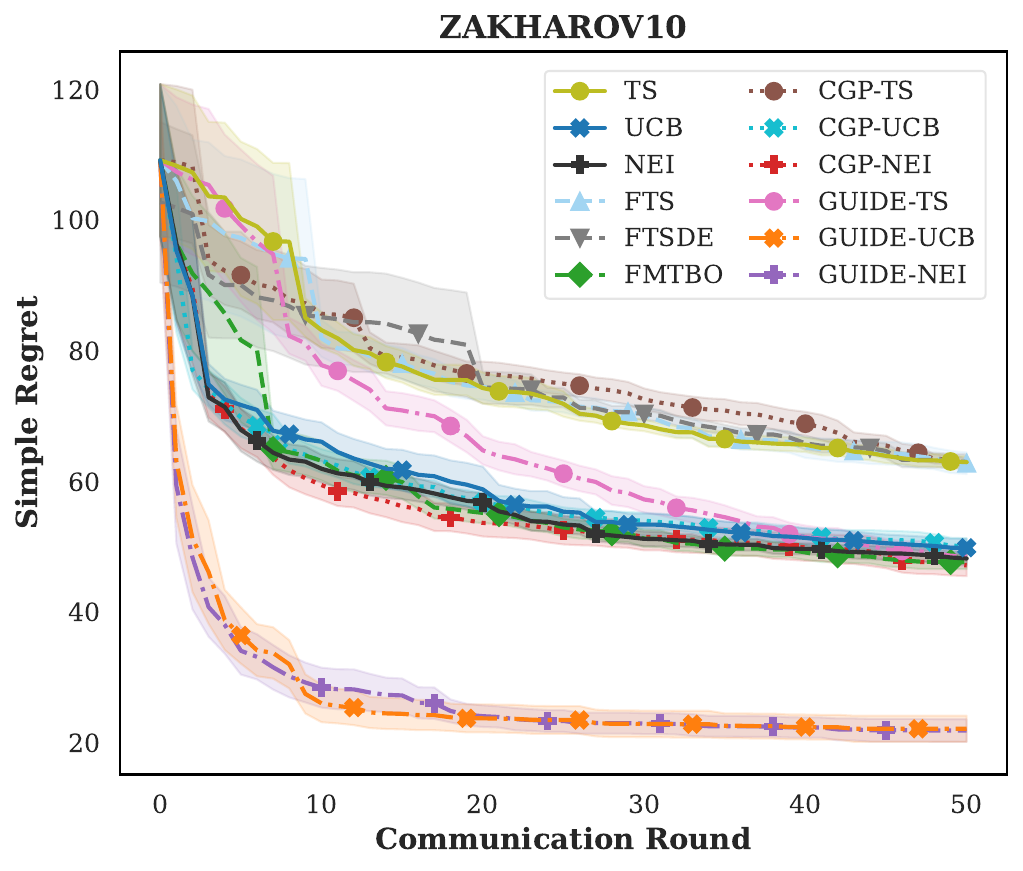}
        \caption{Zakharov}
    \end{subfigure}

    \begin{subfigure}[t]{0.25\textwidth}
        \centering
        \includegraphics[width=\linewidth]{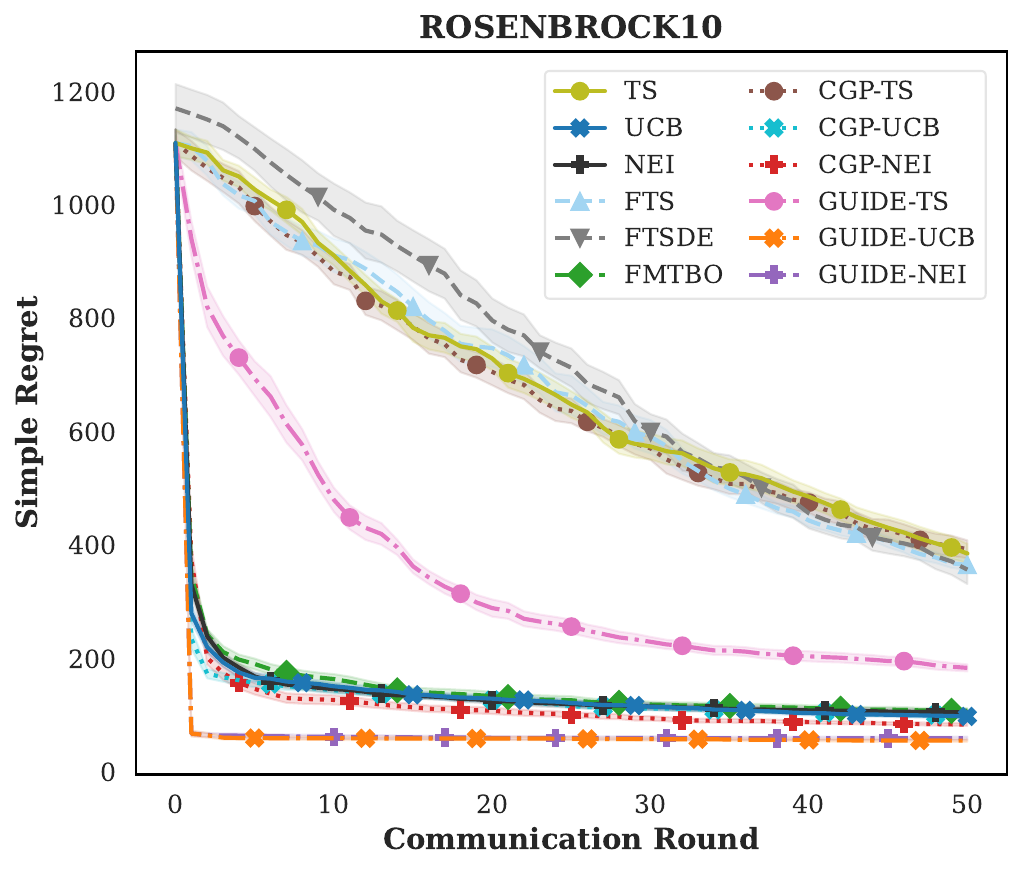}
        \caption{Rosenbrock}
    \end{subfigure}\hfill
    \begin{subfigure}[t]{0.25\textwidth}
        \centering
        \includegraphics[width=\linewidth]{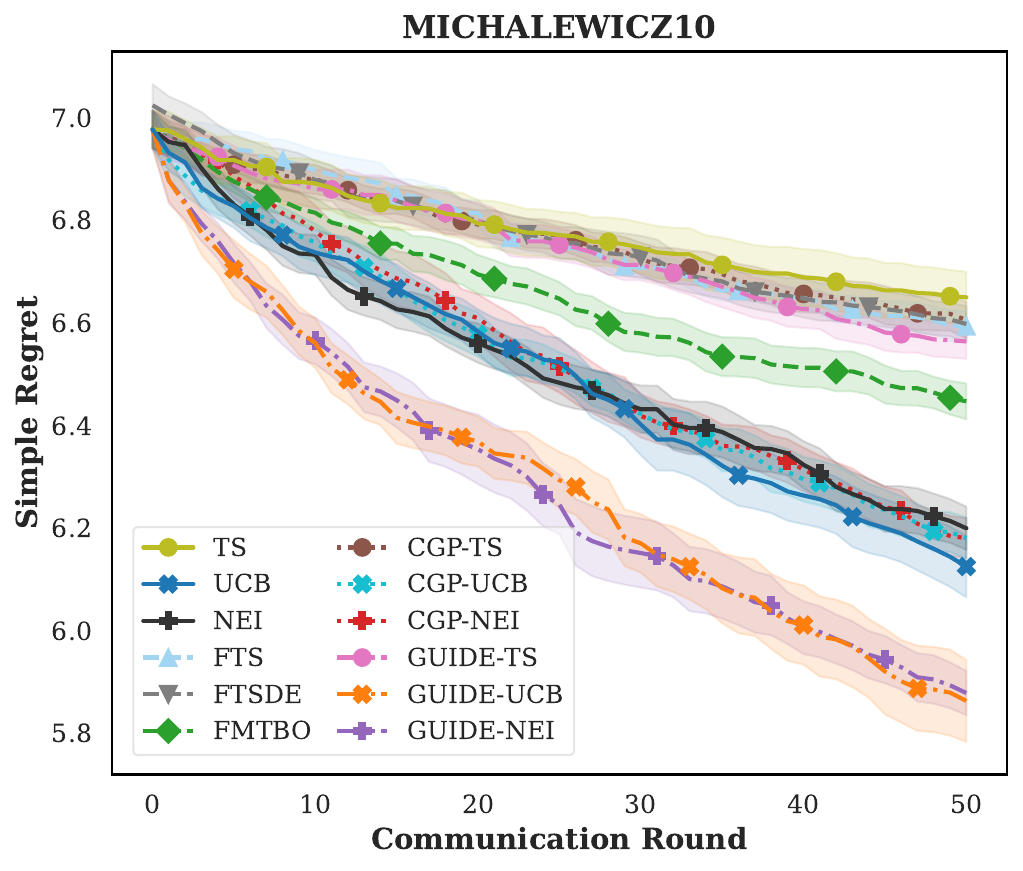}
        \caption{Michalewicz}
    \end{subfigure}\hfill
    \begin{subfigure}[t]{0.25\textwidth}
        \centering
        \includegraphics[width=\linewidth]{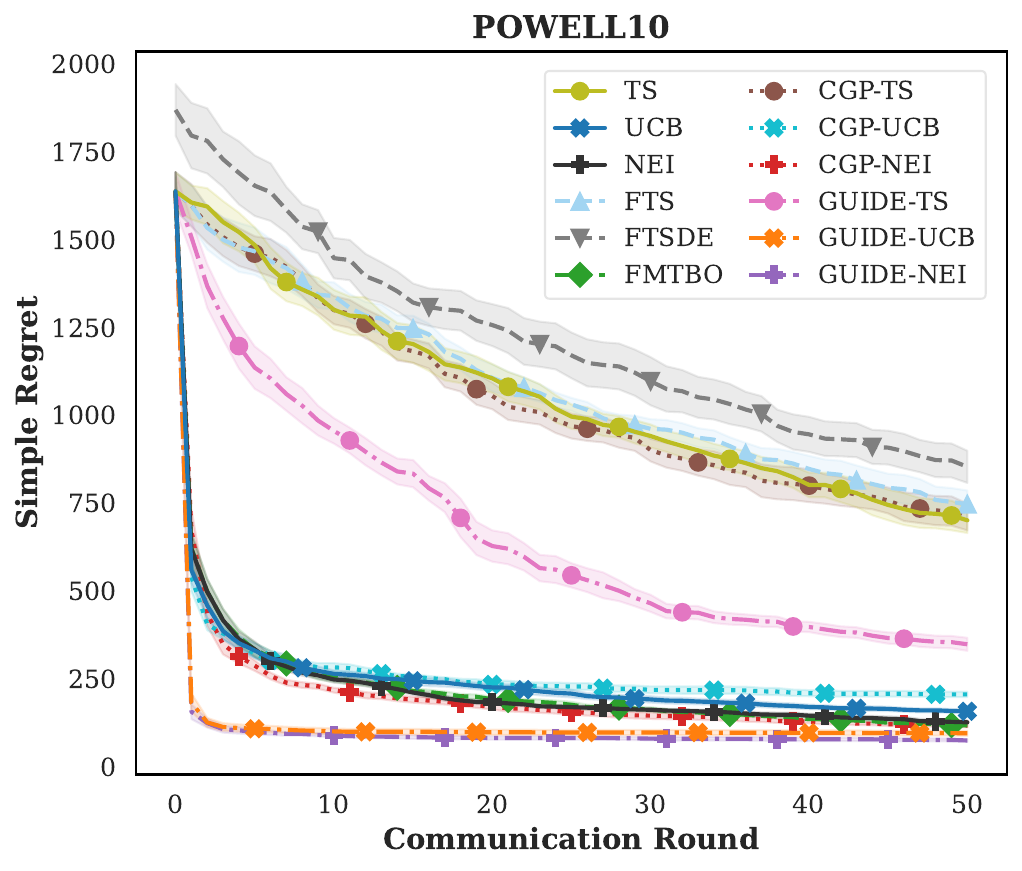}
        \caption{Powell}
    \end{subfigure}\hfill
    \begin{subfigure}[t]{0.25\textwidth}
        \centering
        \includegraphics[width=\linewidth]{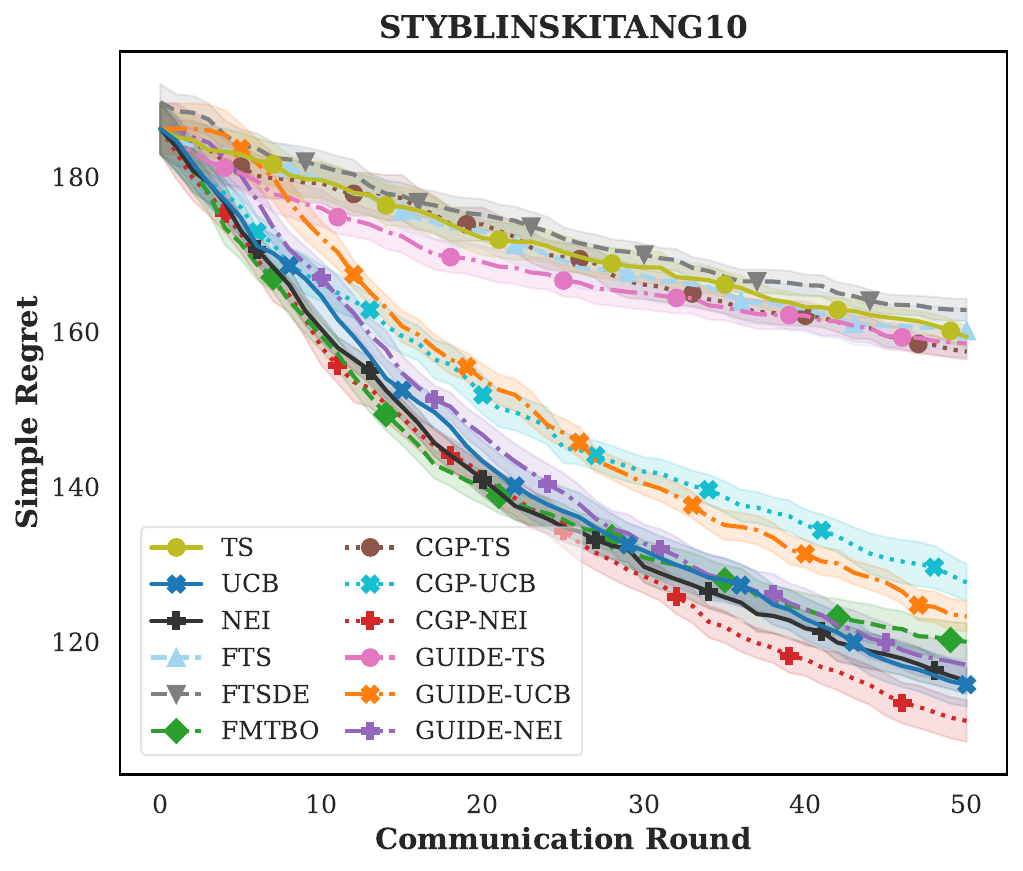}
        \caption{Styblinski--Tang}
    \end{subfigure}

    \caption{Convergence trajectories for all 12 synthetic functions under Level 2 (Mild heterogeneity). Curves show the mean average simple regret over 10 independent runs, and shaded regions denote $\pm 1$ standard error of the mean across runs. Lower is better.}
    \label{fig:appendix_level2_curves}
\end{figure*}

Level~2 introduces mild task heterogeneity, with the empirical mean
pairwise Spearman correlation decreasing from $1.000$ to $0.557$.
GUIDE-UCB remains the strongest aggregate method with an average rank of
$1.58$. It achieves the lowest final regret on $9$ benchmarks and ranks
among the top two methods on $11$ of the $12$ functions. GUIDE-NEI ranks
second overall at $2.17$, is the best method on Zakharov and Powell, and
appears among the top two methods on $10$ functions. Relative to their
matched independent baselines, GUIDE-UCB and GUIDE-NEI improve on $11$ of
the $12$ functions, while GUIDE-TS improves on all $12$.

Large absolute differences remain on several benchmarks. UCB and GUIDE-UCB
obtain $19.4350$ and $8.4785$ on Ackley, $113.6023$ and $13.1331$ on
Ellipsoid, $14.2443$ and $2.2689$ on Sphere, and $97.8257$ and $55.2114$
on Rosenbrock. On Zakharov, GUIDE-NEI obtains $21.8975$, compared with
$47.1746$ for CGP-NEI and $47.6923$ for FMTBO. On Powell, GUIDE-NEI
obtains $74.3549$, compared with $116.5086$ for CGP-NEI and $117.3149$
for FMTBO. The benefit of exchanging distributions over the locations of the agents' optima therefore remains substantial after moderate shifts and rotations are introduced.

The trajectories in Figure~\ref{fig:appendix_level2_curves} retain much of
the qualitative pattern observed at Level~1. GUIDE-UCB and GUIDE-NEI show
a particularly clear lead on Ackley, Levy, Weierstrass, Ellipsoid, Sphere,
and Zakharov, where their curves separate visibly from the main competing
methods during the optimization process. On Griewank, Rastrigin,
Rosenbrock, Michalewicz, and Powell, the two GUIDE variants also reach
competitive or leading low-regret regions, but the gap from the strongest
baselines is smaller. Styblinski--Tang remains the main exception.
Compared with Level~1, GUIDE still reaches low-regret regions earlier on
most benchmarks, but its lead over the strongest baselines becomes smaller
on several functions.

This difference follows directly from the task construction. At Level~1,
the distributions transferred by different agents refer to the same
underlying objective and therefore tend to support the same promising
regions. At Level~2, the shifts move these regions and the rotations change
the local geometry of non-rotationally-invariant objectives. A promising
region for one agent may therefore no longer align exactly with the
corresponding region of another agent. Since these transformations are still
moderate, useful spatial structure can remain shared across agents. GUIDE
transfers distributions over promising regions, so a small spatial mismatch
does not invalidate the shared information as easily as it can invalidate a
single transferred design.

The change on Michalewicz is especially informative. Under Level~1,
CGP-NEI and CGP-UCB outperform the GUIDE variants, whereas under Level~2
GUIDE-UCB and GUIDE-NEI obtain the two lowest final regrets of $5.8629$
and $5.8783$. A likely explanation is that point-level transfer is most
effective when the narrow high-quality valleys of different agents are
closely aligned. Once mild shifts and rotations are introduced, a
high-quality design from one agent can fall outside the corresponding valley
of another agent. GUIDE instead transfers a distribution over a promising
region, making the transferred information less dependent on exact point
alignment. This broader spatial representation is therefore better suited to
the mild mismatch introduced at Level~2.

Michalewicz also illustrates the distinction between global task similarity
and local transferability. Its Level~2 mean pairwise Spearman correlation is
only $0.034$, yet GUIDE-UCB and GUIDE-NEI obtain the two best final results.
The Spearman correlation measures ranking agreement over the full search
space, whereas GUIDE acts on the spatial support of promising regions.
A low global correlation can therefore coexist with useful shared structure
near high-value regions.

Styblinski--Tang remains the main exception. Its many competing local basins
make it difficult for a regional distribution to indicate one clearly
preferred search direction. Under Level~1, point-level transfer is especially
effective because a good point identified by one agent is directly useful to
all others. After the Level~2 transformations are introduced, this exact
point correspondence becomes weaker. This is visible in the final results.
CGP-UCB changes from $62.7248$ at Level~1 to $127.6079$ at Level~2, while
CGP-NEI changes from $56.2330$ to $109.7452$. The corresponding GUIDE-UCB
results are $99.6545$ and $123.2241$, and the GUIDE-NEI results are
$102.8457$ and $116.9887$. GUIDE is therefore still not the best method on
Styblinski--Tang at Level~2, but its performance deteriorates much less than
the two CGP variants as mild heterogeneity is introduced. On this benchmark, the smaller degradation of GUIDE suggests that distributional guidance is less sensitive to mild spatial misalignment than
point-level transfer.

\subsection{Level 3: $\delta_{\mathrm{shift}}=0.3$ and $\delta_{\mathrm{rot}}=1.0$}

\begin{table*}[h]
    \caption{Complete synthetic results under Level 3 (Severe heterogeneity). Entries report the final mean average simple regret over 10 independent runs. Lower is better. The best and second-best results in each row are shown in bold and underlined, respectively.}
    \label{tab:appendix_level3_results}
    \centering
    \setlength{\tabcolsep}{3.0pt}
    \renewcommand{\arraystretch}{1.07}
    \scriptsize
    \resizebox{\textwidth}{!}{%
    \begin{tabular}{lcccccccccccc}
    \toprule
    \textbf{Benchmark}
    & \textbf{TS} & \textbf{UCB} & \textbf{NEI}
    & \textbf{FTS} & \textbf{FTS-DE} & \textbf{FMTBO}
    & \textbf{CGP-TS} & \textbf{CGP-UCB} & \textbf{CGP-NEI}
    & \textbf{GUIDE-TS} & \textbf{GUIDE-UCB} & \textbf{GUIDE-NEI} \\
    \midrule
    Ackley             & 20.4047 & 20.5073 & 20.4203 & 20.4080 & 20.3972 & 20.2383 
                       & 20.3699 & 20.3380 & 20.4331 & 20.3552 & \second{20.1739} & \best{20.1161} \\
    Levy               & 37.5643 & 24.3700 & 25.5279 & 41.1159 & 40.0354 & 25.0598 
                       & 34.5967 & 24.6258 & 30.4253 & 34.2904 & \best{20.8306} & \second{23.6984} \\
    Griewank           & 67.8144 & 28.8836 & 26.9351 & 72.7054 & 74.5664 & 27.0838 
                       & 65.7857 & \second{21.7611} & 30.0882 & 68.9316 & \best{21.2957} & 29.7074 \\
    Rastrigin          & 106.3844 & 95.6377 & 96.0487 & 110.2650 & 111.8623 & 95.4028 
                       & 108.3831 & 96.2275 & 100.1082 & 104.5270 & \best{90.8849} & \second{94.0912} \\
    Weierstrass        & 13.6349 & 13.1439 & \second{12.7244} & 13.6075 & 13.6386 & 12.8031 
                       & 13.4763 & 12.9093 & 12.8169 & 13.4107 & 12.8443 & \best{12.4937} \\
    Ellipsoid          & 157.2353 & 162.9886 & 138.9518 & 172.5919 & 167.3659 & \best{120.5134} 
                       & 156.5689 & 134.2802 & 151.8904 & 149.4864 & \second{127.6623} & 127.9452 \\
    Sphere             & 24.1755 & 22.4991 & 19.2331 & 26.1404 & 26.1197 & \best{15.6632} 
                       & 22.9534 & 18.3412 & 21.1273 & 24.0913 & \second{17.6787} & 17.9497 \\
    Zakharov           & 196.5950 & \second{101.9986} & 1061.0835 & 1095.3250 & 139.9111 & 542.0744 
                       & 167.0674 & 103.6758 & 601.0594 & 128.2162 & \best{92.9444} & 307.8536 \\
    Rosenbrock         & 911.4690 & \best{257.8472} & 316.0209 & 1036.4767 & 1056.0210 & \second{294.0791} 
                       & 873.7124 & 299.9545 & 457.7168 & 919.0478 & 321.9631 & 331.0628 \\
    Michalewicz        & 7.6120 & \best{7.3490} & 7.4127 & 7.6202 & 7.6026 & 7.4742 
                       & 7.6412 & 7.3963 & 7.3846 & 7.6373 & \second{7.3793} & 7.4323 \\
    Powell             & 1397.4194 & \second{681.4695} & 793.8595 & 1683.7471 & 1659.8198 & 751.6538 
                       & 1267.8210 & 718.8015 & 1046.9279 & 1419.2525 & 685.3075 & \best{656.6218} \\
    Styblinski--Tang    & 199.0988 & 157.1972 & 149.6469 & 206.1874 & 207.4384 & \best{148.3097} 
                       & 193.8517 & 169.4907 & 157.8670 & 196.6009 & 165.9432 & \second{149.2083} \\
    \bottomrule
    \end{tabular}%
    }
\end{table*}

\begin{figure*}[h]
    \centering
    \begin{subfigure}[t]{0.25\textwidth}
        \centering
        \includegraphics[width=\linewidth]{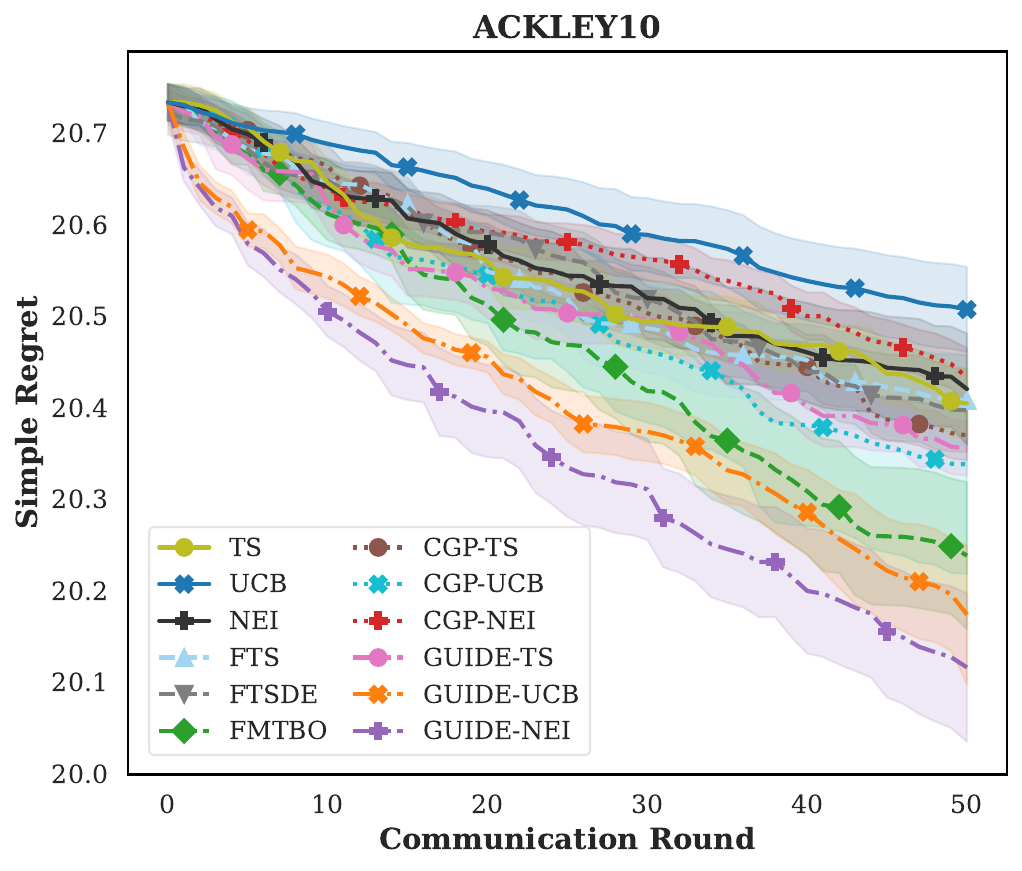}
        \caption{Ackley}
    \end{subfigure}\hfill
    \begin{subfigure}[t]{0.25\textwidth}
        \centering
        \includegraphics[width=\linewidth]{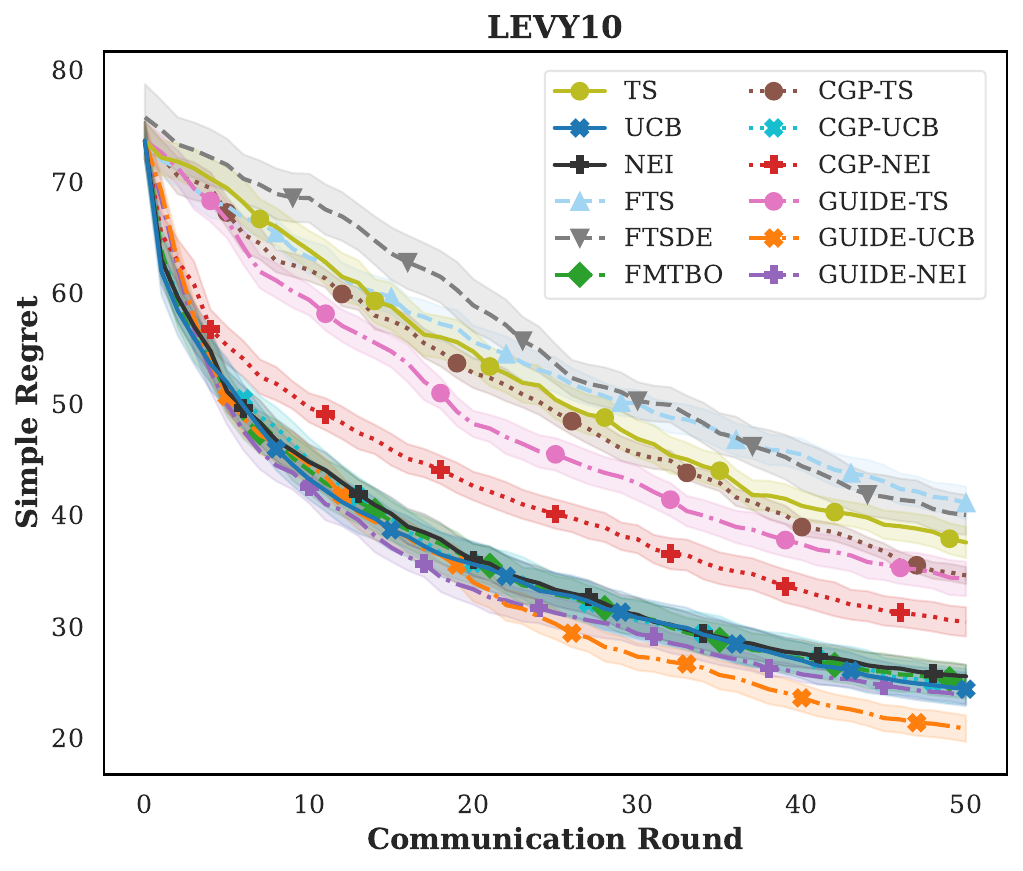}
        \caption{Levy}
    \end{subfigure}\hfill
    \begin{subfigure}[t]{0.25\textwidth}
        \centering
        \includegraphics[width=\linewidth]{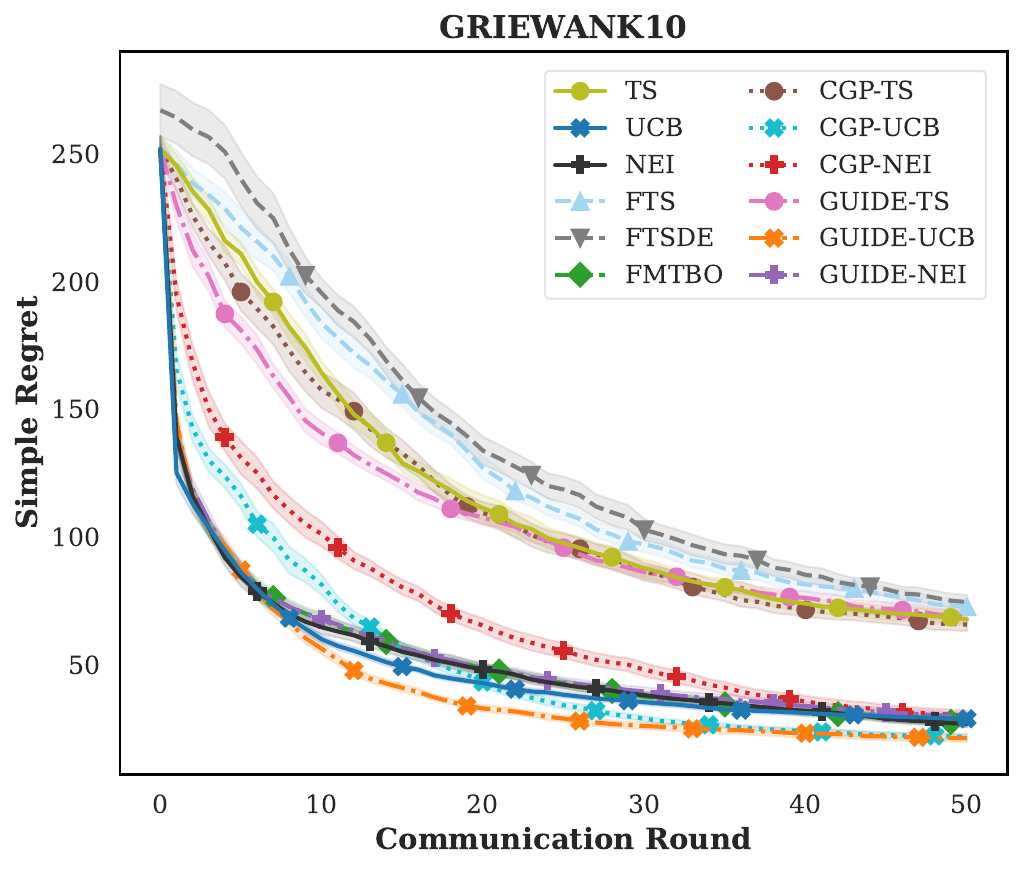}
        \caption{Griewank}
    \end{subfigure}\hfill
    \begin{subfigure}[t]{0.25\textwidth}
        \centering
        \includegraphics[width=\linewidth]{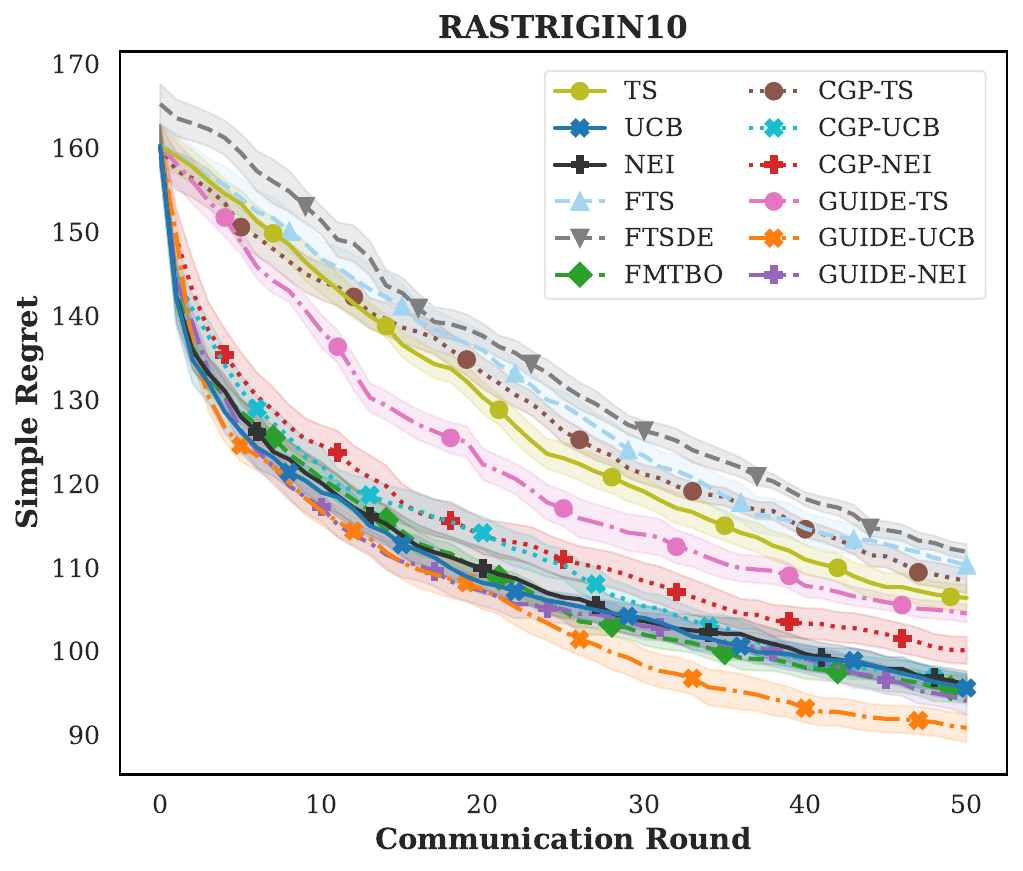}
        \caption{Rastrigin}
    \end{subfigure}

    \begin{subfigure}[t]{0.25\textwidth}
        \centering
        \includegraphics[width=\linewidth]{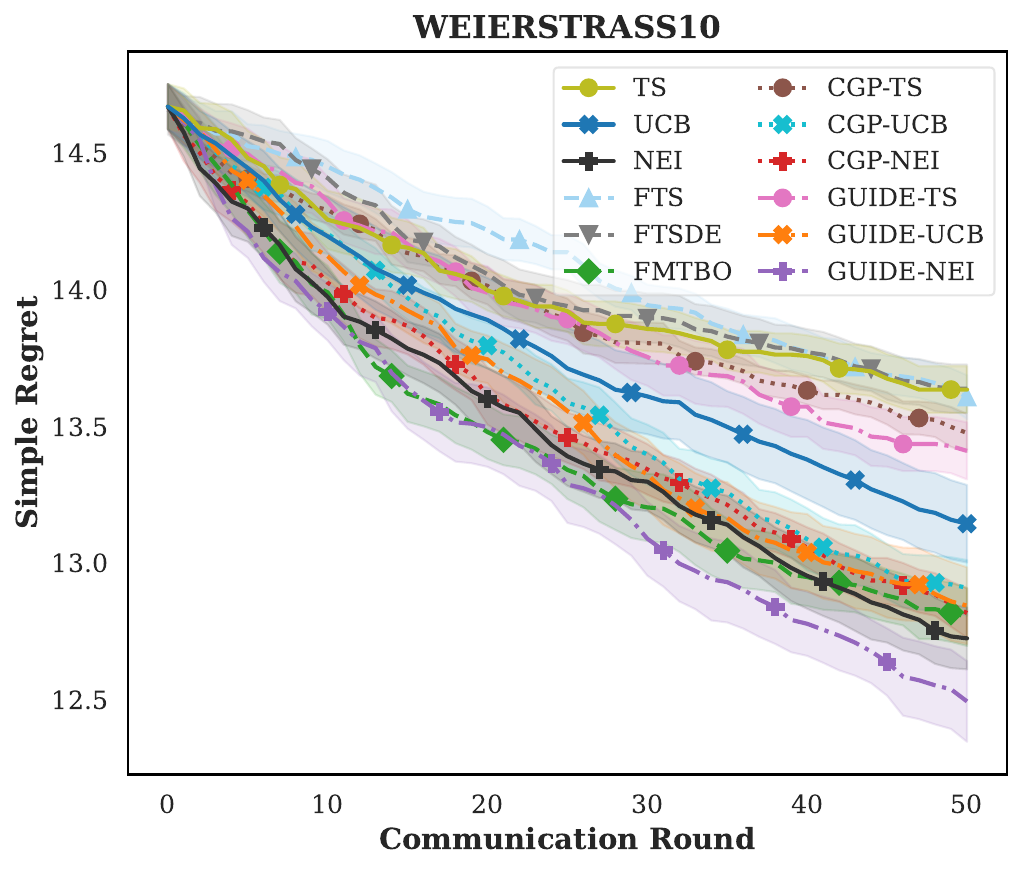}
        \caption{Weierstrass}
    \end{subfigure}\hfill
    \begin{subfigure}[t]{0.25\textwidth}
        \centering
        \includegraphics[width=\linewidth]{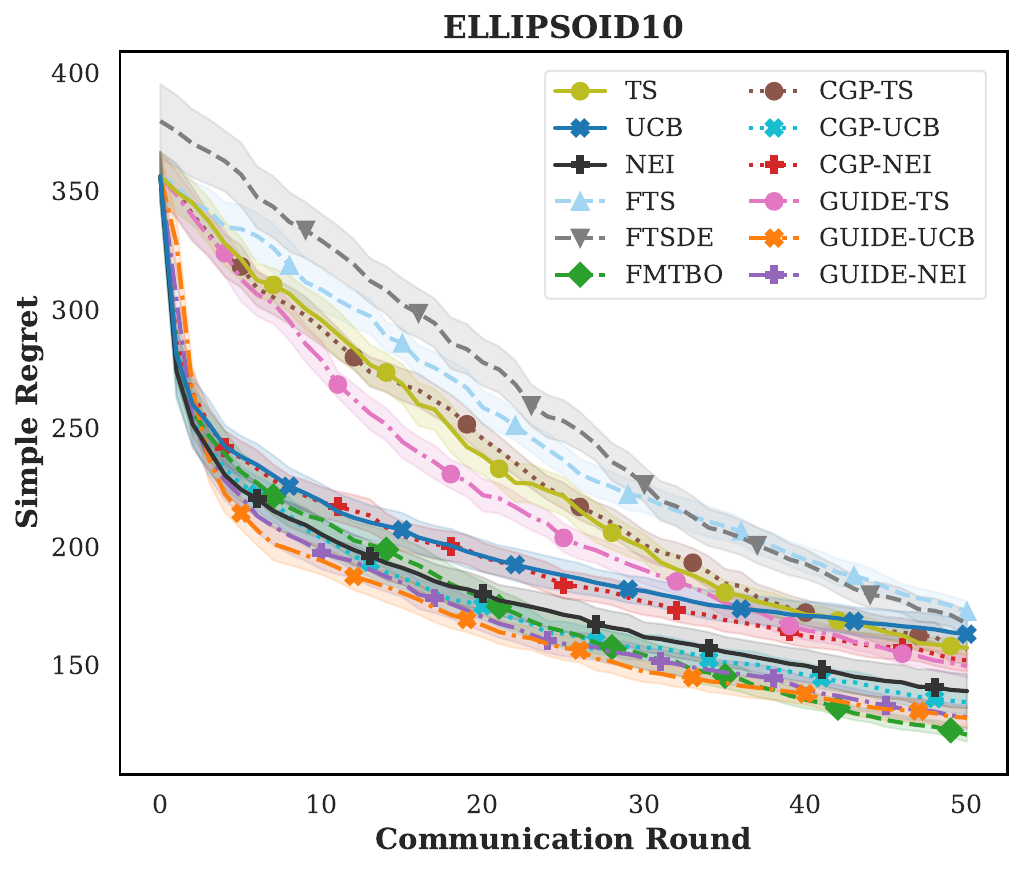}
        \caption{Ellipsoid}
    \end{subfigure}\hfill
    \begin{subfigure}[t]{0.25\textwidth}
        \centering
        \includegraphics[width=\linewidth]{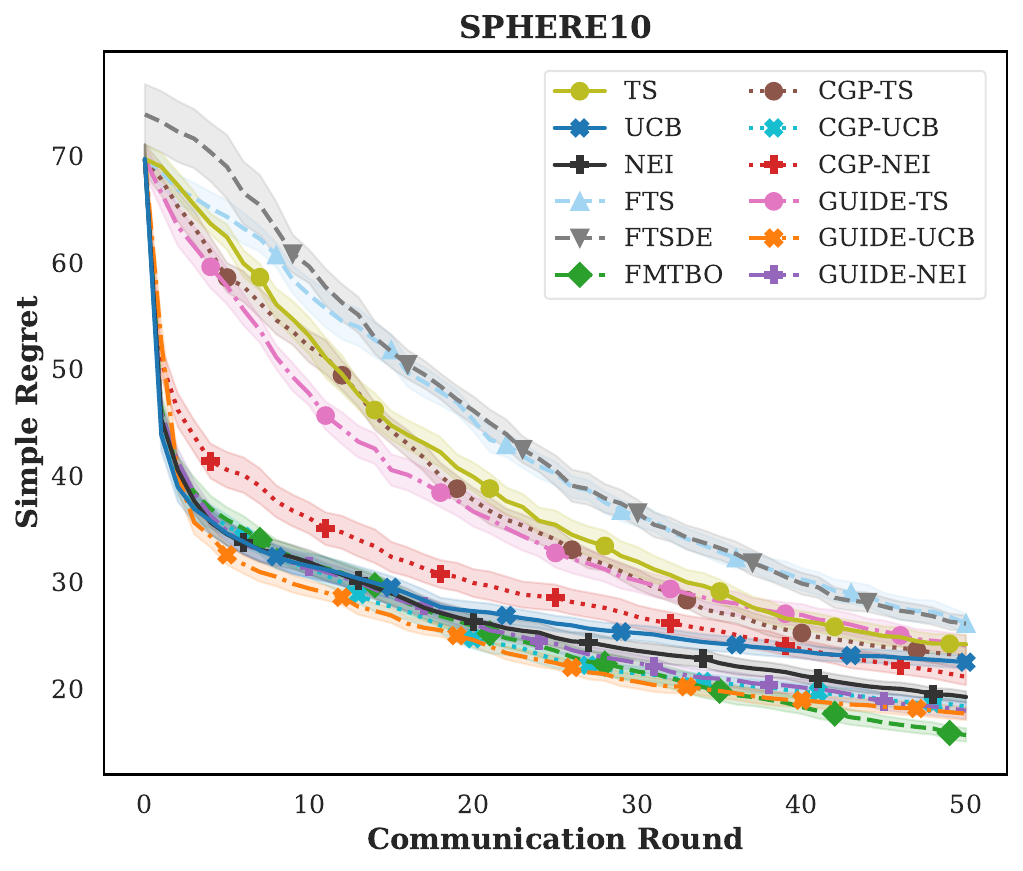}
        \caption{Sphere}
    \end{subfigure}\hfill
    \begin{subfigure}[t]{0.25\textwidth}
        \centering
        \includegraphics[width=\linewidth]{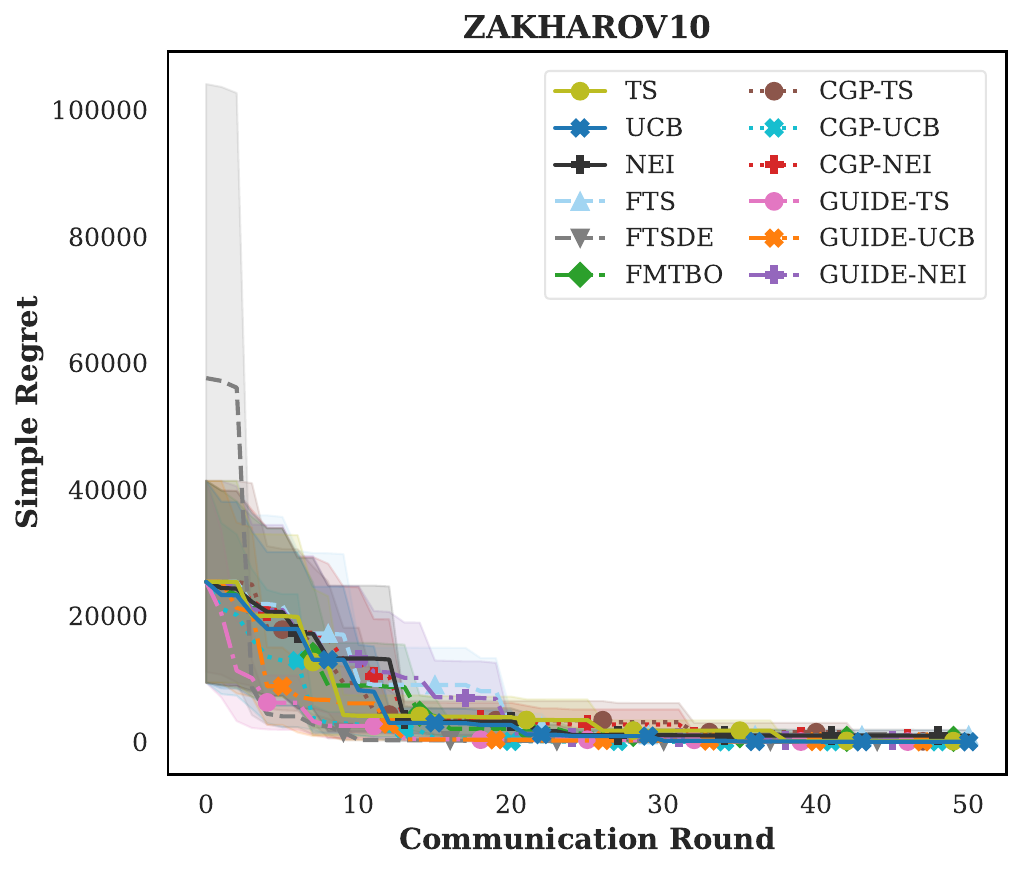}
        \caption{Zakharov}
    \end{subfigure}

    \begin{subfigure}[t]{0.25\textwidth}
        \centering
        \includegraphics[width=\linewidth]{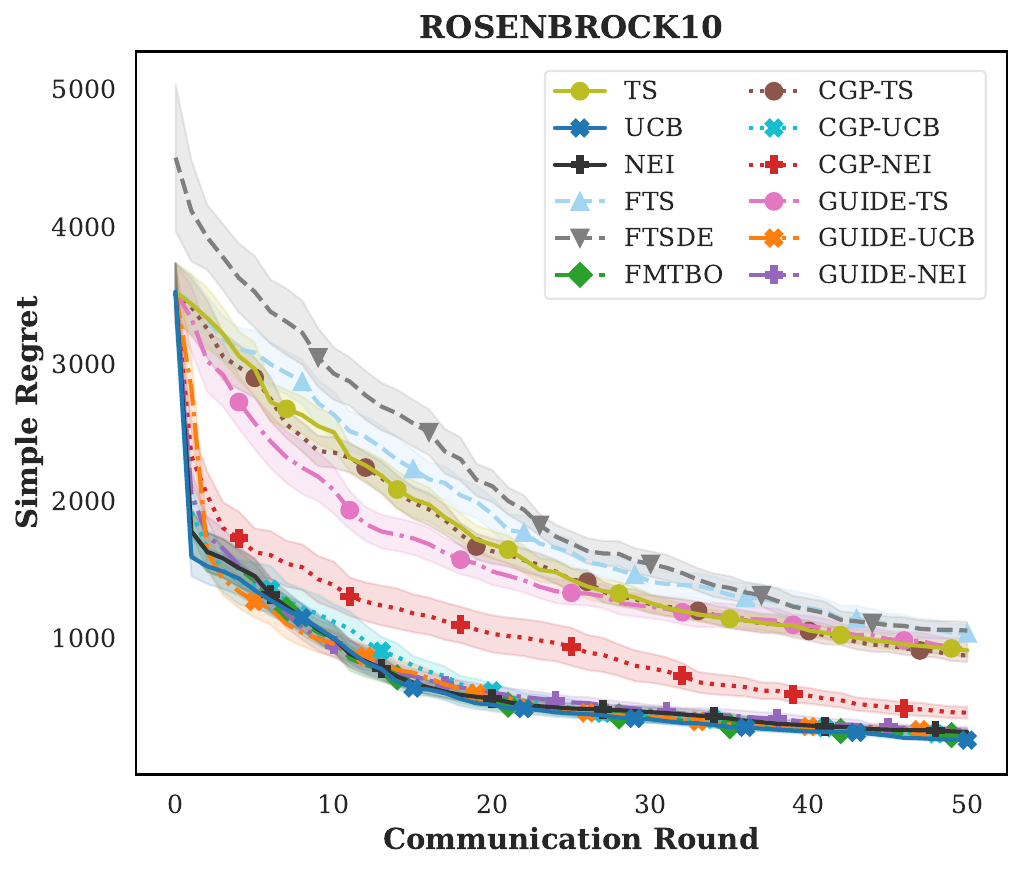}
        \caption{Rosenbrock}
    \end{subfigure}\hfill
    \begin{subfigure}[t]{0.25\textwidth}
        \centering
        \includegraphics[width=\linewidth]{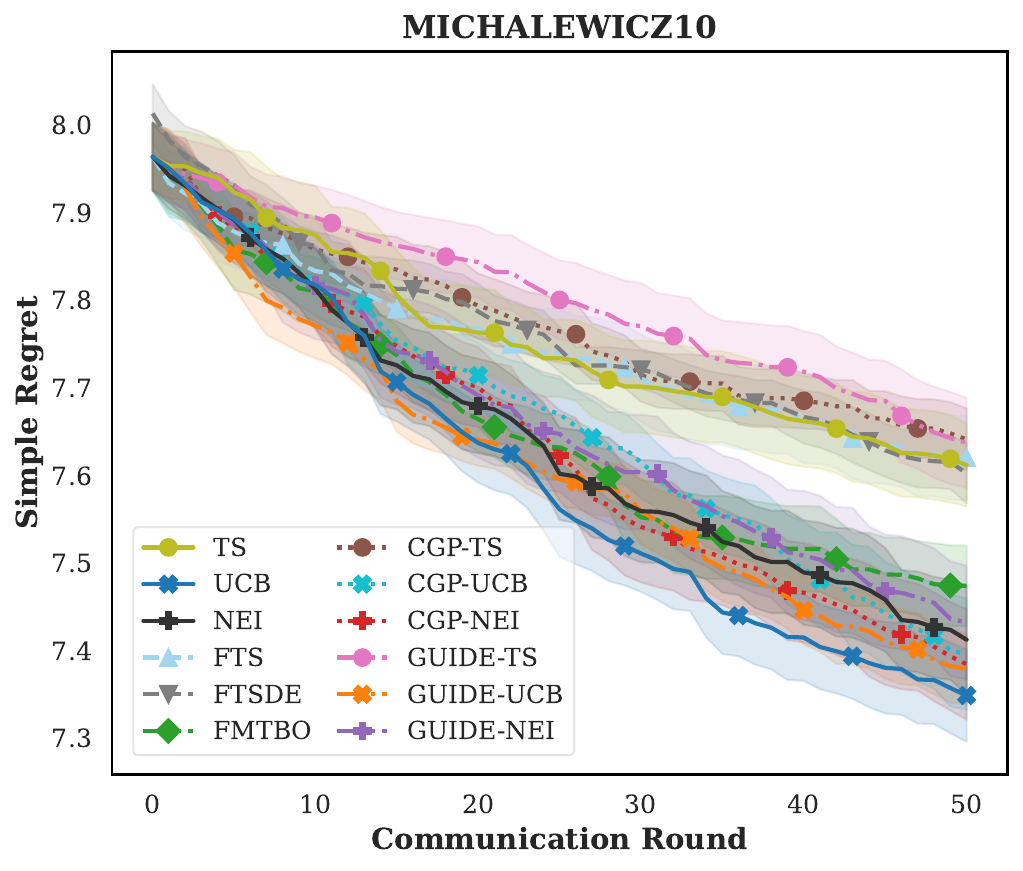}
        \caption{Michalewicz}
    \end{subfigure}\hfill
    \begin{subfigure}[t]{0.25\textwidth}
        \centering
        \includegraphics[width=\linewidth]{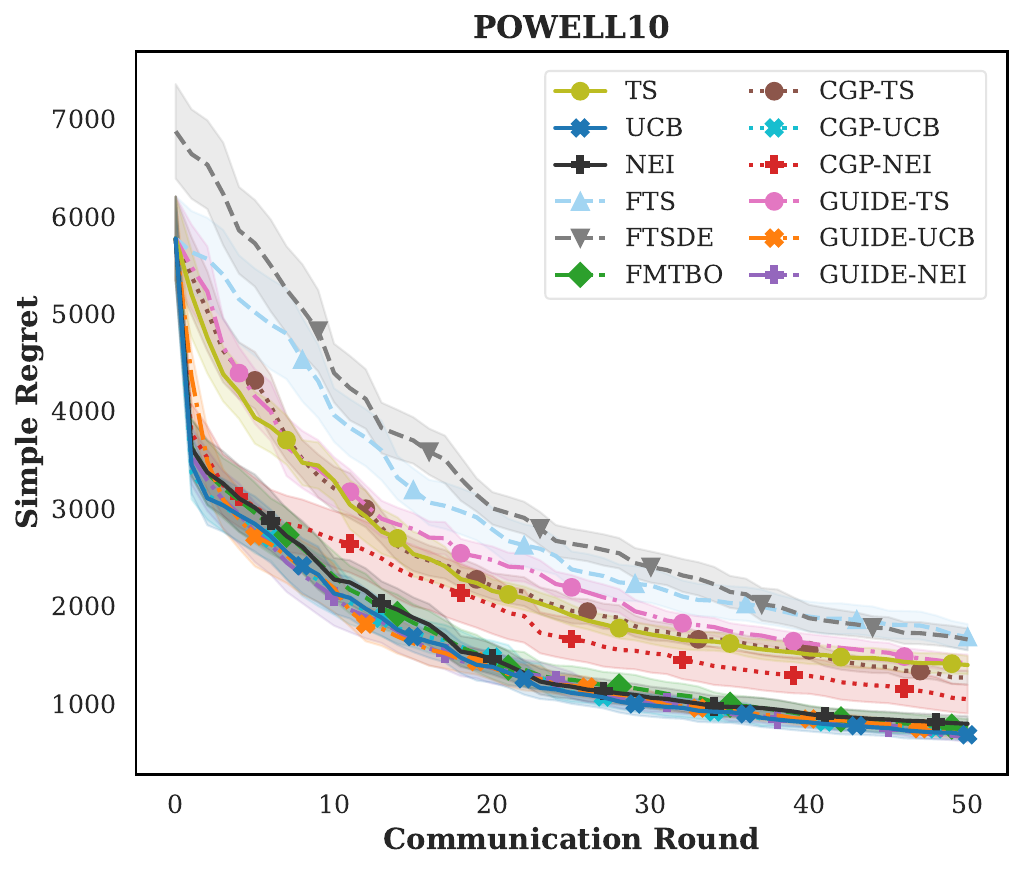}
        \caption{Powell}
    \end{subfigure}\hfill
    \begin{subfigure}[t]{0.25\textwidth}
        \centering
        \includegraphics[width=\linewidth]{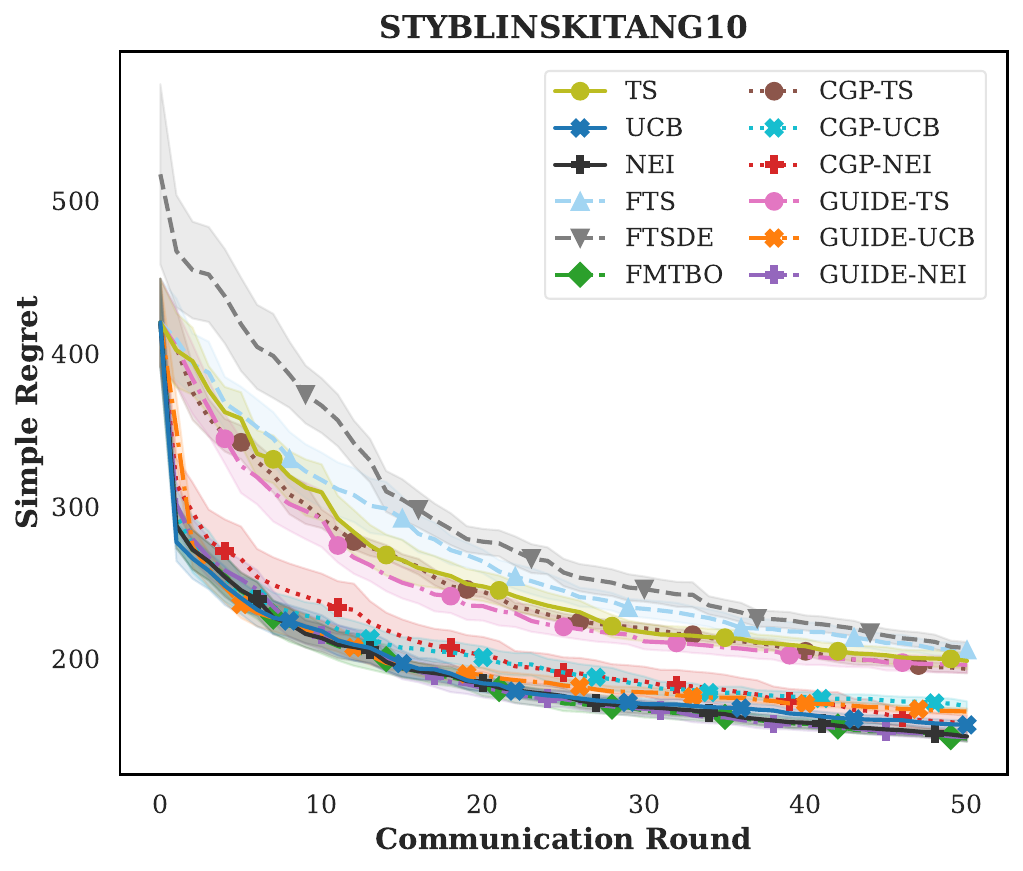}
        \caption{Styblinski--Tang}
    \end{subfigure}

    \caption{Convergence trajectories for all 12 synthetic functions under Level 3 (Severe heterogeneity). Curves show the mean average simple regret over 10 independent runs, and shaded regions denote $\pm 1$ standard error of the mean across runs. Lower is better.}
    \label{fig:appendix_level3_curves}
\end{figure*}

Level~3 introduces much stronger task heterogeneity. The larger translations
move the promising regions of different agents farther apart, while the
stronger rotations can substantially change the geometry of
non-rotationally-invariant objectives. The empirical mean pairwise Spearman
correlation decreases to $0.081$. The agents therefore share much less directly transferable spatial information than under Levels~1 and~2.

This setting is particularly difficult for point-level knowledge transfer.
A promising design found by one agent may no longer lie in a promising
region for another agent after the stronger shifts and rotations. The CGP
results illustrate this limitation. CGP-TS and CGP-UCB only slightly improve their average ranks over independent TS and UCB, from $9.17$ to $8.33$ and from $4.83$ to $4.25$, respectively.
In contrast, CGP-NEI has an average rank of $6.75$, worse than $5.42$ for
independent NEI. Point-level transfer therefore no longer provides a consistent advantage
when the agents' optima are poorly aligned.

FMTBO becomes more competitive in this regime and achieves an average rank
of $3.67$, the best result among the non-GUIDE methods. It is also the best
method on Ellipsoid, Sphere, and Styblinski--Tang. Unlike methods that
transfer particular promising locations, FMTBO communicates model-level
information through GP hyperparameters and estimates task relatedness from
predictive rankings. Such information does not require the optimum of one
agent to occur at nearly the same location as the optimum of another agent.
This weaker dependence on exact spatial alignment helps explain why FMTBO
becomes relatively more competitive at Level~3.

Despite the severe heterogeneity, GUIDE-UCB retains the best aggregate rank
of $2.58$, while GUIDE-NEI ranks second at $3.42$. GUIDE-UCB is among the
top two methods on $8$ of the $12$ benchmarks and obtains the lowest final
regret on Levy, Griewank, Rastrigin, and Zakharov. GUIDE-NEI is best on
Ackley, Weierstrass, and Powell. The complete trajectories in
Figure~\ref{fig:appendix_level3_curves} show a much less uniform advantage
than at Levels~1 and~2. The leading curves overlap and cross more frequently,
and GUIDE no longer separates early from the other methods on most
benchmarks. This is expected because the received distributions now contain
more information that is useful to some agents but less relevant to others.

Two mechanisms help GUIDE remain robust in this regime. First, the strength
of federated guidance decays according to
$\lambda_t=\lambda_{\max}/\sqrt{t}$. Global information therefore has its
largest influence when local observations are scarce, while its influence
gradually decreases as each agent collects more evidence about its own
objective. The direct influence of global guidance therefore decreases over rounds,
allowing local evidence to play a larger role later in the optimization.

Second, FI-GP leaves the local posterior mean unchanged and rescales only
the posterior uncertainty. The absolute uncertainty increment at
$\mathbf{x}$ is
$\lambda_t G_{n,t}(\mathbf{x})\sigma_{n,t-1}(\mathbf{x})$.
Global guidance therefore has a large effect mainly where the federation
provides strong support and the local GP remains uncertain. If a region has
already been sufficiently explored locally, its posterior uncertainty is
typically smaller, so even strong global guidance produces only a limited
intervention. If local evaluations also indicate poor performance, the low
posterior mean is preserved. Such a region is therefore unlikely to be
selected solely because it is globally recommended.

The remaining failures also follow this interpretation. On Rosenbrock,
independent UCB obtains $257.8472$, compared with $321.9631$ for GUIDE-UCB.
Its narrow curved valley makes transferred spatial guidance particularly
sensitive to strong translations and rotations. On Michalewicz, UCB obtains
$7.3490$, while GUIDE-UCB obtains $7.3793$. The difference is much smaller,
but the steep and narrow high-quality regions again leave little tolerance
for spatial mismatch. These cases show that negative transfer can still occur on individual tasks.
The main Level~3 result is that GUIDE retains the best aggregate rank despite
these task-specific failures.

\subsection{Real-World Convergence Trajectories}
\label{app:realworld_curves}

Figure~\ref{fig:realworld_curves} reports the complete convergence
trajectories corresponding to Table~\ref{tab:realworld_results}.
For visualization, we use $1-\mathrm{AUC}$ for Landmine Detection and
$1-\mathrm{Accuracy}$ for Activity Recognition and FedHPO, so lower values
are better.

\begin{figure*}[ht]
    \centering

    \begin{subfigure}[t]{0.32\textwidth}
        \centering
        \includegraphics[width=\linewidth]{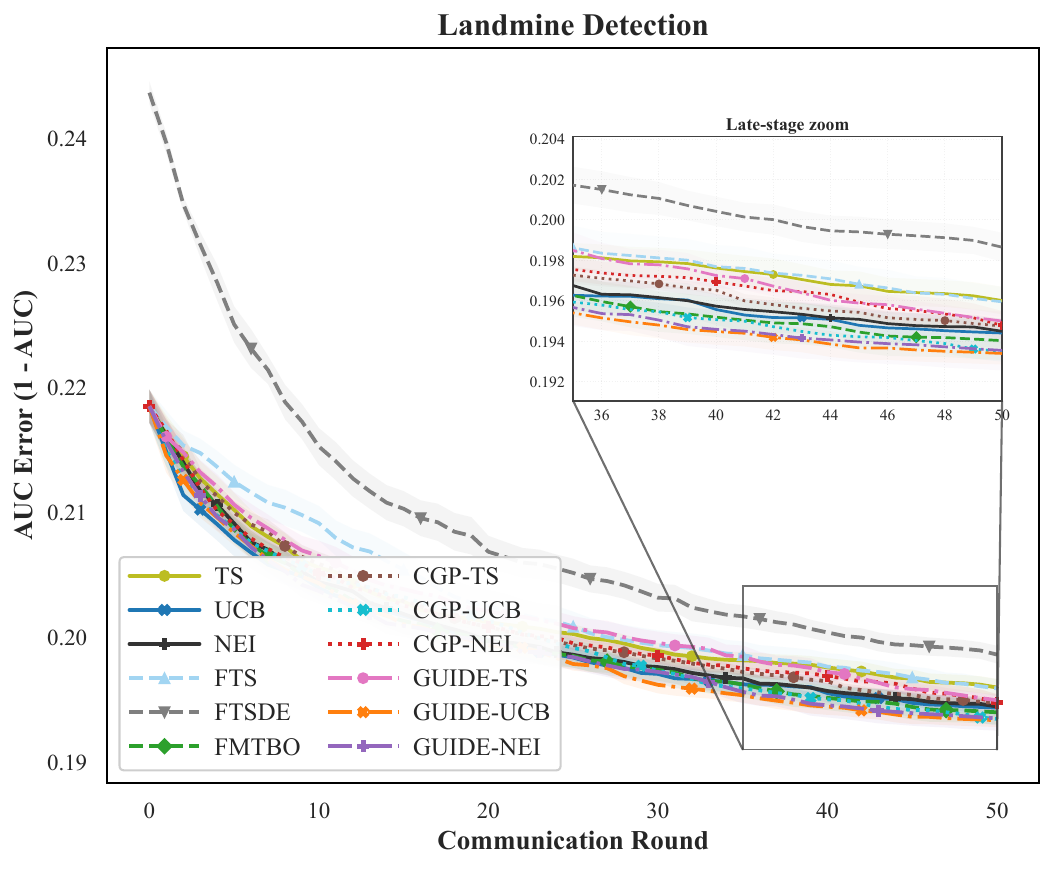}
        \caption{Landmine Detection}
    \end{subfigure}
    \hfill
    \begin{subfigure}[t]{0.32\textwidth}
        \centering
        \includegraphics[width=\linewidth]{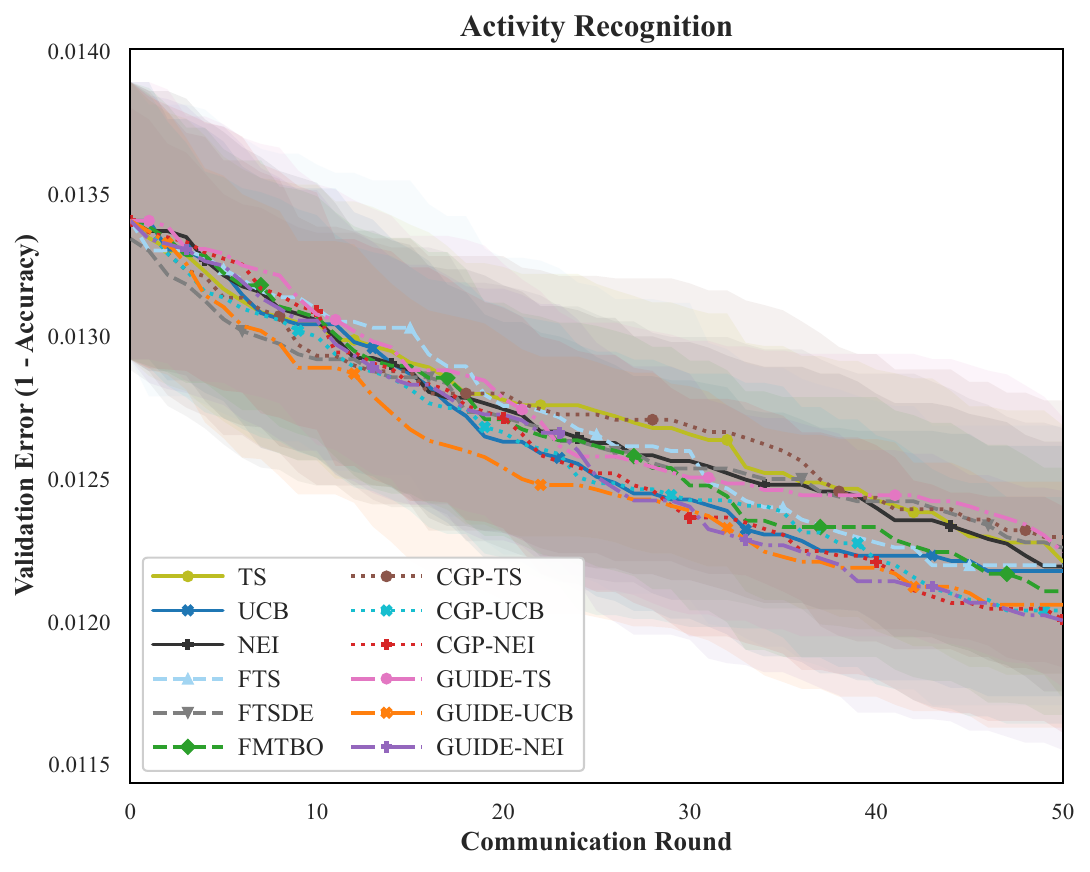}
        \caption{Activity Recognition}
    \end{subfigure}
    \hfill
    \begin{subfigure}[t]{0.32\textwidth}
        \centering
        \includegraphics[width=\linewidth]{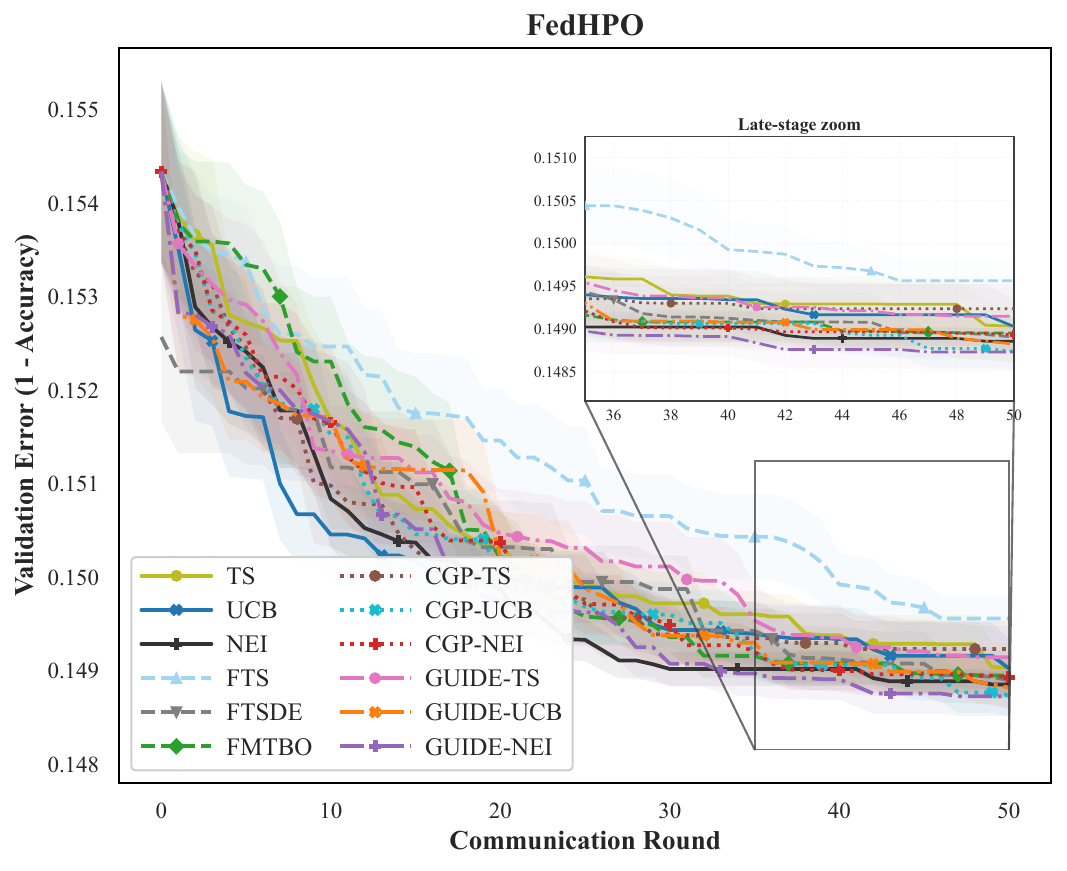}
        \caption{FedHPO}
    \end{subfigure}

    \caption{Convergence trajectories on the three real-world federated optimization tasks. Curves show the mean transformed error metric over 10 independent runs, and shaded regions denote $\pm 1$ standard error of the mean. Lower is better. }
    \label{fig:realworld_curves}
\end{figure*}

On the three real-world tasks, the strongest methods often finish within a
narrow performance range, while the matched comparisons remain consistently
favorable for GUIDE-UCB and GUIDE-NEI.

On Landmine Detection, GUIDE-UCB achieves the highest final mean AUC of
$0.806604$, followed by CGP-UCB at $0.806501$ and GUIDE-NEI at $0.806452$.
Independent UCB obtains $0.805593$. 
The convergence trajectories remain close during the earlier rounds, while
GUIDE-UCB gradually moves to the best mean performance later in the
optimization horizon. Unlike many Level~1 synthetic benchmarks, there is no
clear early separation from the strongest baselines.

Activity Recognition is much more tightly clustered. GUIDE-NEI obtains the
highest final mean validation accuracy of $0.987994$, followed by CGP-NEI at
$0.987990$ and CGP-UCB at $0.987959$. GUIDE-UCB obtains $0.987938$,
compared with $0.987821$ for independent UCB. All leading methods remain within a narrow accuracy range over most of the BO horizon. This task therefore shows that GUIDE remains competitive, but
does not exhibit a large practical separation among the strongest methods
under the tested budget.

FedHPO shows a similar narrow final range. GUIDE-NEI obtains the highest mean
accuracy of $0.851271$, followed by CGP-UCB at $0.851262$. GUIDE-UCB obtains
$0.851178$, compared with $0.850973$ for independent UCB, while independent
NEI obtains $0.851145$. The trajectories fluctuate more during the earlier
rounds and become increasingly concentrated later.

Across the three tasks, GUIDE-NEI achieves the best average rank of $1.67$,
while GUIDE-UCB ranks third at $2.67$. Both variants improve upon their
matched independent baselines on all three tasks. GUIDE-TS has an average
rank of $9.67$, so the real-world evidence is strongest for the UCB and NEI
instantiations.

The absolute gaps among the strongest methods are small on Activity
Recognition and FedHPO, so the results indicate comparable performance among
the leading methods rather than clear superiority over every collaborative
baseline.

\section{Ablation Study Details}
\label{app:ablation}

We conduct the ablation study using GUIDE-UCB on the same four representative benchmarks, namely Ackley, Rastrigin, Zakharov, and Michalewicz.
The main study uses Level~2 heterogeneity, while the same configurations are additionally evaluated under Level~3 to examine whether their effects persist as cross-agent transferability decreases.
All configurations use the same local GP models, posterior sampling procedure, initialization, evaluation budget, UCB decision rule, and repetition seeds.
The Level~2 final results are reported in Table~\ref{tab:main_ablation} in the main text.

\paragraph{Ablation variants.}
``Vanilla UCB'' removes the complete federated distributional-exchange and uncertainty-intervention pipeline and performs standard local UCB optimization independently for each agent.
It serves as a reference for measuring the overall benefit of federated guidance.

``Single-Gaussian belief'' replaces the DPGMM-based representation with a single diagonal Gaussian fitted directly to all posterior-optimum samples, while retaining the same posterior sampling procedure and uploaded message format.
This variant removes the multimodal representation of the distribution over optimum locations while preserving the remaining GUIDE pipeline.

``w/o value-aware reweighting'' removes the standardized conservative value score from the server weights, such that the merged components are weighted only by their normalized mixture masses.

``w/o component merging'' skips server-side clustering and moment matching while retaining value-aware reweighting and subsequent component sampling.
Thus, the original uploaded components directly enter the downstream weighting and redistribution procedure.

``w/o agent-specific sampling'' samples one global subset of at most $P$ components and broadcasts the same subset to all agents.
The packet size and weighted sampling distribution are kept identical to those of the full method, so this variant removes only the agent-specific stochastic redistribution.

``Uniform uncertainty scaling'' retains the complete distributional-exchange pipeline but removes the spatial localization of FI-GP.
Specifically, the spatially varying guidance field $G_{n,t}(\mathbf{x})$ is replaced at the intervention stage by its domain-averaged value
\[
\bar{G}_{n,t}
=
\frac{1}{Q}
\sum_{q=1}^{Q}
G_{n,t}(\mathbf{z}_q),
\]
where $\{\mathbf{z}_q\}_{q=1}^{Q}$ is a fixed quasi-uniform set over $\mathcal{X}$.
The corresponding uncertainty scaling becomes
\[
S^{\mathrm{uni}}_{n,t}(\mathbf{x})
=
1+\lambda_t\bar{G}_{n,t}.
\]
This preserves the average magnitude of the uncertainty intervention while removing its spatial variation, thereby separating spatially targeted guidance from a uniform increase in exploration.

\subsection{Level 2 Ablation Results}

Table~\ref{tab:main_ablation} and
Figure~\ref{fig:appendix_ablation_level2_curves} reveal a clear difference
between spatial uncertainty intervention and a generic increase in posterior
uncertainty. Full GUIDE-UCB achieves the best average rank of $1.50$.
Removing agent-specific sampling gives an average rank of $2.50$, followed
by removing value-aware reweighting at $2.75$, removing component merging
at $3.75$, and replacing the DPGMM with a single Gaussian at $4.50$.
Uniform uncertainty scaling and Vanilla UCB give the two weakest aggregate
results, with average ranks of $6.25$ and $6.75$.

The strongest effect comes from spatial localization. Full GUIDE-UCB,
Uniform uncertainty scaling, and Vanilla UCB obtain final regrets of
$8.4785$, $19.1366$, and $19.4350$ on Ackley, and $58.3082$, $80.3952$,
and $80.4771$ on Rastrigin. On Zakharov, the corresponding values are
$22.1717$, $52.0749$, and $49.8750$, while on Michalewicz they are
$5.8629$, $6.1078$, and $6.1248$. The trajectories show the same pattern.
Uniform uncertainty scaling stays close to Vanilla UCB on Ackley and
Rastrigin and loses most of the early convergence advantage of Full
GUIDE-UCB.

This ablation isolates an important part of FI-GP. The uniform variant keeps the complete distributional-exchange pipeline and
preserves the average magnitude of the intervention, but removes its spatial
variation. Its weak performance shows that a uniform increase in uncertainty is not
sufficient. The main benefit comes from increasing uncertainty selectively
in regions supported by the transferred distributions. FI-GP increases exploration around regions
supported by the received distributions while leaving other parts of the
search space much less affected.

The remaining ablations have smaller and more task-dependent effects.
Replacing the DPGMM with a single Gaussian gives an average rank of $4.50$.
A single Gaussian can represent one broad promising region but cannot retain
multiple separated modes when the posterior samples support several possible
optimum locations. Component merging removes redundant components so that multiple entries in
the limited downlink packet do not repeatedly represent the same region.
Removing this step increases the average rank to $3.75$.

Agent-specific sampling and value-aware reweighting affect the result less
uniformly under Level~2. Removing agent-specific sampling gives the
second-best average rank of $2.50$ and even obtains $21.5431$ on Zakharov,
compared with $22.1717$ for the full method. Removing value-aware
reweighting obtains $8.4753$ on Ackley, compared with $8.4785$ for Full
GUIDE-UCB. These reversals show that neither mechanism must improve every
individual function. Value-aware reweighting makes component selection more selective, while
agent-specific sampling gives different agents different subsets of the
global components.

\begin{figure*}[ht]
    \centering
    \begin{subfigure}[t]{0.25\textwidth}
        \centering
        \includegraphics[width=\linewidth]{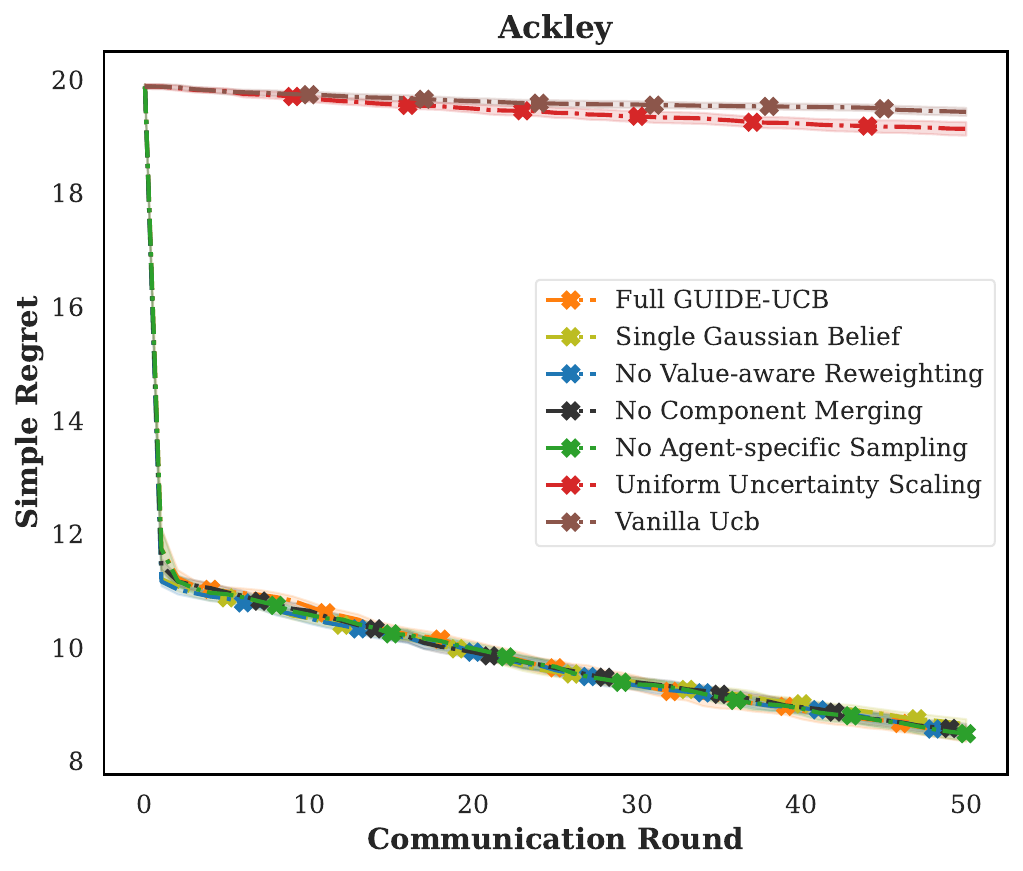}
        \caption{Ackley}
    \end{subfigure}\hfill
    \begin{subfigure}[t]{0.25\textwidth}
        \centering
        \includegraphics[width=\linewidth]{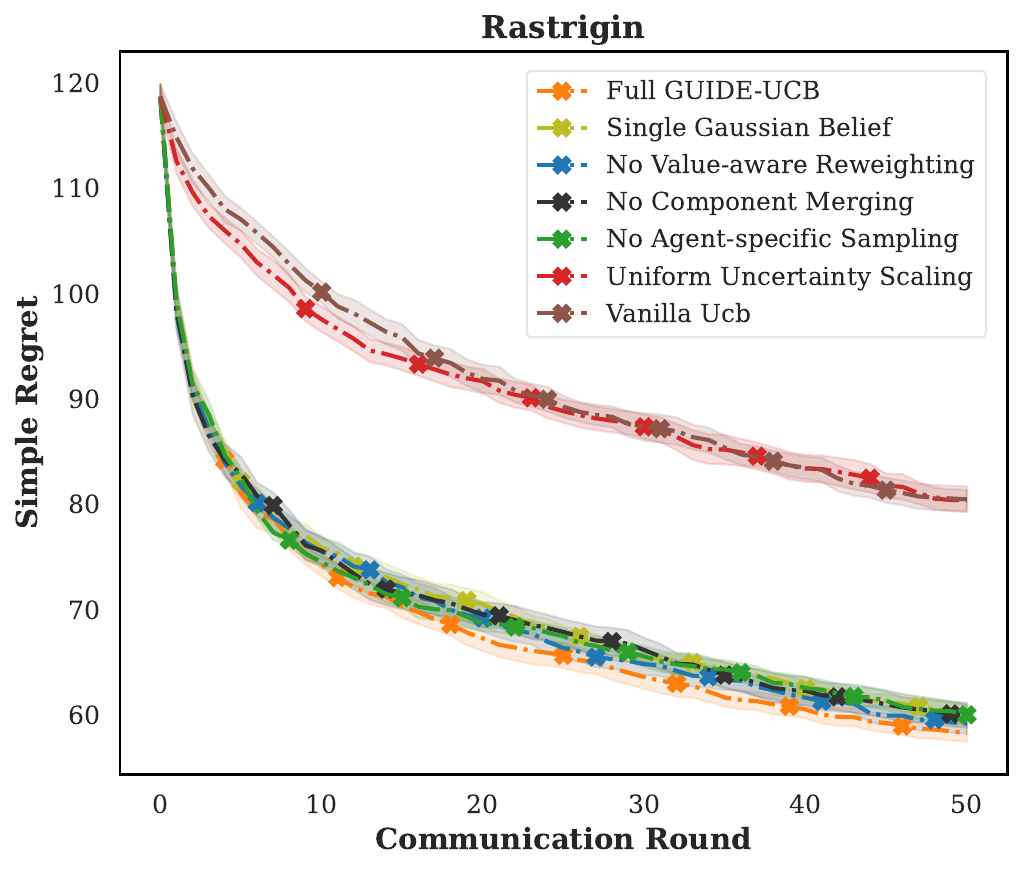}
        \caption{Rastrigin}
    \end{subfigure}\hfill
    \begin{subfigure}[t]{0.25\textwidth}
        \centering
        \includegraphics[width=\linewidth]{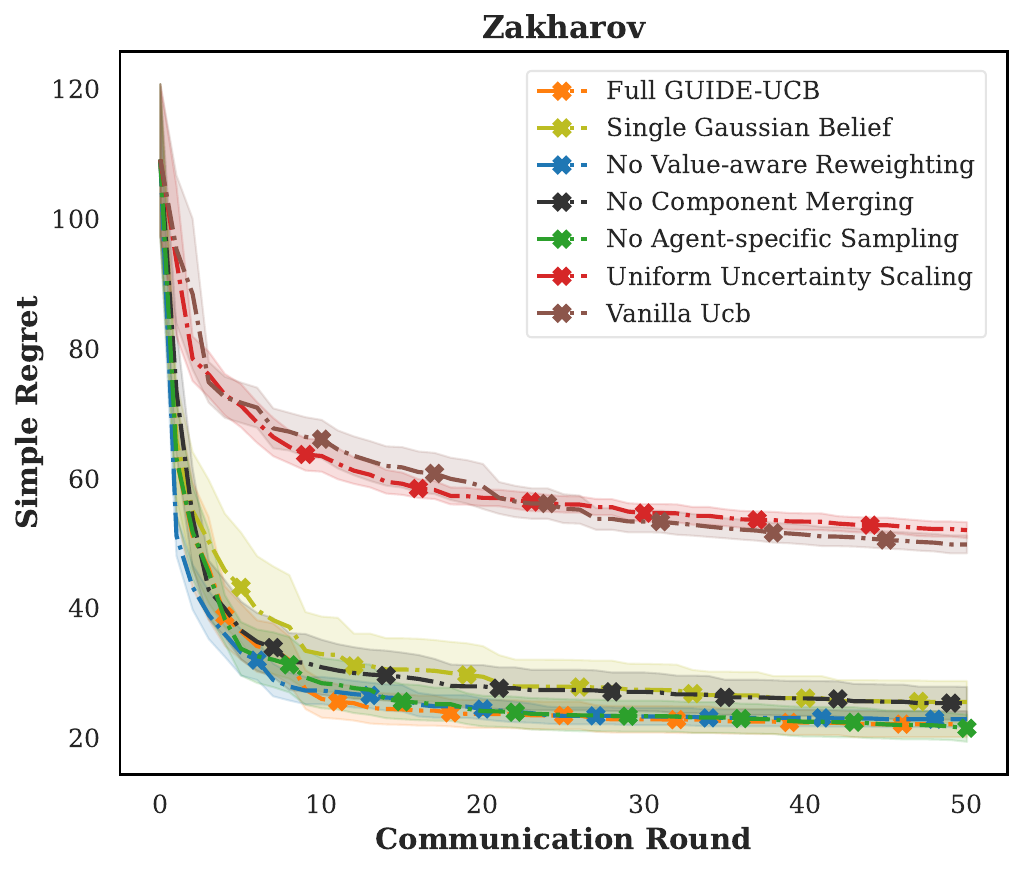}
        \caption{Zakharov}
    \end{subfigure}\hfill
    \begin{subfigure}[t]{0.25\textwidth}
        \centering
        \includegraphics[width=\linewidth]{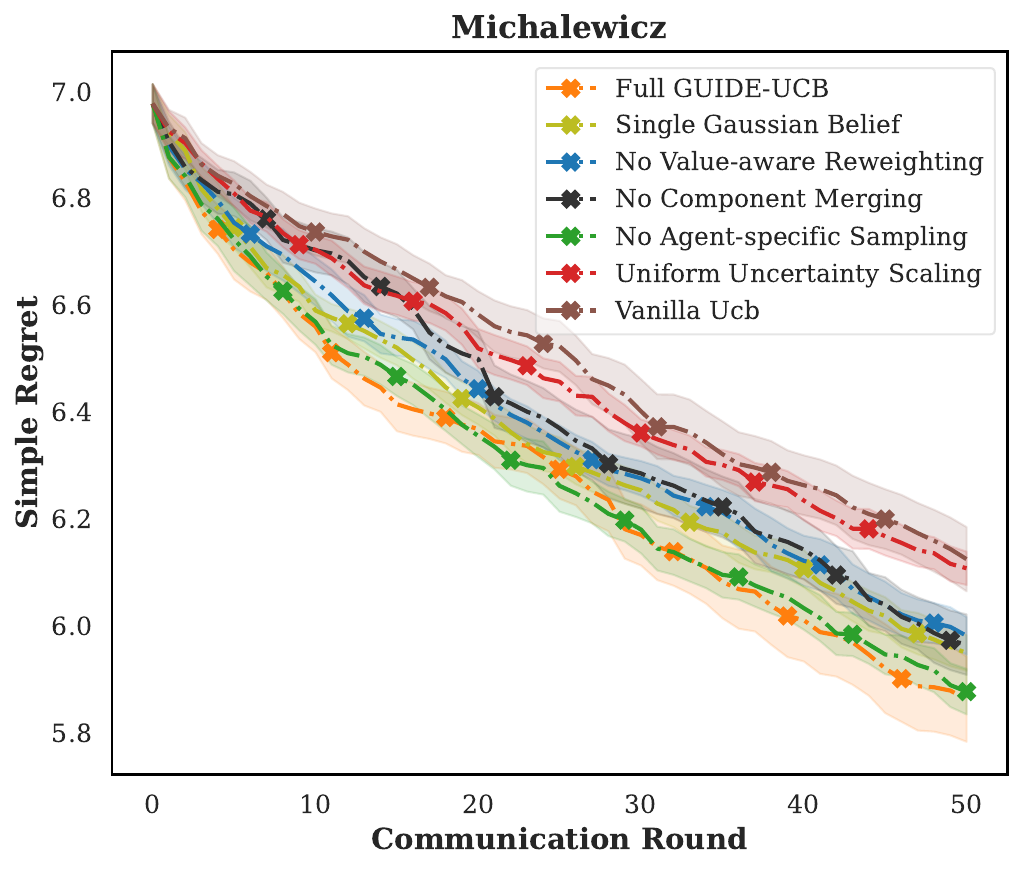}
        \caption{Michalewicz}
    \end{subfigure}
    \caption{Convergence trajectories for the GUIDE-UCB ablation study under Level~2 heterogeneity. Curves show the mean average simple regret over 10 independent runs, and shaded regions denote $\pm 1$ standard error of the mean across runs. Lower is better.}
    \label{fig:appendix_ablation_level2_curves}
\end{figure*}

\subsection{Level 3 Ablation Results}

Under Level~3 heterogeneity, the ablation results become much more
benchmark dependent. Full GUIDE-UCB nevertheless retains the best aggregate
rank of $2.50$. The Single-Gaussian variant ranks second at $3.00$, followed
by the variants without agent-specific sampling and component merging at
$3.75$ and $4.00$. Uniform uncertainty scaling obtains $4.50$, Vanilla UCB
obtains $5.00$, and removing value-aware reweighting gives the weakest
aggregate rank of $5.25$.

The best ablation variant now differs across functions. Removing value-aware
reweighting is best on Ackley with $20.1612$, while Full GUIDE-UCB obtains
$20.1739$. Uniform uncertainty scaling is best on Rastrigin at $90.2925$,
followed by the variant without agent-specific sampling at $90.7210$ and
Full GUIDE-UCB at $90.8849$. Full GUIDE-UCB is best on Zakharov at
$92.9444$, whereas Vanilla UCB is best on Michalewicz at $7.3490$. The
Level~3 trajectories consequently overlap much more than their Level~2
counterparts.

The most notable change is the role of value-aware reweighting. Removing it
gives an average rank of $2.75$ at Level~2 but the worst rank of $5.25$ at
Level~3. Strong heterogeneity means that the server receives distributions
from objectives that agree much less with one another. In this setting,
mixture mass alone is not enough to distinguish a component that is merely
frequent in posterior samples from one that is also supported by favorable
local function values. The value-aware score provides this additional
filter. Its effect is not uniform, as shown by the strong Ackley result
without reweighting, but it becomes more important for avoiding weak
aggregate performance across different tasks.

The Single-Gaussian variant ranks second overall, suggesting that the
additional benefit of multimodal representation is relatively modest in this
Level~3 study. Uniform uncertainty scaling is best on Rastrigin, showing that spatial
localization is not beneficial on every individual function.
Under severe heterogeneity, filtering unreliable transferred components
becomes more important because the received spatial information is less
consistently relevant to each local objective.

Full GUIDE-UCB achieves the best Level~3 average rank of $2.50$ despite not
being the best variant on every benchmark. Its advantage is therefore
primarily one of robustness across different objective landscapes.

\begin{table*}[ht]
    \caption{Ablation study of GUIDE-UCB on four representative synthetic benchmarks under Level~3 heterogeneity. Entries report the final mean average simple regret over 10 independent runs; lower is better. The last row reports the average rank across the four benchmarks. The best and second-best results are shown in bold and underlined, respectively.}
    \label{tab:appendix_ablation_level3}
    \centering
    \setlength{\tabcolsep}{3.0pt}
    \renewcommand{\arraystretch}{1.08}
    \scriptsize
    \resizebox{\textwidth}{!}{%
    \begin{tabular}{lccccccc}
    \toprule
    \textbf{Benchmark}
    & \shortstack{\textbf{Vanilla}\\\textbf{UCB}}
    & \shortstack{\textbf{Single-Gaussian}\\\textbf{belief}}
    & \shortstack{\textbf{w/o value-aware}\\\textbf{reweighting}}
    & \shortstack{\textbf{w/o component}\\\textbf{merging}}
    & \shortstack{\textbf{w/o agent-specific}\\\textbf{sampling}}
    & \shortstack{\textbf{Uniform uncertainty}\\\textbf{scaling}}
    & \shortstack{\textbf{Full}\\\textbf{GUIDE-UCB}} \\
    \midrule
    Ackley      & 20.5073 & 20.2321 & \best{20.1612} & 20.2473 & 20.2122 & 20.4661 & \second{20.1739} \\
    Rastrigin   & 95.6377 & 92.1619 & 94.3347 & 92.5862 & \second{90.7210} & \best{90.2925} & 90.8849 \\
    Zakharov    & 101.9986 & \second{95.9903} & 109.8422 & 96.5643 & 98.2310 & 103.0493 & \best{92.9444} \\
    Michalewicz & \best{7.3490} & \second{7.3529} & 7.5093 & 7.3660 & 7.4075 & 7.4026 & 7.3793 \\
    \midrule
    \textbf{Average rank}
                & 5.00 & \second{3.00} & 5.25 & 4.00 & 3.75 & 4.50 & \best{2.50} \\
    \bottomrule
    \end{tabular}%
    }
\end{table*}

\begin{figure*}[ht]
    \centering
    \begin{subfigure}[t]{0.25\textwidth}
        \centering
        \includegraphics[width=\linewidth]{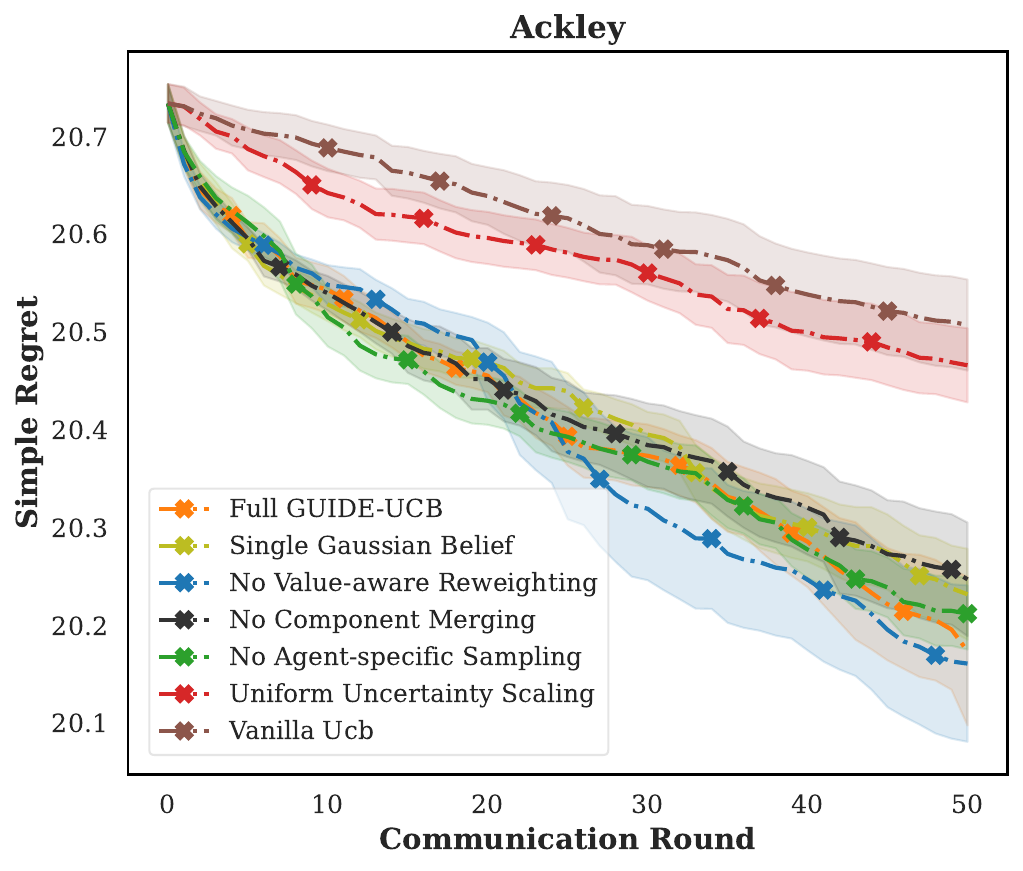}
        \caption{Ackley}
    \end{subfigure}\hfill
    \begin{subfigure}[t]{0.25\textwidth}
        \centering
        \includegraphics[width=\linewidth]{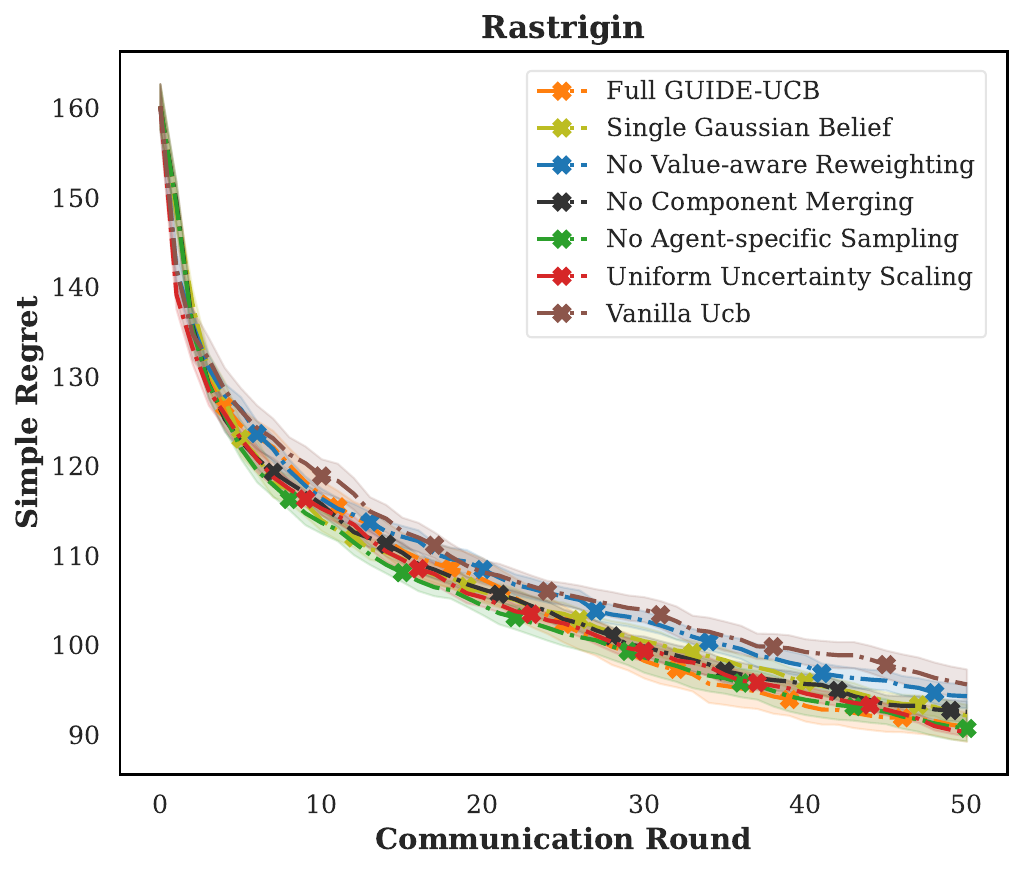}
        \caption{Rastrigin}
    \end{subfigure}\hfill
    \begin{subfigure}[t]{0.25\textwidth}
        \centering
        \includegraphics[width=\linewidth]{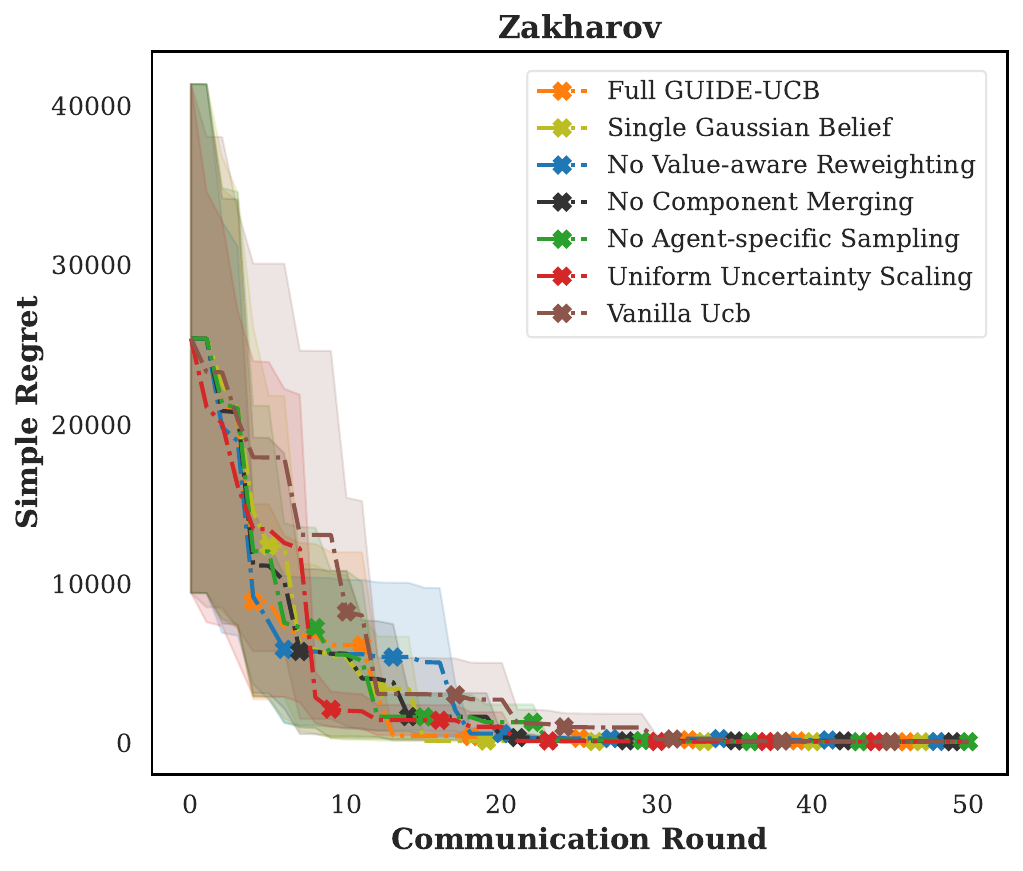}
        \caption{Zakharov}
    \end{subfigure}\hfill
    \begin{subfigure}[t]{0.25\textwidth}
        \centering
        \includegraphics[width=\linewidth]{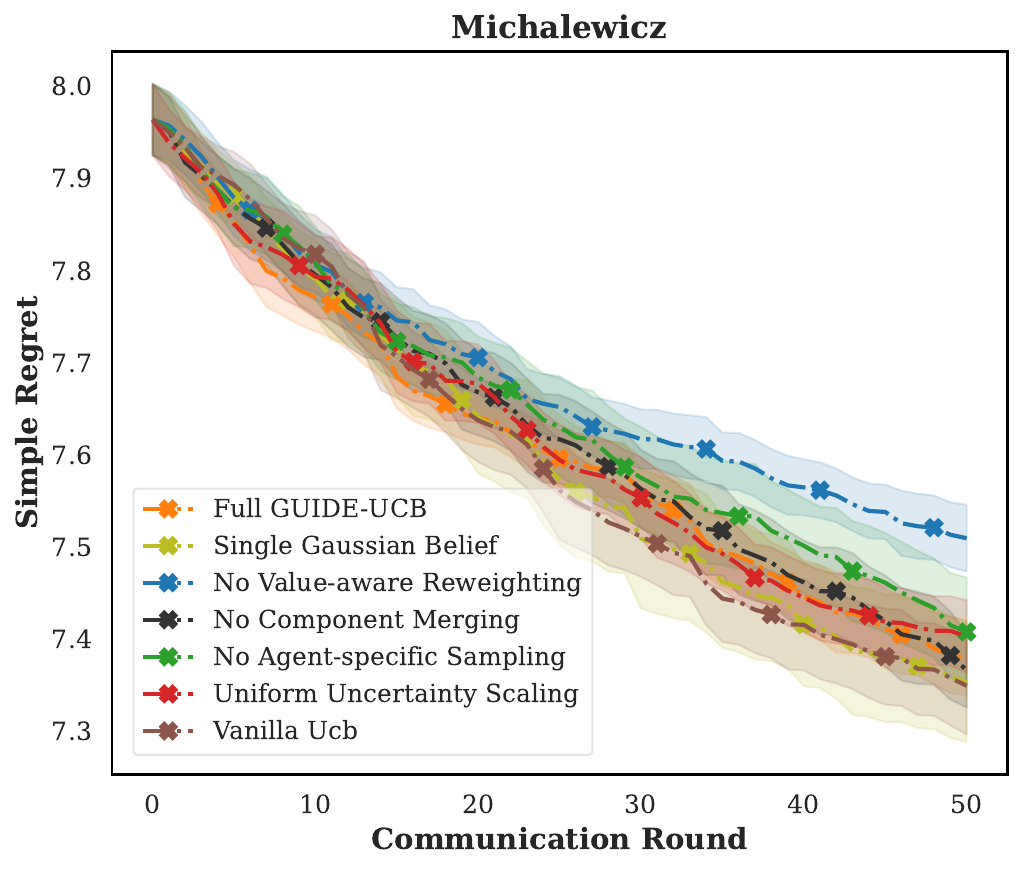}
        \caption{Michalewicz}
    \end{subfigure}
    \caption{Convergence trajectories for the GUIDE-UCB ablation study under Level~3 heterogeneity. Curves show the mean average simple regret over 10 independent runs, and shaded regions denote $\pm 1$ standard error of the mean across runs. Lower is better.}
    \label{fig:appendix_ablation_level3_curves}
\end{figure*}

\section{Sensitivity Analysis}
\label{app:sensitivity}

We evaluate the sensitivity of GUIDE-UCB under Level~3 heterogeneity on
Rastrigin and Michalewicz. Level~3 has the lowest empirical task similarity
($\overline{\mathrm{Sim}}=0.081$) and therefore provides a demanding setting
for examining whether the main GUIDE hyperparameters remain well behaved when
cross-agent agreement is weak. We vary one parameter at a time while keeping
all remaining settings fixed.

The tested parameters are the number of posterior-optimum samples $M$, the
downlink packet size $P$, the maximum intervention strength
$\lambda_{\max}$, and the component merging threshold
$\delta_{\mathrm{merge}}$. The complete parameter grids are reported in
Table~\ref{tab:sensitivity_grids}.

\begin{table}[ht]
    \caption{
    Hyperparameter grids used in the Level~3 sensitivity analysis. Bold entries denote the default settings.
    }
    \label{tab:sensitivity_grids}
    \centering
    \setlength{\tabcolsep}{5.0pt}
    \renewcommand{\arraystretch}{1.05}
    \small
    \begin{tabular}{lll}
    \toprule
    \textbf{Parameter}
    & \textbf{Tested values}
    & \textbf{Default} \\
    \midrule
    posterior-optimum samples $M$
        & $\{100,250,\mathbf{500},750,1000\}$ & 500 \\
    Downlink packet size $P$
        & $\{1,3,\mathbf{5},8,10\}$ & 5 \\
    Maximum guidance $\lambda_{\max}$
        & $\{0.25,0.5,\mathbf{1.0},2.0\}$ & 1.0 \\
    Merging threshold $\delta_{\mathrm{merge}}$
        & $\{0.01,0.02,0.03,0.04,\mathbf{0.05}\}$ & 0.05 \\
    \bottomrule
    \end{tabular}
\end{table}

\begin{table}[t]
    \caption{Sensitivity results of GUIDE-UCB on Rastrigin and Michalewicz under Level~3 heterogeneity.
    Entries report the final mean average simple regret over 10 independent runs; lower is better.}
    \label{tab:sensitivity_results}
    \centering
    \setlength{\tabcolsep}{6.0pt}
    \renewcommand{\arraystretch}{1.05}
    \small
    \begin{tabular}{llcc}
    \toprule
    \textbf{Parameter}
    & \textbf{Value}
    & \textbf{Rastrigin}
    & \textbf{Michalewicz} \\
    \midrule
    $M$
        & 100  & 92.6392 & 7.3885 \\
        & 250  & 91.9958 & 7.4102 \\
        & 500  & 90.8849 & 7.3793 \\
        & 750  & 91.3377 & 7.3489 \\
        & 1000 & 93.1664 & 7.3473 \\
    \midrule
    $P$
        & 1  & 92.5324 & 7.3935 \\
        & 3  & 91.4631 & 7.3586 \\
        & 5  & 90.8849 & 7.3793 \\
        & 8  & 93.0066 & 7.3982 \\
        & 10 & 93.2827 & 7.3982 \\
    \midrule
    $\lambda_{\max}$
        & 0.25 & 90.7113 & 7.3563 \\
        & 0.5  & 92.1546 & 7.3841 \\
        & 1.0  & 90.8849 & 7.3793 \\
        & 2.0  & 95.0430 & 7.4711 \\
    \midrule
    $\delta_{\mathrm{merge}}$
        & 0.01 & 89.8604 & 7.3847 \\
        & 0.02 & 90.5706 & 7.3613 \\
        & 0.03 & 90.4103 & 7.3823 \\
        & 0.04 & 89.5435 & 7.4494 \\
        & 0.05 & 90.8849 & 7.3793 \\
    \bottomrule
    \end{tabular}
\end{table}

\begin{figure*}[t]
    \centering
    \begin{subfigure}[t]{0.25\textwidth}
        \centering
        \includegraphics[width=\linewidth]{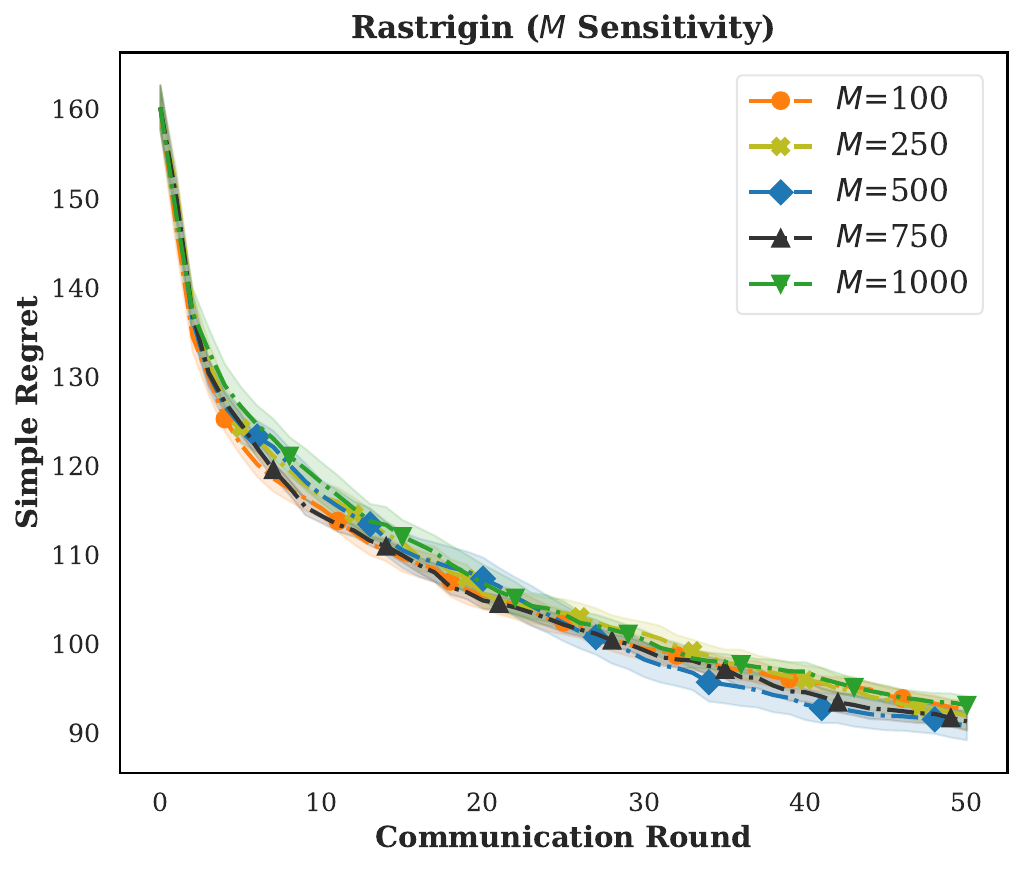}
        \caption{Posterior-optimum samples $M$}
    \end{subfigure}\hfill
    \begin{subfigure}[t]{0.25\textwidth}
        \centering
        \includegraphics[width=\linewidth]{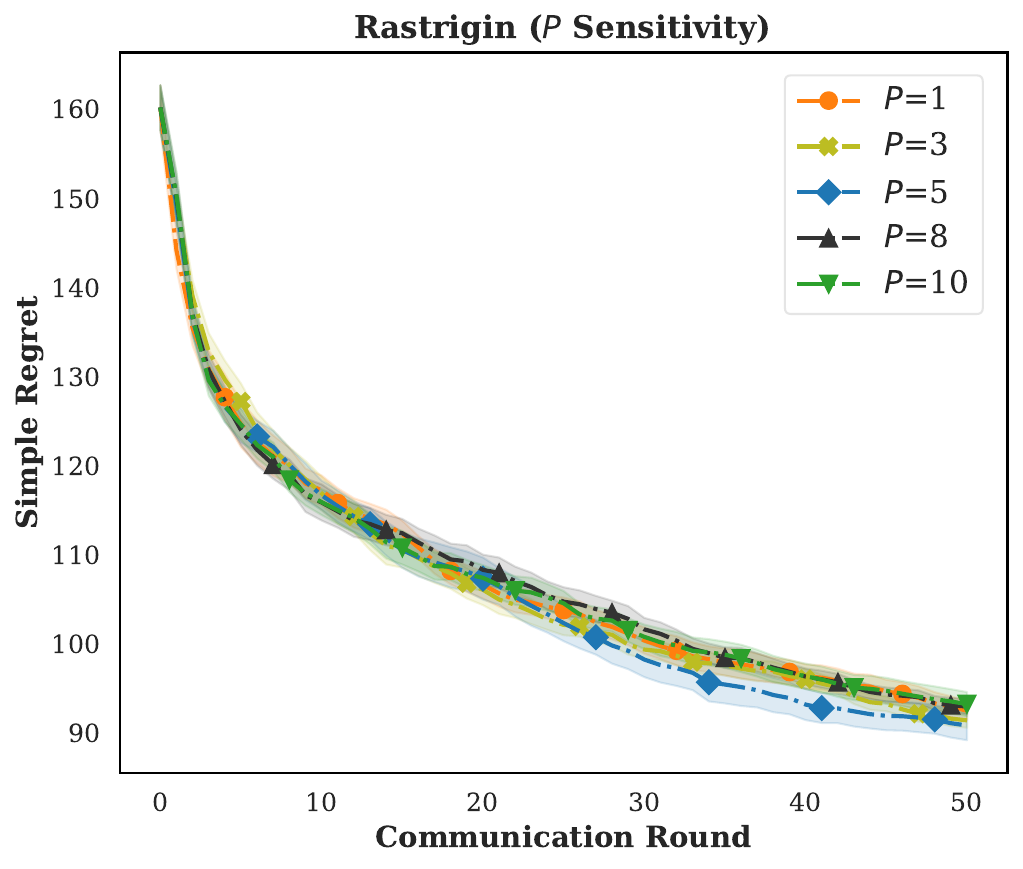}
        \caption{Packet size $P$}
    \end{subfigure}\hfill
    \begin{subfigure}[t]{0.25\textwidth}
        \centering
        \includegraphics[width=\linewidth]{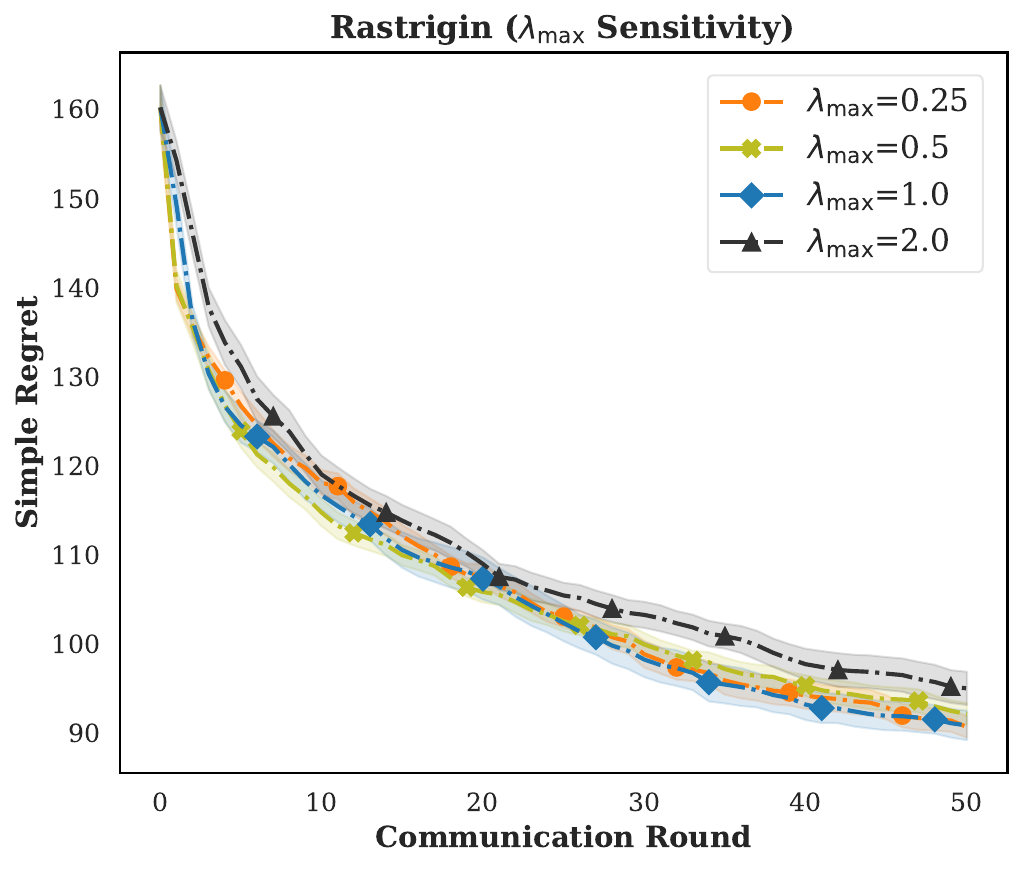}
        \caption{Guidance strength $\lambda_{\max}$}
    \end{subfigure}\hfill
    \begin{subfigure}[t]{0.25\textwidth}
        \centering
        \includegraphics[width=\linewidth]{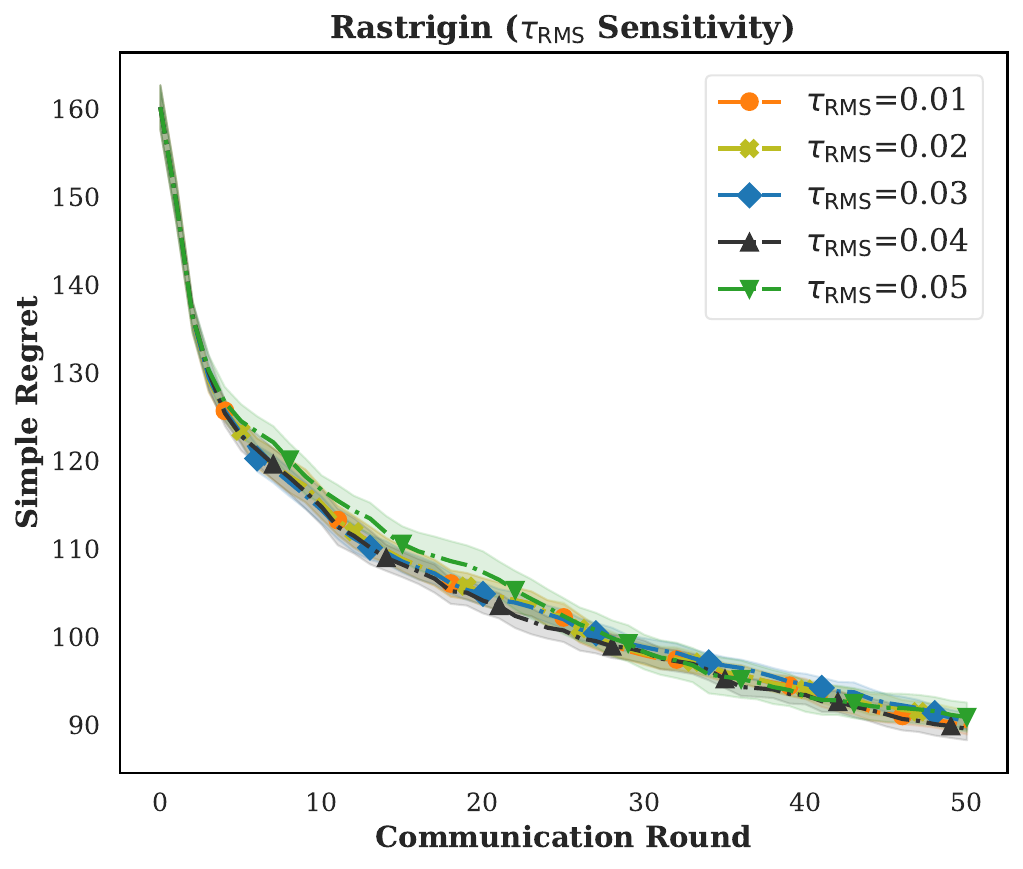}
        \caption{Merging threshold $\delta_{\mathrm{merge}}$}
    \end{subfigure}
    \caption{Sensitivity of GUIDE-UCB to its main hyperparameters on Rastrigin under Level~3 heterogeneity. Curves show the mean average simple regret over 10 independent runs, and shaded regions denote $\pm 1$ standard error of the mean across runs. Lower is better.}
    \label{fig:sensitivity_rastrigin}
\end{figure*}

\begin{figure*}[t]
    \centering
    \begin{subfigure}[t]{0.25\textwidth}
        \centering
        \includegraphics[width=\linewidth]{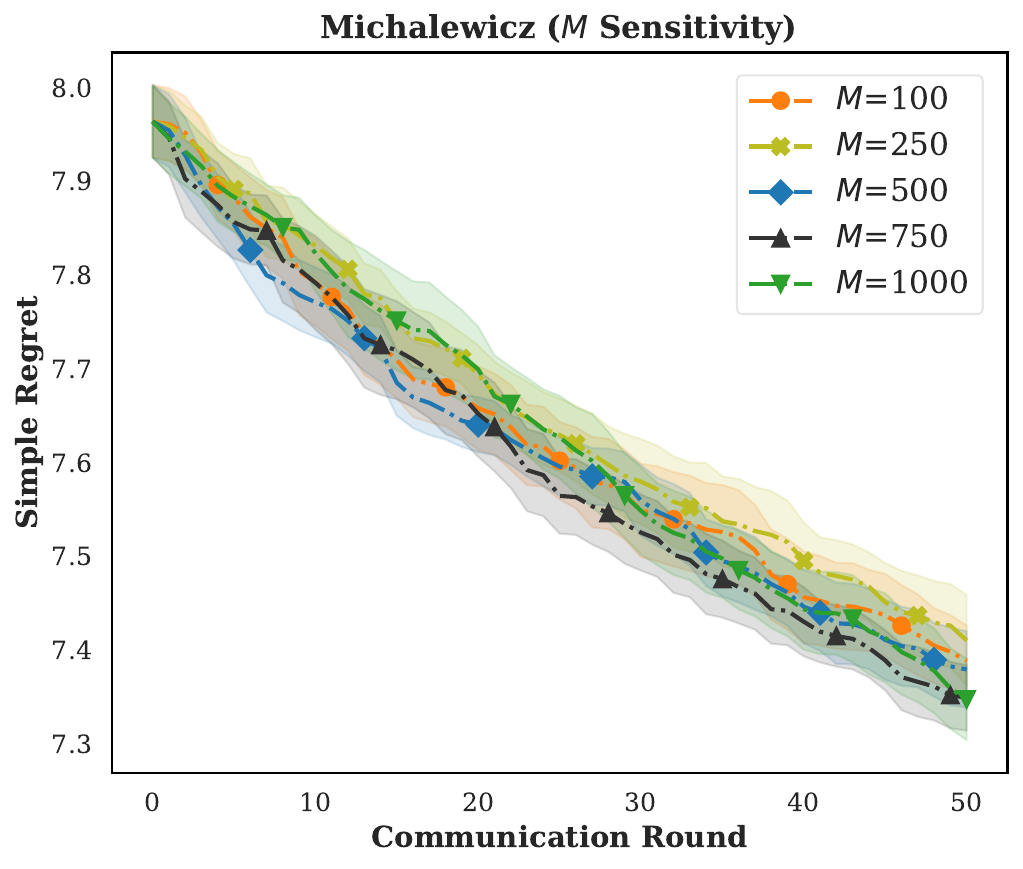}
        \caption{Posterior-optimum samples $M$}
    \end{subfigure}\hfill
    \begin{subfigure}[t]{0.25\textwidth}
        \centering
        \includegraphics[width=\linewidth]{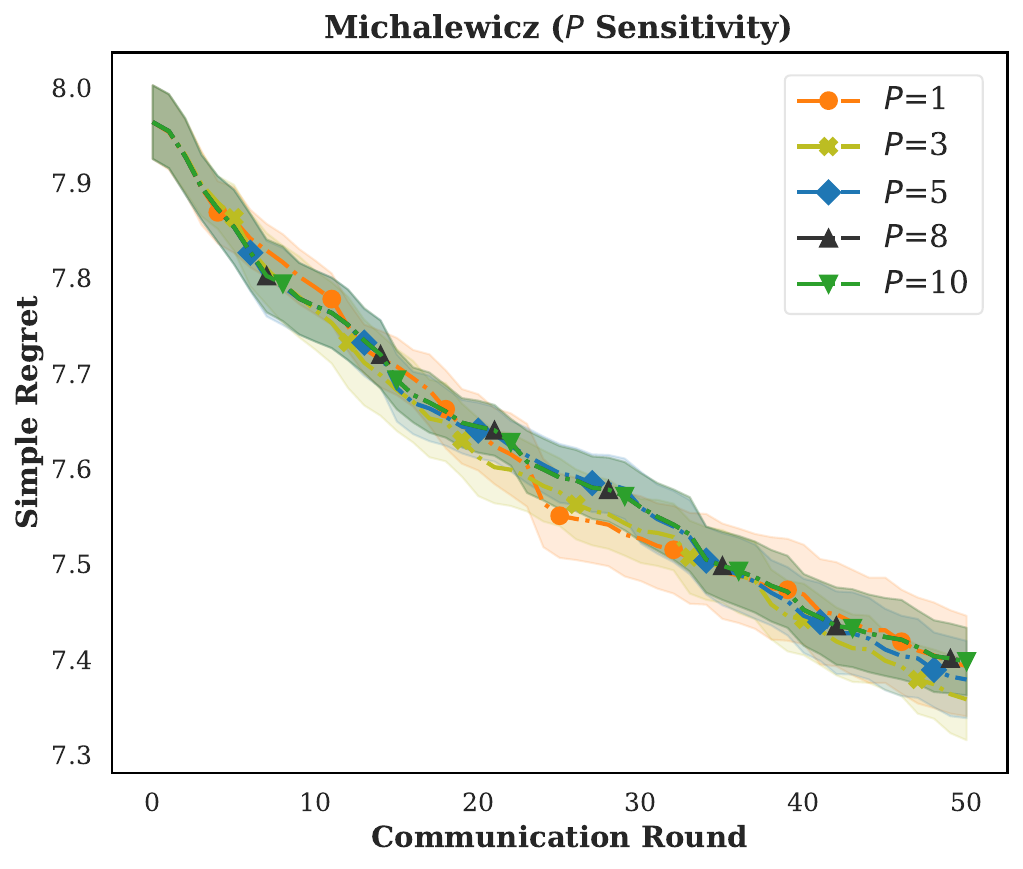}
        \caption{Packet size $P$}
    \end{subfigure}\hfill
    \begin{subfigure}[t]{0.25\textwidth}
        \centering
        \includegraphics[width=\linewidth]{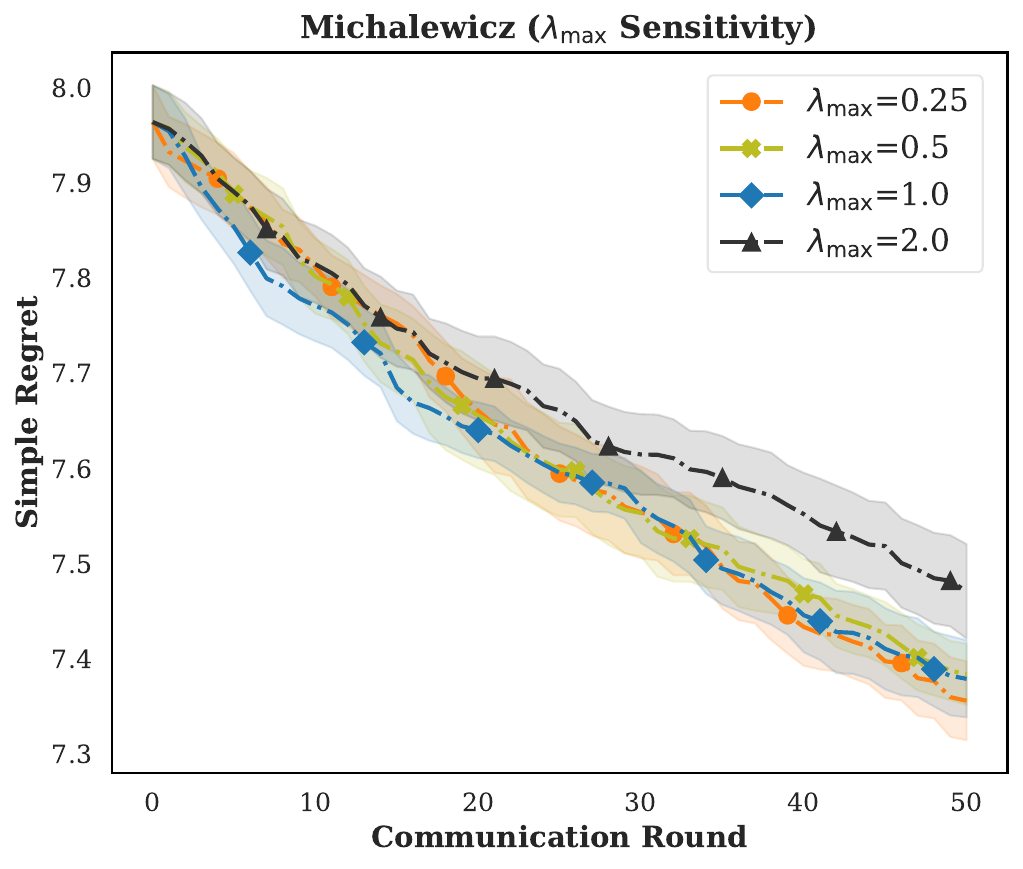}
        \caption{Guidance strength $\lambda_{\max}$}
    \end{subfigure}\hfill
    \begin{subfigure}[t]{0.25\textwidth}
        \centering
        \includegraphics[width=\linewidth]{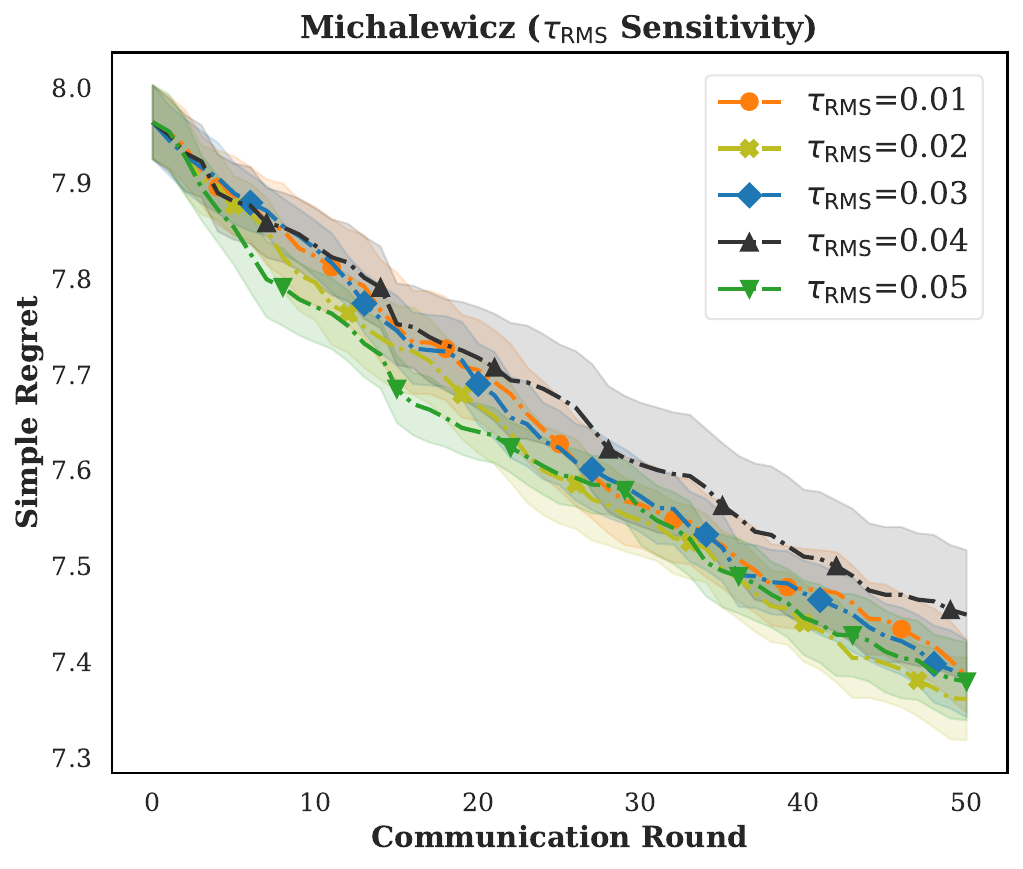}
        \caption{Merging threshold $\delta_{\mathrm{merge}}$}
    \end{subfigure}
    \caption{Sensitivity of GUIDE-UCB to its main hyperparameters on Michalewicz under Level~3 heterogeneity. Curves show the mean average simple regret over 10 independent runs, and shaded regions denote $\pm 1$ standard error of the mean across runs. Lower is better.}
    \label{fig:sensitivity_michalewicz}
\end{figure*}

Table~\ref{tab:sensitivity_results} reports the final average simple regret
for each parameter setting. For each parameter group, only the parameter
under study is varied, while all remaining hyperparameters are fixed at
their default values. Figures~\ref{fig:sensitivity_rastrigin} and
\ref{fig:sensitivity_michalewicz} show the corresponding convergence
trajectories over the complete BO horizon.

The effect of $M$ differs between the two benchmarks. On Rastrigin,
$M=500$ gives the lowest final regret of $90.8849$, while increasing the
sample size to $750$ and $1000$ gives $91.3377$ and $93.1664$,
respectively. Michalewicz shows the opposite tendency, with final regrets of
$7.3489$ and $7.3473$ at $M=750$ and $M=1000$, compared with $7.3793$ at
$M=500$. Increasing $M$ beyond $500$ therefore does not provide a
consistent benefit across the two benchmarks. A moderate number of
posterior-optimum samples is sufficient for Rastrigin, while Michalewicz
benefits slightly from a finer Monte Carlo approximation of the distribution over optimum locations.

The packet size $P$ also shows a task-dependent optimum. Rastrigin performs
best at $P=5$, with a final regret of $90.8849$, whereas Michalewicz performs
best at $P=3$, with $7.3586$. Increasing the packet size further degrades
both benchmarks. At $P=8$ and $P=10$, Rastrigin reaches $93.0066$ and
$93.2827$, while Michalewicz reaches $7.3982$ for both settings. A larger
packet gives each agent access to more global components, but under severe
heterogeneity some of these additional components may correspond to regions
that are poorly aligned with the local objective. Increasing $P$ therefore
does not necessarily improve the quality of the guidance. Moderate packet
sizes also keep the communication cost bounded while limiting the amount of
potentially conflicting global information received by each agent.

The clearest common pattern appears for $\lambda_{\max}$. Among the tested
values, $\lambda_{\max}=0.25$ gives the lowest final regret on both
benchmarks, with $90.7113$ on Rastrigin and $7.3563$ on Michalewicz. In
contrast, the strongest intervention, $\lambda_{\max}=2.0$, gives the worst
result on both benchmarks, reaching $95.0430$ and $7.4711$, respectively.
The results show that strong federated guidance is undesirable when
cross-agent agreement is weak. Increasing posterior uncertainty too
aggressively gives transferred information too much influence on the local
decision. A weaker intervention is sufficient to exploit useful global
information while leaving more weight to the local posterior.

The merging threshold $\delta_{\mathrm{merge}}$ has no common optimum across
the two benchmarks. Rastrigin achieves its lowest final regret of $89.5435$
at $\delta_{\mathrm{merge}}=0.04$, whereas Michalewicz achieves its lowest
value of $7.3613$ at $\delta_{\mathrm{merge}}=0.02$. This difference is
expected because the threshold controls the spatial resolution of the global
representation. A larger value merges more nearby components and produces a
coarser representation of the promising regions, while a smaller value
retains finer spatial distinctions. The preferred resolution therefore
depends on the geometry of the objective.

Overall, the sensitivity study does not favor uniformly larger parameter
settings. Increasing $M$, $P$, or $\lambda_{\max}$ does not consistently
reduce final regret, and the preferred merging threshold is task dependent.
The most consistent finding is that excessively strong uncertainty
intervention should be avoided under severe heterogeneity.

\section{Communication Cost Analysis}
\label{app:communication}

We measure communication by the number of scalar values transmitted per agent per BO round. 
Independent TS, UCB, and NEI do not communicate. 
The following accounting uses the default synthetic setting $d=10$, $D_{\mathrm{RFF}}=500$, diagonal GUIDE covariance, and $P=5$.

\paragraph{FTS and FTS-DE.}
FTS uploads and downloads a $D_{\mathrm{RFF}}$-dimensional random feature weight representation, giving $500$ scalars in each direction. 
Under the four-region distributed-exploration implementation of FTS-DE, the corresponding cost is $4D_{\mathrm{RFF}}=2000$ scalars in each direction.

\paragraph{FMTBO.}
FMTBO uploads two GP hyperparameters and ten ranking-related values, and receives two updated hyperparameters. 
Its uplink and downlink costs are therefore 12 and 2 scalars, respectively.

\paragraph{CGP family.}
Each CGP agent uploads one $d$-dimensional design point, one LCB value, and one posterior-mean value, for a total of $d+2=12$ scalars. 
Each agent receives 20 design points, giving $20d=200$ downlink scalars.

\paragraph{GUIDE-FBO.}
With a diagonal covariance, one uploaded component contains a $d$-dimensional mean, a $d$-dimensional covariance vector, one mixture weight, and one standardized value score. Hence,
\begin{equation*}
    C_{\mathrm{up}}^{\mathrm{GUIDE}} = 2d+2 = 22.
\end{equation*}
Each downlink component contains a mean, diagonal covariance, and global guidance weight. Therefore,
\begin{equation*}
    C_{\mathrm{down}}^{\mathrm{GUIDE}} = P_t(2d+1) \le P(2d+1) = 5\times21 = 105.
\end{equation*}
Thus, the total communication cost is at most 127 scalars per agent per round under the default setting. 
For a full covariance implementation, the covariance contribution becomes $d(d+1)/2$ scalars per component when only the symmetric entries are transmitted. 
The resulting communication costs under the default diagonal-covariance implementation are summarized in Table~\ref{tab:main_communication}.

\section{Computational Efficiency Analysis}
\label{app:efficiency}

We report the mean wall-clock time per independent run for each method using the same software environment and hardware. Experiments are executed on an AMD Ryzen 9 7945HX platform with 32 logical processors and 16 physical cores. Reported times are averaged over the repeated runs and measured in seconds.

\begin{table*}[h]
    \caption{Mean wall-clock time per independent run on the synthetic benchmarks under Level 2 heterogeneity, averaged over 10 runs. Unit: seconds.}
    \label{tab:appendix_time_synthetic}
    \centering
    \setlength{\tabcolsep}{3.0pt}
    \renewcommand{\arraystretch}{1.07}
    \scriptsize
    \resizebox{\textwidth}{!}{%
    \begin{tabular}{lcccccccccccc}
    \toprule
    \textbf{Benchmark}
    & \textbf{TS} & \textbf{UCB} & \textbf{NEI}
    & \textbf{FTS} & \textbf{FTS-DE} & \textbf{FMTBO}
    & \textbf{CGP-TS} & \textbf{CGP-UCB} & \textbf{CGP-NEI}
    & \textbf{GUIDE-TS} & \textbf{GUIDE-UCB} & \textbf{GUIDE-NEI} \\
    \midrule
    Ackley             & 129.51 & 143.48 & 269.34 & 141.88 & 148.32 & 160.56 
                       & 252.10 & 411.90 & 1811.38 & 267.19 & 275.87 & 402.82 \\
    Levy               & 111.27 & 125.76 & 262.50 & 133.12 & 124.19 & 256.59 
                       & 237.84 & 263.90 & 765.01 & 215.37 & 217.87 & 499.63 \\
    Griewank           & 85.13 & 99.67 & 279.07 & 118.16 & 113.43 & 242.71
                       & 246.30 & 253.81 & 773.83 & 179.68 & 185.57 & 344.68 \\
    Rastrigin          & 120.42 & 163.72 & 290.74 & 128.78 & 121.39 & 297.36 
                       & 265.75 & 359.93 & 1221.39 & 227.67 & 262.98 & 555.44 \\
    Weierstrass        & 126.20 & 140.59 & 250.92 & 141.79 & 149.75 & 171.29 
                       & 246.86 & 347.38 & 1050.71 & 266.59 & 267.94 & 677.94 \\
    Ellipsoid          & 175.70 & 164.10 & 306.68 & 161.63 & 153.78 & 246.67 
                       & 304.80 & 326.98 & 892.53 & 249.66 & 253.14 & 624.32 \\
    Sphere             & 167.19 & 174.37 & 327.94 & 178.33 & 164.62 & 266.78 
                       & 257.60 & 275.84 & 814.01 & 223.04 & 234.56 & 534.97 \\
    Zakharov           & 115.02 & 113.23 & 224.83 & 115.94 & 120.83 & 225.49 
                       & 242.18 & 240.12 & 449.03 & 233.55 & 240.50 & 562.81 \\
    Rosenbrock         & 65.79 & 78.15 & 172.91 & 78.58 & 89.00 & 154.95
                       & 154.68 & 163.35 & 333.75 & 194.66 & 201.80 & 439.28 \\
    Michalewicz        & 114.72 & 166.72 & 267.53 & 131.72 & 137.26 & 195.98 
                       & 231.74 & 368.63 & 1406.76 & 264.11 & 294.16 & 742.14 \\
    Powell             & 82.17 & 90.78 & 186.99 & 97.34 & 107.10 & 167.66 
                       & 185.97 & 185.55 & 296.17 & 227.31 & 234.32 & 477.27 \\
    Styblinski--Tang   & 83.22 & 88.35 & 216.58 & 101.46 & 106.42 & 183.35 
                       & 224.13 & 274.83 & 1048.93 & 206.32 & 219.41 & 548.39 \\
    \bottomrule
    \end{tabular}%
    }
\end{table*}

The synthetic objectives are inexpensive to evaluate, so the reported
wall-clock times primarily reflect the computational overhead of the BO
algorithms themselves. Averaged across the 12 Level~2 benchmarks, UCB,
GUIDE-UCB, and CGP-UCB require approximately $129.1$, $240.7$, and
$289.4$ seconds per complete run, respectively. GUIDE-UCB is faster than
CGP-UCB on $9$ of the $12$ benchmarks.

For the NEI family, GUIDE-NEI and CGP-NEI require approximately $534.1$
and $905.3$ seconds on average, respectively, with GUIDE-NEI being faster
on $9$ of the $12$ benchmarks. GUIDE therefore introduces clear
additional computation relative to independent local BO, mainly from
posterior-optimum sampling, distribution fitting, and server-side processing, while
remaining less expensive than the corresponding CGP variants on most of
the tested benchmarks.

This additional computation is a practical cost when objective
evaluations themselves are inexpensive. In expensive black-box
applications, however, simulation, model training, or physical
experimentation may dominate the total optimization time, in which case
optimizer-side computation can represent a smaller fraction of the
end-to-end cost. The relative runtime impact therefore depends on the
expense of the application being optimized.

\section{Orthogonal Random Features for Posterior-Optimum Sampling}
\label{app:orf}

GUIDE uses random Fourier features \citep{rahimi2007random} with an orthogonal-block construction inspired by \citet{yu2016orthogonal} to generate approximate posterior paths.
The local GP provides the posterior mean, while a Bayesian linear feature model provides the stochastic residual.

\paragraph{Feature representation.}
For a fixed agent and round, let $\boldsymbol{\ell}$ and $\sigma_f^2$ denote the fitted local GP's ARD lengthscales and output variance, and define $\boldsymbol{\Lambda} =\operatorname{diag}(\boldsymbol{\ell})$. We suppress the agent and round indices on the feature map for readability.
Direction vectors $\mathbf q_j$ are obtained from rows of orthogonal matrices generated by QR decomposition of independent standard Gaussian matrices.
For each direction, draw independent $u_j\sim\chi_d^2$, $v_j\sim\chi_5^2$, and $b_j\sim\mathrm{Unif}[0,2\pi)$, and define
\begin{equation*}
\boldsymbol{\omega}_j
=
\sqrt{u_j}\sqrt{\frac{5}{v_j}}\,
\boldsymbol{\Lambda}^{-1}\mathbf q_j,
\qquad
\phi_j(\mathbf x)
=
\sqrt{\frac{2\sigma_f^2}{D_{\mathrm{ORF}}}}
\cos(\boldsymbol{\omega}_j^\top\mathbf x+b_j).
\end{equation*}
The Student-$t_5$ radial scaling and ARD transformation follow the spectral parameterization of the local Mat\'ern-$5/2$ kernel. The resulting feature map approximates its stationary covariance through
$k(\mathbf x,\mathbf x')
\approx
\boldsymbol{\phi}(\mathbf x)^\top
\boldsymbol{\phi}(\mathbf x')$.

\paragraph{Posterior paths.}
We sample feature weights from their Gaussian posterior under the Bayesian linear feature model.
Let $\widetilde{\mathbf w}_{n,t}^{(m)}$ denote such a draw conditioned on $\mathcal D_{n,t-1}$, and let $\boldsymbol{\mu}_{w,n,t-1}$ denote the corresponding posterior weight mean. To retain the fitted local GP posterior mean, we construct
\begin{equation*}
\widehat f_{n,t}^{(m)}(\mathbf x)
=
\mu_{n,t-1}(\mathbf x)
+
\boldsymbol{\phi}(\mathbf x)^\top
\left(
\widetilde{\mathbf w}_{n,t}^{(m)}
-
\boldsymbol{\mu}_{w,n,t-1}
\right).
\end{equation*}
Conditional on the fitted model and feature map, the residual has zero mean.
Thus, the paths preserve the local GP posterior mean while approximating its posterior covariance through the feature model.

\paragraph{Optimum distribution extraction and TS decisions.}
For optimum distribution extraction, we generate $M$ paths with independent posterior weight draws, sharing
a feature map and a uniformly sampled candidate set $\mathcal X_{\mathrm{cand}}\subset\mathcal X$.
We obtain optimum-location samples as
\begin{equation*}
\widehat{\mathbf x}_{n,t}^{(m)}
\in
\arg\max_{\mathbf x\in\mathcal X_{\mathrm{cand}}}
\widehat f_{n,t}^{(m)}(\mathbf x),
\end{equation*}
and use these samples to fit the local DPGMM.

For GUIDE-TS, a fresh posterior weight draw defines a path whose centered residual is scaled by the guidance multiplier:
\begin{equation*}
\widetilde f_{n,t}(\mathbf x)
=
\mu_{n,t-1}(\mathbf x)
+
S_{n,t}(\mathbf x)
\left(
\widehat f_{n,t}(\mathbf x)
-
\mu_{n,t-1}(\mathbf x)
\right).
\end{equation*}
This path is optimized over $\mathcal X$ using continuous acquisition optimization.
For GUIDE-UCB and GUIDE-NEI, the feature sampler is used for optimum distribution extraction, while acquisition
evaluation uses the intervened local GP posterior.
All features and sampled paths remain on the agent; only the selected distributional component is uploaded.

\end{document}